\documentclass{article}

\usepackage{microtype}
\usepackage{graphicx}
\usepackage{subfigure}
\usepackage{booktabs} % for professional tables

\usepackage{hyperref}

\usepackage[accepted]{icml2020}

\usepackage[utf8]{inputenc} % allow utf-8 input
\usepackage[T1]{fontenc}    % use 8-bit T1 fonts
\usepackage{url}            % simple URL typesetting
\usepackage{booktabs}       % professional-quality tables
\usepackage{amsfonts}       % blackboard math symbols
\usepackage{nicefrac}       % compact symbols for 1/2, etc.
\usepackage{microtype}      % microtypography

\usepackage{times}
\usepackage{graphicx}
\usepackage{natbib}
\usepackage{amsmath}
\usepackage[psamsfonts]{amssymb}
\usepackage{amstext}
\usepackage{amsthm}
\usepackage{latexsym}
\usepackage{color}
\usepackage{graphicx}
\usepackage{enumerate}
\usepackage{url}
\usepackage[mathscr]{euscript}
\usepackage{booktabs}
\usepackage{makecell}
\usepackage{tabularx}
\usepackage{xcolor, colortbl}

\usepackage{wrapfig}

\usepackage{algorithm}
\usepackage{caption}
\usepackage{nccmath}

\usepackage{eqparbox}
\renewcommand{\algorithmiccomment}[1]{\hfill\eqparbox{COMMENT}{\# #1}}

\let\Pr\undefined
\def\Rset{\mathbb{R}}
\def\Hset{\mathbb{H}}

\DeclareMathOperator*{\E}{\mathbb{E}}
\DeclareMathOperator*{\Pr}{\mathbb{P}}
\DeclareMathOperator*{\argmax}{\rm argmax}
\DeclareMathOperator*{\argmin}{\rm argmin}
\DeclareMathOperator*{\var}{var}

\definecolor{Gray}{gray}{0.85}
\newcolumntype{g}{>{\columncolor{Gray}}c}

\newcommand{\iwal}{\textsc{iwal}}
\newcommand{\margin}{\textsc{margin}}
\newcommand{\eiwal}{\textsc{eiwal}} % e-iwal

\newcommand{\riwal}{\textsc{riwal}}
\newcommand{\oriwal}{\textsc{oriwal}} % opt-rb-iwal
\newcommand{\arbal}{\textsc{arbal}} % opt-rb-iwal
\newcommand{\splitregion}{\textsc{Split}} % opt-rb-iwal
\newcommand{\prodH}{\sH_{[\kappa]}}

\newcommand{\tp}{\mathtt{p}}

\makeatletter
\newtheorem*{rep@theorem}{\rep@title}
\newcommand{\newreptheorem}[2]{%
\newenvironment{rep#1}[1]{%
 \def\rep@title{#2 \ref{##1}}%
 \begin{rep@theorem}}%
 {\end{rep@theorem}}}
\makeatother

\hypersetup{
 colorlinks  = true, %Colours links instead of ugly boxes
 urlcolor   = blue, %Colour for external hyperlinks
 linkcolor  = blue, %Colour of internal links
 citecolor  = blue %Colour of citations
}

\newcommand{\one}{1}
\newcommand{\set}[2][]{#1 \{ #2 #1 \} }

\newcommand{\h}{\widehat}

\newcommand{\cF}{\mathcal{F}}

\newcommand{\cH}{\mathscr{H}}

\newcommand{\cX}{\mathscr{X}}
\newcommand{\cY}{\mathscr{Y}}
\newcommand{\cZ}{\mathscr{Z}}

\newcommand{\gain}{\rho}

\newcommand{\e}{\epsilon}

\newcommand{\sD}{\mathscr{D}}
\newcommand{\sH}{\mathscr{H}}

\newcommand{\condsD}{\sD}

\usepackage{mathtools}

\newcommand{\slack}{\sigma_{T}}
\newcommand{\Slack}[1]{{\sigma_{T}}}

\newtheorem{theorem}{Theorem}
\newtheorem{lemma}[theorem]{Lemma}

\newtheorem{corollary}[theorem]{Corollary}
\newtheorem{proposition}[theorem]{Proposition}

\newreptheorem{theorem}{Theorem}
\newreptheorem{lemma}{Lemma}
\newreptheorem{corollary}{Corollary}
\newreptheorem{proposition}{Proposition}

\usepackage{tikz}
\usetikzlibrary{decorations.pathreplacing,calc}

\newcommand{\ignore}[1]{}

\icmltitlerunning{Adaptive Region-Based Active Learning}

\begin{document}

\twocolumn[
\icmltitle{Adaptive Region-Based Active Learning}

\icmlsetsymbol{equal}{*}
\begin{icmlauthorlist}
\icmlauthor{Corinna Cortes}{google}
\icmlauthor{Giulia DeSalvo}{google}
\icmlauthor{Claudio Gentile}{google}
\icmlauthor{Mehryar Mohri}{google,courant}
\icmlauthor{Ningshan Zhang}{stern}
\end{icmlauthorlist}

\icmlaffiliation{google}{Google Research, New York, NY}
\icmlaffiliation{courant}{Courant Institute of Mathematical Sciences, 
New York, NY}
\icmlaffiliation{stern}{New York University, New York, NY}

\icmlcorrespondingauthor{Ningshan Zhang}{nzhang@stern.nyu.edu}

\icmlkeywords{Active Learning, Region-Based}

\vskip 0.3in
]

\printAffiliationsAndNotice{}  % leave blank if no need to mention equal contribution
\begin{abstract}
  We present a new active learning algorithm that adaptively
  partitions the input space into a finite number of regions, and
  subsequently seeks a distinct predictor for each region, both phases
  actively requesting labels. We prove theoretical guarantees for
  both the generalization error and the label complexity of our
  algorithm, and analyze the number of regions defined by the algorithm
  under some mild assumptions. We also report the results of an extensive
  suite of experiments on several real-world datasets demonstrating substantial
  empirical benefits over existing single-region and non-adaptive
  region-based active learning baselines.
\end{abstract}

\section{Introduction}
\label{sec:intro}
\vskip -0.05in

In many learning problems, including document classification, image
annotation, and speech recognition, large amounts of unlabeled data
are at the learner's disposal at practically no cost. In contrast,
reliable labeled data is often more costly to acquire, since it
requires careful assessments by human labelers. To limit that cost, in
\emph{active learning}, the learner seeks to request as few labels as
possible to learn an accurate predictor. This is an attractive
learning scenario with significant practical benefits, which remains a
challenging theoretical and algorithmic setting.

\ignore{
Several active learning algorithms have been designed in the past for
both the \emph{pool setting}, where the learner has access to the full
set of unlabeled data, and the \emph{on-line setting} where unlabeled
i.i.d.\ samples are received sequentially and where the decision about
requesting a label must be made on-line. 
}

The literature on active learning is very broad. Thus, we 
give only a brief discussion of previous work here and refer the
reader to \citep{Dasgupta2011} for an in-depth survey of the main
algorithmic and theoretical ideas, as well as its current challenges.
For separable problems, \citet{cohn1994improving} introduced the
\textsc{cal} algorithm, which only requires a logarithmic number of
label requests, $\log (\frac 1 \e)$, to obtain $\e$-accuracy.  Later,
other on-line active learning algorithms for general hypothesis
classes and distributions were designed with guarantees both for
generalization and label complexity in the agnostic setting
\citep{freund1997selective,Balcan2006AgnosticAL,hanneke2007bound,
  dasgupta2008general,beygelzimer2009importance,
  beygelzimer2010agnostic,huang2015efficient,zhang2014beyond}, and in the
separable settings \cite{das04,gk17,no11,td17}.

The theoretical analysis of the label complexity 
of active learning for various hypothesis classes and data
distributions has been discussed in several publications
\citep{dasgupta2006coarse,castro2008minimax,koltchinskii2010rademacher,
  hanneke2015minimax,Hanneke2014,ml18}.
In particular, for hypothesis sets consisting of
linear separators, a series of publications gave
margin-based on-line active learning algorithms that admit
guarantees under some specific distributional assumptions \citep{dasgupta2005analysis,balcan2007margin,
  balcan2013active,awasthi2015efficient,Zhang2018EfficientAL}.

For all these algorithms, the hypothesis set or \emph{version space}
$\sH$ is fixed beforehand and, over time, as more labeled information
is acquired, it is gradually shrunk 
to rule out hypotheses too far from the best-in-class hypothesis.
This paper initiates the study of an
alternative family of algorithms where the hypothesis set $\sH$ is
first expanded over time before shrinking. Specifically, we consider
active learning algorithms that \emph{adaptively} partition the input
space into a finite number of disjoint regions, each equipped with the
hypothesis set $\sH$, and that subsequently seek a distinct predictor
for each region.  Such algorithms can achieve a substantially better
performance, as shown by our theoretical analysis and largely demonstrated
by our experiments.

\ignore{
This region-specific predictor is at least as accurate as the
global best-in-class predictor, but often it achieves a significantly
higher accuracy, as demonstrated by our experiments.

In fact, there are many real-world applications where the input space
consists of natural regions that define a partition, and the
best-in-class predictor within each region is likely more accurate and
distinct from the single overall best-in-class predictor. For
instance, in speech recognition, the input data may admit natural
clusters, each corresponding to different types of speech data, like
conversational or casual speech, broadcast news, and dictation
speech. There are significant differences among the properties of such
datasets. For example, in conversational speech, the topic, speaking
rate, and sentence length may vary considerably, while this may occur
far less frequently in broadcast news or dictation. Similar natural
clusters can often be found in other learning tasks.
}

The design of such algorithms raises several questions: How should the
input space be partitioned to ensure an improvement in overall
performance? How can labels be requested most effectively across
regions to learn an accurate predictor per region? Can we provide
generalization and label complexity guarantees? In this paper, we tackle these 
questions and devise an algorithm for this problem, called
Adaptive Region-Based Active Learning (\arbal), benefiting from
favorable theoretical guarantees. From a theoretical standpoint, there
are several challenging problems: ensuring that the region-specific
best-in-class hypothesis is not discarded, the selection of the
splitting criteria, and the dependency of the final generalization
bound on such criteria.

Of course, if a beneficial partition of the input space is available
to the learner, as assumed in the related work of
\citet{cortes2019rbal}, then no further work is needed to adaptively
seek one. In practice, however, such strong oracle information may not
be available and, even when a natural pre-partitioning of the input
space is available, without recourse to labeled data, it is not
guaranteed to help improve the generalization error. Furthermore, we
will not assume that dividing the input space is always
beneficial. However, if there exists indeed a partition such that a
region-specific predictor performs significantly better than a global
one, then, with high probability, \arbal\ will find it. Otherwise, no
split is made and \arbal\ works just like a single-region active
learning algorithm.  In practice, in almost all cases we tested,
\arbal\ splits the input space into multiple regions and achieves a
significant performance improvement.

Another line of work somewhat related to our paper is the hierarchical
sampling approach of \citet{dasgupta2008hierarchical} in the
\emph{pool-based setting} of active learning, further analyzed by
\cite{urner2013plal} and \cite{kpotufe2015hierarchical}, where the
learner receives as input a batch of unlabeled points to select
from. However, it is important to stress that the methods proposed in
those papers rely on (hierarchical) clusterability assumptions of the
data that help save labels, while, here, we are more concerned with a
problem in model selection for active learning, where splitting the
input space is likely to improve generalization rather than reducing
label complexity.

In summary, we present an active learning algorithm, \arbal, that
adaptively partitions the input space and performs region-based active
learning.
Our theoretical results (Theorem~\ref{thm:splitiwal2} and
Theorem~\ref{thm:splitiwal}) show that, remarkably, when the algorithm
splits the input space into $K$ regions, modulo a standard term in
$O(1/\sqrt{T})$ decreasing with the number of rounds $T$, the
generalization error of \arbal\ is close to $R^* - \gamma (K - 1)$,
where $R^*$ is the best-in-class error for the unpartitioned original
input space and $\gamma > 0$ a parameter of the algorithm. Thus, when
at least one split is made by \arbal\, ($K > 1$), then, for $T$
sufficiently large, the error of the algorithm is close to a quantity
strictly smaller than the original best-in-class error!
Moreover, we show that, under mild theoretical assumptions, \arbal\
indeed splits the original input space into multiple subregions
(Proposition~\ref{lemma:first_split_time} and
Corollary~\ref{cor:num_splits}). Our experiments confirm that this
almost always occurs (Section~\ref{sec:experiments}).
This significant theoretical improvement over even the original
best-in-class error is further corroborated by our extensive
experimental study with $25$ datasets where, in most cases,
\arbal\ achieves a better performance than the best 
active learning algorithm working with the original single region.

\ignore{
To summarize, we propose an active learning algorithm, \arbal, that
adaptively partitions the input space and performs region-based active
learning.  Our theoretical results (Theorem~\ref{thm:splitiwal2}) show
generalization guarantees for \arbal\ with respect to the {\em
  region-based} best-in-class error, which admits significant
improvement over single-region best-in-class error.  When the
algorithm splits the original input space into multiple subregions
which, under mild theoretical assumptions, does occur
(Proposition~\ref{lemma:first_split_time} and
Corollary~\ref{cor:num_splits}) and which experimentally almost always
occurs (Section~\ref{sec:experiments}), the improvement in
generalization error is very significant.  This improvement is also
corroborated by our extensive array of experiments.
}

The rest of this paper is structured as follows. In
Section~\ref{sec:prelim}, we introduce the preliminaries relevant to
our discussion and give a more formal definition of the learning
scenario. In Section~\ref{sec:alg}, we present our new
learning algorithm, \arbal, and justify its splitting criterion via
theoretical guarantees.
\ignore{
 which relies on the \iwal\ algorithm
\citep{beygelzimer2009importance} as a subroutine.  As explained in
Section \ref{sec:alg}, other base active learning algorithms could be
used instead of \iwal. \arbal\ learns to partition while actively
requesting labels, the partition being guided by expected improvements
in the generalization error of the best-in-class hypothesis after each
split. We justify our splitting criterion with theoretical
guarantees.
} 
In Section~\ref{sec:theory}, we provide generalization and
label complexity bounds for \arbal\ in terms of a key parameter for the
splitting criterion, and the number of regions partitioned. Moreover,
in Section~\ref{subsec:adagamma} 
we show that, under some natural assumptions about the data
distribution, \arbal\ benefits from guaranteed improvement over \iwal\
\citep{beygelzimer2009importance}. In Section~\ref{sec:experiments}, we
report the results of a series of experiments on multiple datasets,
demonstrating the substantial benefits of \arbal\ over existing
non-region-based active learning algorithms, such as \iwal\, and
margin-based uncertainty sampling, and over the nonadaptive
region-based active learning baseline \oriwal\ \citep{cortes2019rbal}.

\vskip -0.1in
\section{Learning scenario}
\label{sec:prelim}
\vskip -0.05in

We now discuss the learning scenario, starting with some
preliminary definitions.
Let $\cX \subseteq \Rset^D$ denote the input space,
$\cY = \set{-1, +1}$ the output space, and $\sD$ an unknown
distribution over $\cX \times \cY$. We denote by $\sD_\cX$ the
marginal distribution of $\sD$ over $\cX$ and, given a prediction
space $\cZ\subseteq \Rset$, we denote by
$\ell \colon \cZ \times \cY \to [0, 1]$ a loss function, which we
assume to be $\mu$-Lipschitz with respect to its first argument, for
some constant $\mu>0$. Let $\sH$ be a family of hypotheses consisting
of functions mapping $\cX$ to $\cZ$.
%a prediction space $\cZ\subseteq \Rset$.  
Then, the generalization error
or expected loss of a hypothesis $h \in \sH$ is denoted by $R(h)$ and
defined as $R(h) = \E_{(x, y) \sim \sD}[\ell(h(x), y)]$.

We consider the on-line setting of active learning where, at each
round $t \in [T] = \set{1, \ldots, T}$, the learner receives as input
a point $x_t \in \cX$ drawn i.i.d.\ according to $\sD_\cX$ and must
decide to request or not its label $y_t$. The decision is final
and cannot be retroactively changed.\ignore{ of not requesting the
  label $y_t$ is permanent, meaning that at any future round $t' > t$,
  the learner cannot change her past decision and ask for $y_t$.} At
the end of $T$ rounds, the learner returns a hypothesis
$\h h_T \in \sH$.
In this setting, two conflicting quantities determine the performance
of an on-line active learning algorithm: its label complexity, that
is, the number of labels it has requested over $T$ rounds, and the
generalization error $R(\h h_T)$ of the hypothesis it returns.

In the standard case where the hypothesis set $\sH$ is given
beforehand, the learner seeks a single best predictor from $\sH$.
Here, we consider instead the setup where the algorithm adaptively
partitions the input space $\cX$ into $K$ regions
$\cX_1, \ldots, \cX_{K}$, each equipped with a copy of the hypothesis set $\sH$ and with $K$ upper-bounded by some parameter
$\kappa \geq 1$.
Given the partition $\cX_1, \ldots, \cX_{K}$, the hypothesis $\h h_T$
returned by the algorithm after $T$ rounds admits the following form:
$\h h_T(x) = \sum_{k = 1}^{K} 1_{x \in \cX_k} \h h_{k, T}(x)$, where
$\h h_{k, T}$ is the hypothesis chosen after $T$ rounds by the algorithm for region
$\cX_k$.

Let $\tp_k = \Pr(\cX_k)$ denote the probability of region $\cX_k$ with
respect to $\sD_\cX$, $k \in [K]$, and let $R_k(h)$ denote the
conditional expected loss of a hypothesis $h$ on region $\cX_k$,
that is $R_k(h) = \E_{(x, y) \sim \sD}[\ell(h(x), y) | x \in \cX_k]$.  By
definition, we have $R(h) = \sum_{k = 1}^{K} \tp_k R_k(h)$ for any
hypothesis $h$. 
We assume the learner has access to large amounts of \emph{unlabeled}
data, which can be used to accurately estimate $\tp_k$.  In fact, our
results can be easily adapted to the case where the $\tp_k$s are
estimated via a collection of unlabeled examples requested on-the-fly.
While this would not add much to our analysis in terms of technical
difficulty, it would make the entire theoretical effort unnecessarily
more cluttered.

We denote by $h^* \in \sH$ the overall best-in-class hypothesis over $\cX$ (single region before any splitting)
 and by
$h_k^* \in \sH$ the $k$-th region's best-in-class hypothesis, that is,
$h^* = \argmin_{h \in \sH} R(h)$ and
$h_k^* = \argmin_{h \in \sH} R_k(h)$. We will also use as shorthand
the following notation: $R^* = R(h^*)$ and $R_k^* = R_k(h^*_k)$.

Observe that minimizing the generalization error within each region
$\cX_k$ individually is equivalent to minimizing the overall error
over the larger set
$\sH_{[K]}= \big\{ \sum_{k = 1}^{K} \one_{x \in \cX_k} h_k(x)
\colon h_k \in \sH \big\}$. Clearly, the performance of the
best predictor in $\sH_{[K]}$ is always at least as favorable as that
of the best predictor in $\sH$, but it can be considerably better,
especially when the algorithm chooses a large $K$, or when the local
performances of $h_k^*$s with large $\tp_k$ are substantially superior
to that of $h^*$ on the same regions.

\vskip -0.1in
\section{Algorithm}
\vskip -0.05in
\label{sec:alg}

\ignore{
In this section, we describe an 
We will show in the next section that this algorithm benefits
from favorable generalization and label complexity guarantees.
}

Our algorithm, called \arbal\, (Adaptive Region-Based Active Learning),
is an on-line active learning algorithm 
that \emph{adaptively} partitions the input space into subregions. 
\arbal\, adopts a label requesting policy similar to
that of the single-region \iwal\ algorithm of \citet{beygelzimer2009importance},
which is based on the largest possible disagreement among the current
set of hypotheses on the current input: at round $t$, given the hypothesis
set $\sH_t$ and input point $x_t \in \cX$, \iwal\ flips a coin
$Q_t \in \set{0, 1}$ with bias $p_t = p(x_t)$ defined by
\begin{equation*}
  p_t = \max_{h, h' \in \sH_t, y \in \cY} \ell(h(x_t), y) - \ell(h'(x_t), y). 
\end{equation*}
If $Q_t = 1$, then the label of $x_t$ is requested and the algorithm
receives $y_t$, otherwise no label is revealed. Since the loss
function $\ell$ takes values in $[0, 1]$, the requesting probability
$p_t \in [0, 1]$ is well defined. \iwal\ then seeks to shrink the
current set $\sH_t$ to reduce the querying probability $p_t$ for
future inputs, while, at the same time, keeping (with high
probability) the overall best-in-class hypothesis in this set.
At the end of $T$ rounds, \iwal\ returns the importance-weighted
empirical risk minimizer
$\h h_T = \argmin_{h\in\cH_T} \sum_{t=1}^T Q_t \ell(h(x_t), y_t)/
p_t$.

Our techniques and ideas for splitting are illustrated with \iwal, 
since \iwal\ works with any hypothesis set and bounded loss function,
and admits generalization guarantees with no distributional assumption.
In contrast, CAL \citep{cohn1994improving} assumes a separable case; 
DHM \citep{dasgupta2008general} and A$^2$ \citep{Balcan2006AgnosticAL} 
are designed for the 0-1 loss, and many other margin-based algorithms 
only work for linear classifiers. 
Furthermore, for the separable case 
($R^* = 0$), the recent work of \citep{cortes2019rbal} proposes an
enhanced version of \iwal, called \eiwal\,, whose label complexity is in
the order of $\log\big( \frac{|\sH|}{\epsilon} \big)$, thereby matching
the bound of CAL and DHM. That being said, our techniques can be easily 
applied to other algorithms available in the literature, so long as they have valid 
concentration bounds, such as Corollary 1 of the DHM paper \citep{dasgupta2008general},
and Theorem 1 of \citet{cortes2019rbal}.
In that case, we just need to change the splitting criterion accordingly,
and our theoretical analysis can then be easily adapted to the new
concentration bound.

Our algorithm can be viewed as an adaptive region-based version of
\iwal, where the label requesting policy just described and the
shrinking procedure are applied at the regional level.  As already
mentioned, the following
questions arise when designing the algorithm: (1) How should
we determine the regions? (2) Can we learn to adaptively partition the
input space into favorable subregions, using actively requested
labels? We now explicitly address both questions and describe
our algorithm in detail.

The pseudocode of \arbal\ is given in Algorithm~\ref{alg:arbal}. The
algorithm admits two phases: in the first phase (\emph{split phase}),
the algorithm partitions the input space into $K$
disjoint regions while actively requesting labels according to \iwal's policy
on the regional level. 
This phase is constrained by two input parameters:
$\kappa$ limits the maximum number of regions generated ($K \leq \kappa$),
and $\tau$ caps the maximal number of online rounds for this phase.
Section~\ref{subsec:split} describes in detail
the main subroutine of this phase, \splitregion\
(Algorithm~\ref{alg:split}), including the splitting conditions that
guarantee a significant improvement in generalization ability
resulting from the split. Whenever the algorithm decides to split,
each resulting region is equipped with a copy of the original
hypothesis set $\sH$. Notice that the algorithm actively selects
labels in this phase, even if it does not shrink the hypothesis set(s),
and thus it still requests fewer labels than passive learning.
For simplicity, the regions will be axis-aligned
rectangles, though more convoluted splitting shapes are clearly possible
(see Section \ref{sec:experiments}).

In the second phase (\emph{\iwal\ phase}), \arbal\ runs \iwal\
separately on each of the regions produced by the first phase,
to learn a good predictor per region. 
After $T$ rounds, \arbal\ returns $\h h_T$, which combines
region-specific importance-weighted empirical risk minimizers
$\h h_{k, T}$.
In Section~\ref{subsec:rbal}, we describe the \iwal\ phase, and
discuss its connections to \oriwal\  \citep{cortes2019rbal}.

%\begin{figure}[ht!]
%\begin{minipage}[t]{0.48\textwidth}
\begin{algorithm}[t!]
  \small
  \caption{\arbal$(\sH, \tau, \kappa, (\gamma_t)_{t \in [T]} )$} 
  \label{alg:arbal}
  \begin{algorithmic}
    \STATE $K \gets 1$, $\cX_1 \gets \cX$, $\cH_1 \gets \cH$
    \FOR {$t \in [T]$}
    \STATE Observe $x_t$; set $k_t \gets k$ such that $x_t \in \cX_k$
      \STATE $p_t \gets \displaystyle \max_{h,h' \in \sH_{k_t}, y\in\cY} 
    \ell(h(x_t),y_t) - \ell(h'(x_t),y_t)$
    \STATE $Q_t \gets \textsc{Bernoulli}(p_t)$
    \IF {$Q_t=1$}
    \STATE $y_t \gets \textsc{Label}(x_t)$
    \ENDIF
    \IF {$t\leq \tau$ and $K<\kappa$} 
        \STATE $\cX_l,\cX_r \gets \splitregion(\cX_{k_t}, \gamma_t)$ \COMMENT{\texttt{\small \small split phase}}
    \IF {split}
    \STATE $K\gets K+1$,
    $\cX_{k_t} \gets \cX_l$, $\cX_{K} \gets \cX_r$ 
    \STATE $\cH_{K} \gets \cH$, $\cH_{k_t} \gets \cH$ 
    \ENDIF
    \ELSE 
    \STATE $\sH_{k_t} \gets \textsc{Update}(\sH_{k_t})$ \COMMENT{\texttt{\small \small IWAL phase}}
    \ENDIF
    \ENDFOR
    \STATE {\bf return $ \h h_T \gets \sum_{k=1}^{K} 1_{x\in \cX_k} \h h_{k, T}$} 
  \end{algorithmic}
\end{algorithm}
%\end{minipage}
\hfill
%\begin{minipage}[t]{0.48\textwidth}
\begin{algorithm}
  \small
  \caption{\splitregion$(\cX_k, \gamma)$} 
  \label{alg:split}
  \begin{algorithmic}
    \FOR {$d\in [D]$ and $c\in \Rset$}
    \STATE $(\cX_l, \cX_r) \gets \textsc{RegSplit}(\cX_k, d, c)$
    \STATE $\gamma_{d, c} \gets \tp_{k} 
    \Big[ L_{k, t}(\h h_{k, t}) - L_{k, t}(\h h_{lr,t}) -
    \sqrt{\frac{2 \slack}{T_{k,t}}}\Big]$ %\eta_{k, t} \Big]$
    \ENDFOR
    \STATE $(d^*, c^*) \gets \argmax_{d\in[D], c\in \Rset} \gamma_{d, c}$
    \IF {$\gamma_{d^*, c^*} \geq \gamma$ }
    \STATE $\cX_l^* \gets \set{x\in\cX_k \colon x[d^*] \leq c^*}$ \COMMENT{\texttt{\small \small split}}
    \STATE $\cX_r^* \gets \set{x\in\cX_k \colon x[d^*] > c^*}$
    \STATE \textbf{return} $\cX_l^*,\cX_r^*$
    \ELSE
    \STATE \textbf{return} $\emptyset$ \COMMENT{\texttt{\small \small no split}}
    \ENDIF
  \end{algorithmic}
\end{algorithm}
%\end{minipage}
\vskip -0.2in
\ignore
{
\caption{The \arbal\, algorithm. In the pseudocode of \splitregion, the subroutine 
$\textsc{\small RegSplit}(\cX_k, d, c)$ computes the axis-aligned partitioning of region $\cX_k$
  according to coordinate $d$ and threshold $c$, and 
  $\eta_{k, t} = \sqrt{{2\kappa D \log(8|\sH|^3 T_{k, t}
      (T_{k, t} + 1)\kappa TD/\delta)}/{T_{k, t}}}$, where $T_{k,t}$ is defined in Section \ref{subsec:split}.}
}
%\vskip -0.2in
%\end{figure}

One question naturally arises: Why do we separate the learning horizon
into two phases, where we first determine the partition, and then
perform region-based learning?  Given all possible partitions of the
input space, why not running \iwal\ with the family of hypotheses
containing all possible partitions of $\cX$ with leaf predictors
$h_k \in \cH$, that is,
$\Hset = \set[\big]{\sum_{k = 1}^\kappa 1_{x \in \cX_k} h_k \colon h_k
  \in \cH, \cup_{k = 1}^\kappa \cX_k = \cX, \cX_k \cap \cX_{k'} =
  \emptyset \text{ for } k\ne k'}$?  First, $\Hset$ is an exceedingly
complex hypothesis set, whose complexity can lead to vacuous learning
guarantees. Second, its computational cost makes it prohibitive to use
with \iwal.  Moreover, even if we fix the
partition and only vary the predictors in the leaf nodes, as proven in
Appendix~\ref{app:disagreement}, running \iwal\ with $\Hset$ may cost
up to $\kappa$ times more labels than running \iwal\ separately within
each partitioned region.  For all these reasons, we adopt the
  two-phases learning framework.

\vskip -0.1in
\subsection{\splitregion\ phase}
\label{subsec:split}
\vskip -0.05in

The advantage of region-based learning hinges on the improvement in
the best-in-class error after each split, which motivates 
our splitting subroutine: \splitregion\ 
splits a region if and only if the best-in-class error is likely to improve 
by a strictly positive amount. We will show in Corollary~\ref{thm:split} 
that, with high-probability,
the best-in-class error is guaranteed to decrease from each split.

The pseudocode of \splitregion\ is given in
Algorithm~\ref{alg:split}. 
At time $t$, \splitregion\ searches for the most favorable choice of the 
splitting parameters $(d, c)$ as follows.
Adopting the axis-aligned
splitting method commonly used for (binary) decision trees:
for a fixed pair $(d, c)$, the algorithm calls subroutine 
$\textsc{\small RegSplit}(\cX_k, d, c)$ to split
$\cX_{k}$ into a left region $\cX_l$ ($x_d\leq c$) and a right region $\cX_r$ ($x_d > c$),
and then computes a \emph{confidence gap}
$\gamma_{d, c}$ as defined in Algorithm~\ref{alg:split},
where $L_{k, t}(h)$ denotes the importance-weighted empirical risk of
hypothesis $h$ on region $\cX_k$,
\[L_{k, t}(h) = \frac{1}{T_{k, t}} \sum_{s \in [t], x_s \in \cX_k }
\frac{Q_s}{p_s} \, \ell(h(x_s),y_s) ,\] 
where $T_{k, t} = |\set{s \in [t] \colon x_s \in \cX_k}|$ is the number
of samples that have been observed in region $\cX_k$ up to time $t$, 
and $\slack = \kappa D \log\Big[\mfrac{8T^3|\sH|^3 \kappa D}{\delta}\Big]$ denotes the slack term.
Furthermore,
$\h h_{k, t} =\argmin_{h \in \sH} L_{k, t}(h)$ denotes the empirical risk
minimizer (ERM) on $\cX_{k_t}$.
Similarly, $\h h_{l,t}$ and $\h h_{r,t} $ denote the ERM of
region $\cX_l$ and $\cX_r$, respectively, and 
$\h h_{lr, t} = 1_{x \in \cX_l} \h h_{l, t} + 1_{x \in \cX_r} \h h_{r,
t}$ is the combination of the two region-specific ERMs.   
The confidence gap $\gamma_{d,c}$ serves as a conservative estimate of 
the improvement in the best-in-class error.
\splitregion\ searches for the maximum confidence gap
over all distinct pairs:
$(d^*, c^*) = \argmax_{d, c} \gamma_{d, c}$. When $\gamma_{d^*, c^*}$
is larger than the pre-specified threshold parameter $\gamma$,
it splits with $(d^*, c^*)$ and allocates to the two newly created 
regions the same initial hypothesis set $\sH$ (see Algorithm~\ref{alg:arbal}),
otherwise it does not split.

To implement the \splitregion\ subroutine, we maintain an array of region labels of past samples,
and $D$ sorted arrays of past samples according to each of the $D$ coordinates.
At time $t$, for each coordinate $d\in[D]$, it takes $O(\log(t))$ to insert 
$x_t$ into the sorted array, and $O(t)$ to compute the key term $L_{k,t}$ for all 
$t+1$ splitting thresholds on the sorted array.
Here, we use the fact that although there are infinitely many
possible splitting threshold values, we only need to consider $t + 1$ many 
thresholds to distinguish the $t$ feature values $\set{x_{s, d}\colon s \in [t]}$.
It also takes $O(t)$ to update the region labels of past samples after split, and
thus a total of $O(tD)$ to run \splitregion\ at time $t$.
Furthermore, as already mentioned in Section~\ref{sec:prelim}, we assume access 
to a set $U$ of unlabeled samples to estimate all the $\tp_k$s. To do so, we maintain a binary
tree corresponding to the splits. A new split converts a leaf node
$u_i$ with number of elements $|u_i| \leq |U|$
into an internal node at the cost of $O(|u_i|)$. The cost of updating the tree for the $\kappa$ splits in order to
estimate all $\tp_k$s  is hence $O(\sum_{i=1}^\kappa |u_i|)$, 
where the sum is over all
the internal nodes of the tree.

Alternatively, these probabilities can be
estimated incrementally during the on-line execution of the algorithm, and our theoretical analysis 
can be extended along these lines using a union bound similar to the one 
in Lemma~\ref{lemma:split_concentrate}'s proof. 
\ignore
{
Nevertheless, in our experiments in 
Section~\ref{sec:experiments}, the $\tp_k$s have been
estimated on a held-out set of unlabeled samples collected before 
running the on-line active learning algorithms.
}

We now introduce some additional notation before discussing the theoretical guarantees of the \splitregion\
algorithm. Let $h_k^*, h_{l}^*, h_{r}^*$ be the 
best-in-class predictors on region $\cX_k$, $\cX_l$, and $\cX_r$, respectively,
and denote by $h_{lr}^* = 1_{x \in \cX_l} h_{l} ^* + 1_{x \in \cX_r} h_{r}^*$.
Then, the improvement in the best-in-class error after this split is
$\tp_k [R_k(h_k^*) - R_k(h_{lr}^*)]$.  The following
concentration lemma relates the improvement in the best-in-class error
to its empirical counterparts, which leads to the theoretical guarantee 
for the \splitregion\ subroutine (Corollary~\ref{thm:split}).
Its proof uses a martingale concentration bound, as well as covering number techniques 
to guarantee that the high-probability bound holds uniformly for any possible sequence of splitting.
The proof is given in Appendix~\ref{app:proof}.

\begin{lemma}
\label{lemma:split_concentrate}
  With probability at least $1 - \delta/4$, for all binary trees with (at most) $\kappa$ leaf nodes,
  the improvement in the minimal empirical error by splitting 
  concentrates around the improvement in the best-in-class error:
   \begin{align*}
     & \Big|
     \big[R_k(h_k^*) \!-\! R_k(h_{lr}^*)\big] \!-\! 
     \big[L_{k, t}(\h h_{k, t}) \!-\! L_{k, t}(\h h_{lr,t})\big]\Big|  
     \! \leq \!\sqrt{\mfrac{2\slack}{T_{k,t}}}.
     \ignore{&\leq \sqrt{{2\kappa D \log\big(2|\sH|^3 T_{k, t} (T_{k, t}+1) 
     \kappa TD/\delta\big)} / {T_{k, t}}}.}
   \end{align*}
 \end{lemma}

\begin{corollary}
\label{thm:split}
  With probability at least $1-\delta/4$, for all splits made by \arbal,
  the improvement in the best-in-class error is at least $\gamma_t$,
  where $\gamma_t$ is the threshold at {the time of split}.
\ignore{Let region $\cX_{k}$ be split at time $t$ into $\cX_l$ and $\cX_r$ with
threshold $\gamma$. 
Then with probability at least $1 - \delta$ the improvement in the error of the best in class
satisfies $\tp_k \big[R_k(h_k^*) - R_k(h_{lr}^*) \big] \geq \gamma$.}
\end{corollary}
Corollary~\ref{thm:split} guarantees that, with high-probability, 
whenever \arbal\ splits, the best-in-class error is strictly improved by at least $\gamma_t>0$.
This yields the fundamental advantage of region-based learning.

One challenge \arbal\ faces is that, whenever it chooses to split, it
commits to competing against a more accurate predictor, that is the
region-specific best-in-class hypothesis on the refined regions. To
ensure success, we need to guarantee not only that the best-in-class
over the current region or those over subregions after the split are
not pruned out, but also that the best-in-class hypothesis over any
\emph{future}\/ region produced after further splitting remains in the
hypothesis space that will be given as input to \arbal's second phase.

One can show that, if \arbal\ prunes out some hypotheses before the
split phase has ended, it may lose the future best-in-class predictor,
and thus fail dramatically.  As a simple example, consider the binary
classification problem depicted in Figure~\ref{fig:region-bic}, where
the unlabeled data is uniformly distributed within a square, and the
true classification boundary admits a zig-zag shape
(the left plot of Figure~\ref{fig:region-bic}).
If the learner uses the class of linear separators as the initial hypothesis
set $\sH$, then, after receiving a certain number of labeled samples,
it finds that the best performing hypothesis is
approximately the diagonal separator from bottom left to top right. 
Suppose the algorithm would now trim $\sH$ to only maintain
separators performing similarly to the diagonal separator, with
decision surfaces indicated by the shaded area in the middle plot of
Figure~\ref{fig:region-bic}. If later on, the learner splits
the input space (the square) into two regions (left and right
rectangles in the right plot of Figure~\ref{fig:region-bic}), then the
best-in-class separators for the two rectangles are horizontal
separators. Clearly, the two horizontal best-in-class separators are
not contained in the current $\sH$ (which is meant to apply to the
entire input space).  In summary, trimming $\sH$ before
making splits introduces the risk of losing the best-in-class
separators on the partitioned regions.  This is the reason why \arbal\
maintains throughout the split phase the original hypothesis space
$\sH$. The shrinkage of $\sH$ only takes place during the \iwal\
phase, presented next.

\begin{figure}[t!]
 \begin{center}
     \centerline{\includegraphics[width=\columnwidth]{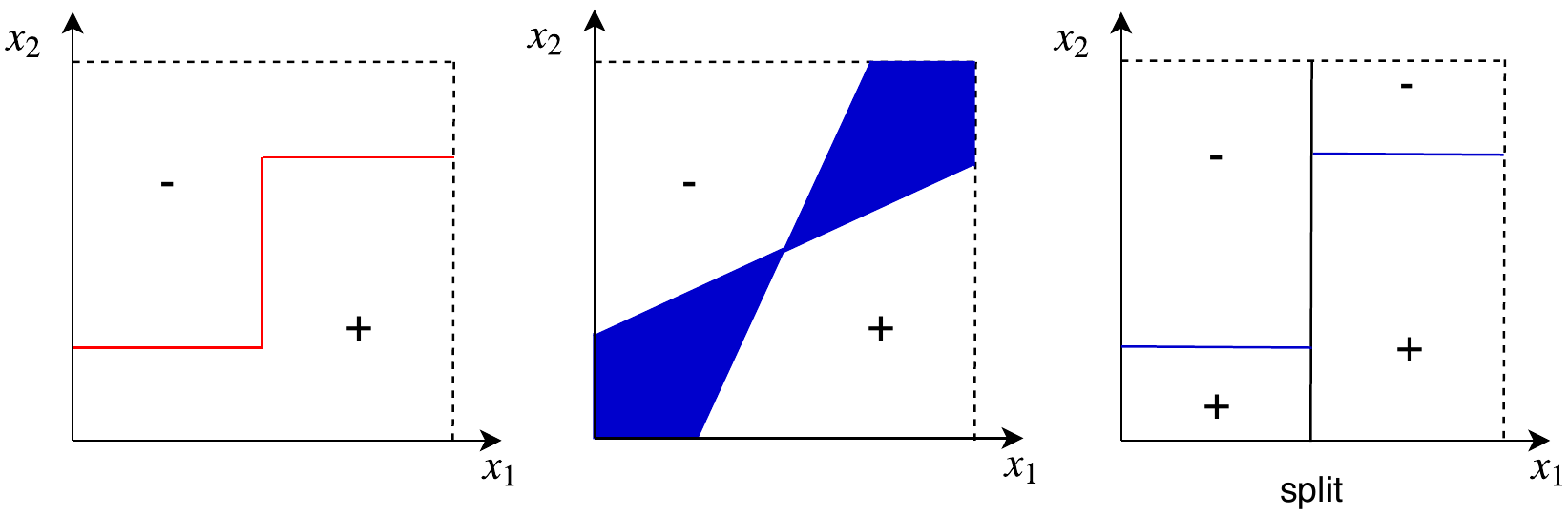}}
     \caption{ The input space $\cX$ is a (2-dimensional) square. Left:
   the true classification boundary (in
   \textcolor{red}{red}) as a function of $x_1$ and $x_2$. Middle:
   (approximately) the current hypothesis space
   (\textcolor{blue}{shaded blue} area) after trimming an initial set
   of linear separators given several labeled samples.
   Right: the best-in-class predictors (in \textcolor{blue}{blue}) when the
   input space $\cX$ splits into two regions $\cX_l$ and $\cX_r$ from
   the middle of $x_1$. }
\label{fig:region-bic}
 \end{center}
  \vskip -0.2in
\end{figure}

\vskip -0.1in
\subsection{\iwal\ phase}
\label{subsec:rbal}
\vskip -0.05in

In this phase, with the regions $\cX_1,\ldots, \cX_{K}$ being fixed,
\arbal\ runs a separate \iwal\ subroutine on each one of them,
requesting labels and reducing the hypothesis space from $\sH$ to
region-specific $\sH_{k}, \forall k\in[K]$. As the algorithm requests
labels, $\sH_{k}$ shrinks towards the best-in-class hypothesis on
region $\cX_k$. The hypothesis space is updated according to the
\iwal\ update rule, which is derived from the concentration bound.
Specifically, we
update the hypothesis space $\sH_{k, t}$ sitting on region $\cX_k$ at time $t$ by
\[
   \sH_{k, t} \!= \!\set[\Big]{\!h \in \sH_{k, t-1}\colon \!L_{k, t}(h) 
   \!\leq \! \min_{h\in \sH_{k, t-1}} \!\! L_{k, t}(h) 
   + \sqrt{\mfrac{8\slack}{T_{k, t}}}}.
\]
Thus, in this phase, \arbal\ freezes the regions
$\cX_1,\ldots, \cX_{K}$, allowing no further splits, and requests
labels and shrinks the set of hypotheses hosted by each such region.

Starting with a fixed partition makes the second
phase of \arbal\ very similar to the learning scenario recently
investigated by \citet{cortes2019rbal}, who proposed the \oriwal\
algorithm for this learning scenario.  In particular, during the
second phase, we could also run the \oriwal\ algorithm to achieve
additional improvement in generalization error.  Since \oriwal\ is
orthogonal to the main contribution of this paper, we do not discuss
it at length here.  \ignore{\citet{cortes2019rbal} additionally proposed an
enhanced version of \iwal, called \eiwal, which uses a more greedy
policy for pruning the hypothesis space, resulting in sharper
performance guarantees than the original \iwal, especially when the
best-in-class error $R(h^*)$ is small. We could in fact use \eiwal\
instead of \iwal\ in the second phase of \arbal\ on each
region. Extending the guarantees from \eiwal\ to \arbal\ is no
different than extending from \iwal\ to \arbal, hence for ease of
presentation this extension is omitted from this paper.}

\vskip -0.1in
\section{Theoretical analysis}
\label{sec:theory}
\vskip -0.05in

In this section, we present generalization error and label complexity 
guarantees for the \arbal\ algorithm. We first need some definitions and concepts from
\citet{beygelzimer2009importance}.  Define the distance $\rho(f, g)$
between two hypotheses $f, g \in \sH$ as
$\rho(f, g) = \E_{(x,y) \sim \sD} \left| \ell(f(x), y) - \ell(g(x), y)
\right|$.\footnote{This definition of $\rho(f,g)$ slightly differs
  from the original definition in \citet{beygelzimer2009importance},
  and it improves the label complexity bound of \iwal\ by a constant.
  See Appendix~\ref{app:iwal} for more details.}  \ignore{Given $r > 0$, let
$B(f, r)$ denote the ball of radius $r$ centered in $f \in \sH$:
$B(f, r) = \set{g \in \sH \colon \rho(f, g) \leq r}$.}  The generalized
disagreement coefficient $\theta(\sD, \sH)$ of a class of functions
$\sH$ with respect to distribution $\sD$ is defined as the minimum value of
$\theta$, such that for all $r > 0$,
\[
    \E_{x \sim \sD_\cX}\bigg[\sup_{h \in \cH,\rho(h,h^*)\leq r, y \in \cY}
\big| \ell(h(x), y) - \ell(h^{*}(x), y) \big| \bigg] \leq \theta r~.
\]
Since \arbal\ calls \iwal\ as a subroutine, the theoretical results of
\arbal\ directly depend on those of \iwal , which
are summarized in Theorem~\ref{thm:iwal} in Appendix~\ref{app:iwal}.

Recall the definition of the confidence gap $\gamma$ in
Algorithm~\ref{alg:split}, which is the minimum value of the confidence
gap $\gamma_{d,c}$ that allows \arbal\ to split a region. 
We discuss \arbal\ under two settings: using a fixed threshold
$\gamma$, and using a time-varying and data-dependent adaptive threshold $\gamma_t$.

\vskip -0.1in
\subsection{\arbal\ with a fixed $\gamma$}
\label{sec:theory1}
\vskip -0.05in

Suppose we run \arbal\ with a fixed threshold $\gamma$.  The label
complexity of the algorithm depends on the region-based disagreement
coefficient $\theta_k = \theta(\condsD_k, \sH)$, where
$\condsD_k = \sD|\cX_k$ is defined as the conditional distribution of
$x$ on region $\cX_k$.  Let
$\theta_{\max} = \max_{k\in K} \theta_k$ denote the maximum
disagreement coefficient across regions, and let
$r_0 = \max_{h \in \sH} \rho(h, h^*)$. 
Let $\cF_t$ denotes the $\sigma$-algebra generated by
$(x_1, y_1, Q_1), \ldots, (x_t, y_t ,Q_t)$. 
\begin{theorem}
\label{thm:splitiwal2}
Assume that a run of \arbal\ over $T$ rounds has split the input space into
$K$ regions.  Then, for any $\delta > 0$, with probability at least
$1 - \delta$, the following inequality holds:
  \begin{align*}
    R(\h h_T)
    \leq R_U +  \sqrt{\frac{32 K \slack}{T} }
    + \frac{16K\slack}{T},
  \end{align*}
where $R_U = R^* - \gamma (K-1)$ 
is an upper bound on the best-in-class error obtained by \arbal.
Moreover, with probability at least $1 - \delta$, the expected number of 
labels requested, $\tau_T= \sum_{t=1}^T \E_{x_t\sim\sD_{\cX}} \big[p_t |\cF_{t-1}\big]$,
satisfies
  \begin{align*}
    \tau_T
    & \leq \min \set{2\theta r_0, 1} \tau 
    + 4\theta_{\max}(T - \tau) 
    \Big[ R_U 
    + 8 \sqrt{\mfrac{K \slack}{T - \tau}} \Big] \\
& +  \sqrt{32}K\slack .
  \end{align*}
\end{theorem}
The proof is given in Appendix~\ref{app:proof}. It combines the
learning guarantee of \iwal\ (Theorem~\ref{thm:iwal}) with those for
splitting (Corollary~\ref{thm:split}).  Theorem~\ref{thm:splitiwal2} shows that, with high probability, the generalization
  error of the hypothesis returned by \arbal\ is close to $R_U = R^* -
  \gamma (K - 1)$, which
  is a more favorable benchmark than the single-region best-in class error
  $R^*$ by $\gamma (K - 1)$. 
  We will show later that, under natural assumptions, with high probability, there is at least
  one split, which implies $\gamma (K - 1) > 0$ (Proposition~\ref{lemma:first_split_time}). 
  Furthermore, the reduction in the best-in-class error
also improves label complexity: when $T \gg \tau$, the label
complexity of \arbal\ is $O\big(R_U T\big)$ compared to \iwal's
$O\big(R^* T)$.

In practice, we set $\gamma=\Omega(\sqrt{{\slack}/{T}}\,)$ to ensure
that the generalization bound in Theorem~\ref{thm:splitiwal2} is more
favorable than the generalization bound of \iwal\
(Theorem~\ref{thm:iwal}).  We give more details on how to set this
fixed $\gamma$ in Appendix~\ref{app:proof} (see comments following the
proof of Theorem~\ref{thm:splitiwal2}).

There is a critical trade-off when determining the key parameters $\tau$ and $\kappa$. 
With a larger $\tau$ and $\kappa$, \arbal\ is likely to split into more regions and thus admits a smaller $R_U$.
On the other hand, a larger $\tau$ means a longer split phase, 
where \arbal\ requests labels more often compared to the original \iwal\ algorithm
since \arbal\ does not shrink the hypothesis set $\sH$ during this phase,
and a larger $\kappa$ yields a larger $\slack$,
which slightly affects the generalization error.
Nevertheless, our experimental results show that larger values of $\tau$ and $\kappa$ almost always improve the 
final excess risk, at the expense of higher computational cost.

\vskip -0.1in
\subsection{\arbal\ with adaptive $\gamma_t$}
\label{subsec:adagamma}
\vskip -0.05in

The learning guarantees of Theorem~\ref{thm:splitiwal2}
depend on the number of regions $K$ defined by the algorithm.
Given any fixed value of $\gamma$, however, there is
no guarantee on the number of times \arbal\ splits within the first $\tau$ rounds
(the duration of the first phase). In the worst case when $K = 1$,
\arbal\ offers no improvement over \iwal, yet
\arbal\ requests more labels than \iwal\ during the initial $\tau$ rounds.

In this section, we show that by adopting a time-varying
and data-dependent splitting threshold $\gamma_t$, we can enable
\splitregion\ to split more often, and thus achieve an enhanced
performance guarantee. To do so, we make additional assumptions
on the potential gain of splitting.

Let $\cX_k$ be an intermediate region created during the split phase,
possibly the original input space $\cX$. Assume that for any
such $\cX_k$, there exists at least one way of splitting $\cX_k$ into
$\cX_l \cup \cX_r$ such that the conditional improvement in the
best-in-class error is at least $\gain$:
$ R_k(h_k^*) - R_k(h_{lr}^*) \geq \gain $, where $\gain>0$ is a
positive constant.
With this assumption, we can derive upper bounds on the time
 \arbal\ splits when run with a time-varying adaptive
$\gamma_t = \Pr(\cX_{k_t}) \gain/2$.
\begin{proposition}
\label{lemma:first_split_time}
  Let \arbal\ be run with $\gamma_t = \gain \Pr(\cX_{k_t}) /2$.  Then,
  for any $\delta > 0$, with probability at least $1 - \delta/2$, the
  first split occurs before round
  $\Big\lceil{2\slack \big(\frac 4 {\gain} +1 \big)^2}\Big\rceil$.
\end{proposition}
Thus, when $\tau\geq \big\lceil{2\Slack{2}
  \big(\frac 4 {\gain} + 1 \big)^2}\big\rceil$,
with high probability, \arbal\ will split $\cX$ and reduce
the best-in-class error by at least $\gain/2$, according to 
Proposition~\ref{lemma:first_split_time} and Corollary~\ref{thm:split}.
In Appendix~\ref{app:proof}, we prove a more general version 
(Lemma~\ref{lemma:split_time}) that upper bounds 
the time of split for all regions created during the split phase. 

If we further assume that the
splitting with at least $\gain$ improvement in the best-in-class error 
results in regions that are not too small,
i.e., $\min\{\Pr(\cX_l), \, \Pr(\cX_r)\} \geq c \Pr(\cX_k)$, with $0<c<0.5$,
then we can also prove a lower bound on the number of splits.
\begin{corollary}\label{cor:num_splits}
  Let \arbal\ run with $\gamma_t = \Pr(\cX_{k_t}) \gain/2$. Then, with probability 
  at least $1 - \delta/2$, \arbal\ splits more than 
  $\min \set[\Big]{\log_{1/c} \Big[\frac{\tau}
          {2 \slack
  \left(\frac{4}{\gain} + 1\right)^2}\Big], \kappa-1}$
  times by the end of the split phase.
\end{corollary}
Corollary~\ref{cor:num_splits} gives the minimal number of splits under
the assumptions made in this section.
It states that, as the duration of the split phase $\tau$ increases,
or as the conditional improvement $\gain$ increases,
or as the minimal proportion of subregion size $c$ increases,
\arbal\ tends to make more splits and therefore achieves a better 
generalization guarantee.
Note that the lower bound in Corollary~\ref{cor:num_splits} tends to be loose,
as it assumes that \arbal\ keeps splitting the smallest region,
which is unlikely to be the case in practice.
In Appendix~\ref{app:proof}, we combine Proposition~\ref{lemma:first_split_time} 
and Corollary~\ref{cor:num_splits} to give an upper bound on the final 
best-in-class error after the splits by \arbal.

Note that the true value of $\gain$ is the property of the underlying distribution, 
and to accurately estimate $\gain$ is an open question that is beyond the scope of this paper.
One practical solution is to explore $\rho$ on various orders of magnitude, e.g. $(0.1, 0.01) $ etc., 
such that \arbal\ makes a reasonable number of splits. 
We set $\gain=0.01$ in our experiments. 

\vskip -0.1in
\section{Experiments}
\label{sec:experiments}
\vskip -0.05in

In this section, we report the results of a series of experiments.
We tested 24 binary classification datasets from the UCI and openml repositories,
and also the MNIST dataset with 3 and 5 as the two classes,
which is standard binary classification task extracted from the MNIST dataset (e.g., \cite{ckd09}).
Table~\ref{tb:data_summary} in Appendix~\ref{app:moreexp} lists summary statistics for these datasets. 
For ease of experimental comparison, for datasets with large input
dimension $D$, we followed the preprocessing step in \cite{cortes2019rbal},
retaining only the first $10$ principal components of the original
feature vectors. Due to space limitations, in this section we show the 
results on several medium-sized datasets.
The results for the remaining datasets are provided in Appendix~\ref{app:moreexp}.
For each experiment, we randomly shuffled the dataset, ran the algorithms
on the first half of the data (so that the number of active learning rounds $T$ equals $N/2$),
and tested the classifier returned on the remaining half to measure
misclassification loss. We only showed results on the first $10^{3.5}\approx 3000$ requested labels,
which are enough to differentiate the performances among various algorithms.
We repeated this process $50$ times on each
dataset, and report average results with standard error across the 50
repetitions. 
We use the logistic loss function $\ell$ defined for all
$(x, y) \in \cX \times \cY$ and hypotheses $h\colon \cX \to \Rset$ by
$\ell(h(x), y) = \log(1 + e^{-yh(x)})$, which we then rescale to
$[0, 1]$. 
The initial hypothesis set $\sH$ consists of $3\mathord,000$ randomly
drawn hyperplanes with bounded norms.
As mentioned in Section~\ref{sec:theory1}, larger values of $\tau$ and $\kappa$ almost always yield
better final excess risk. Thus, we chose $\kappa=20$ and allow the first phase to run at most $\tau=800$ rounds 
so as to make \arbal\ fully split into the desired number of regions on almost all datasets.
Since the slack term $\slack$ derived from high-probability analyses are typically
overly conservative, we simply 
use $0.01/\sqrt{T_k}$ in the \splitregion\ subroutine.

\begin{figure*}[ht]
  \begin{center}
  \subfigure{\centering\includegraphics[width=0.24\textwidth,,trim= 5 10 10 5,clip=true]{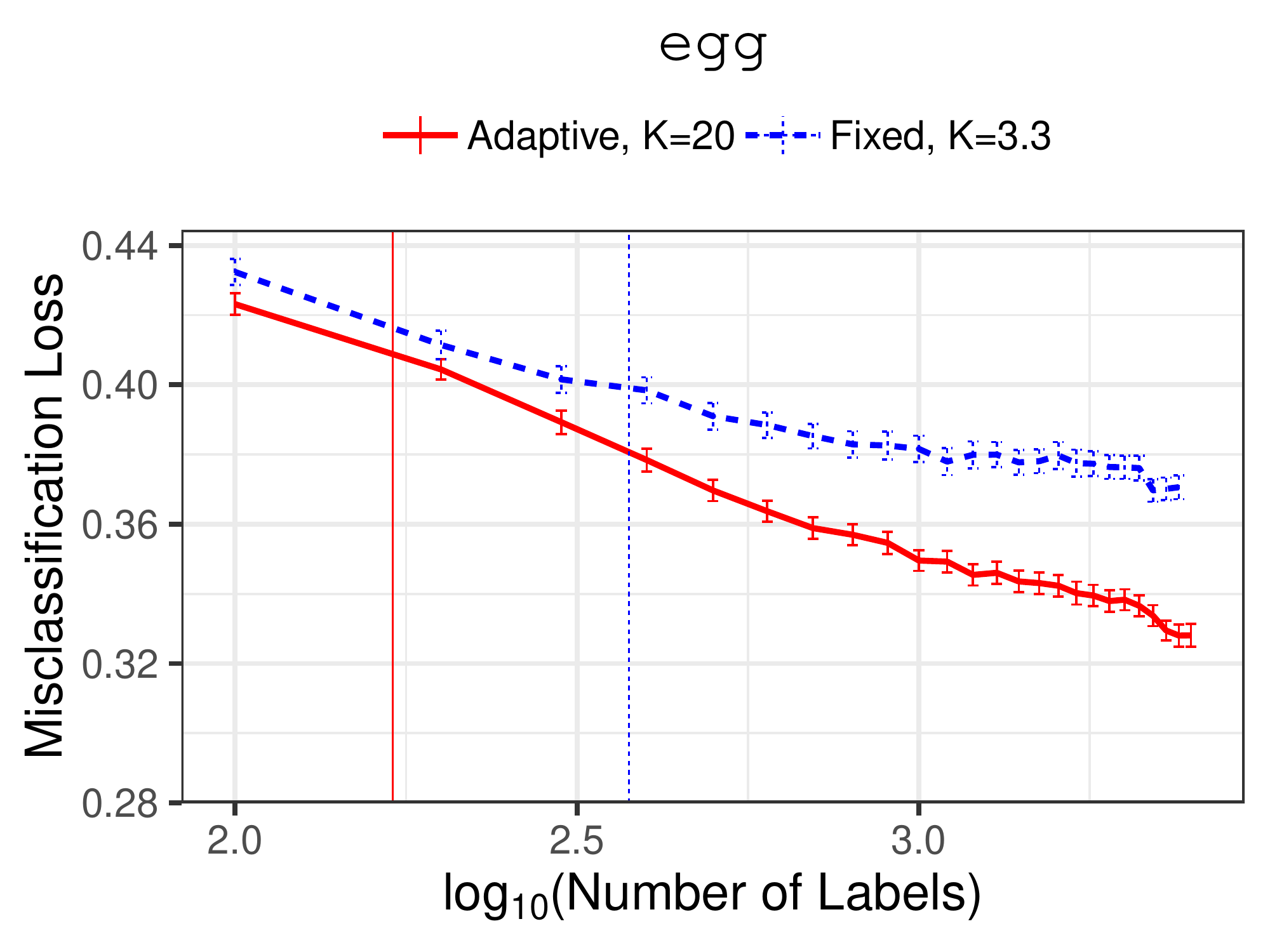}}
  \subfigure{\centering\includegraphics[width=0.24\textwidth,,trim= 5 10 10 5,clip=true]{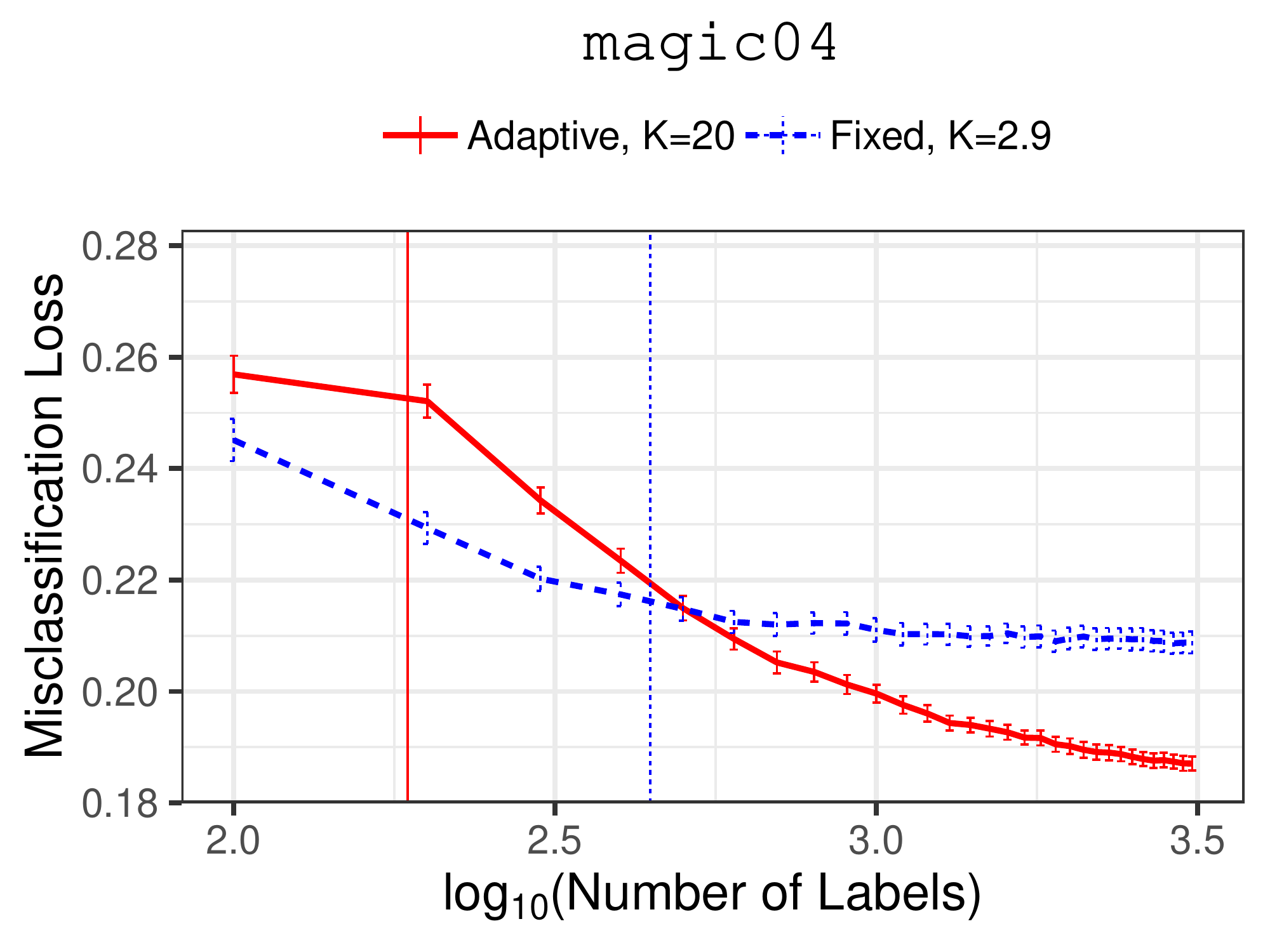}}
  \subfigure{\centering\includegraphics[width=0.24\textwidth,,trim= 5 10 10 5,clip=true]{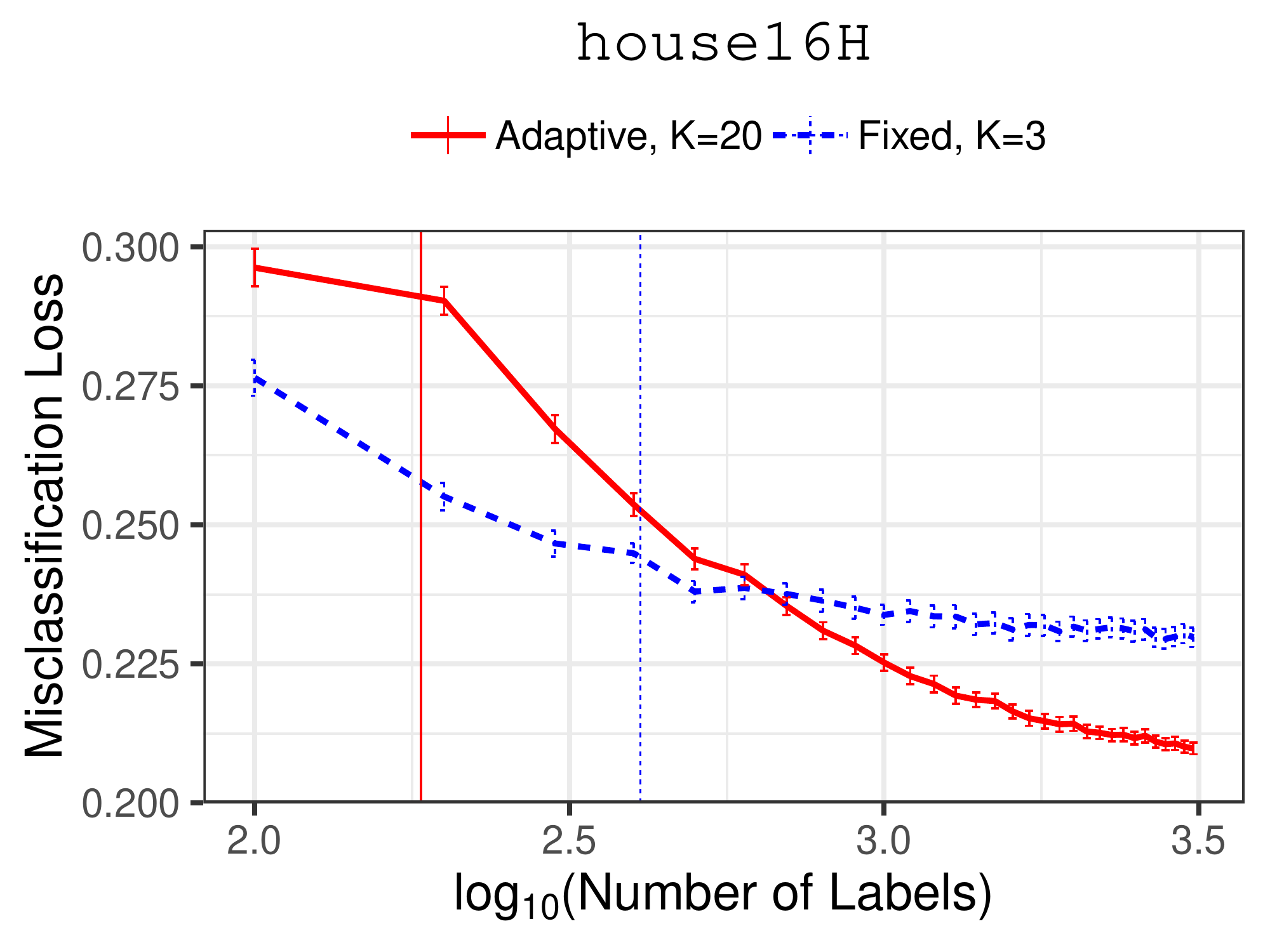}} 
  \subfigure{\centering\includegraphics[width=0.24\textwidth,,trim= 5 10 10 5,clip=true]{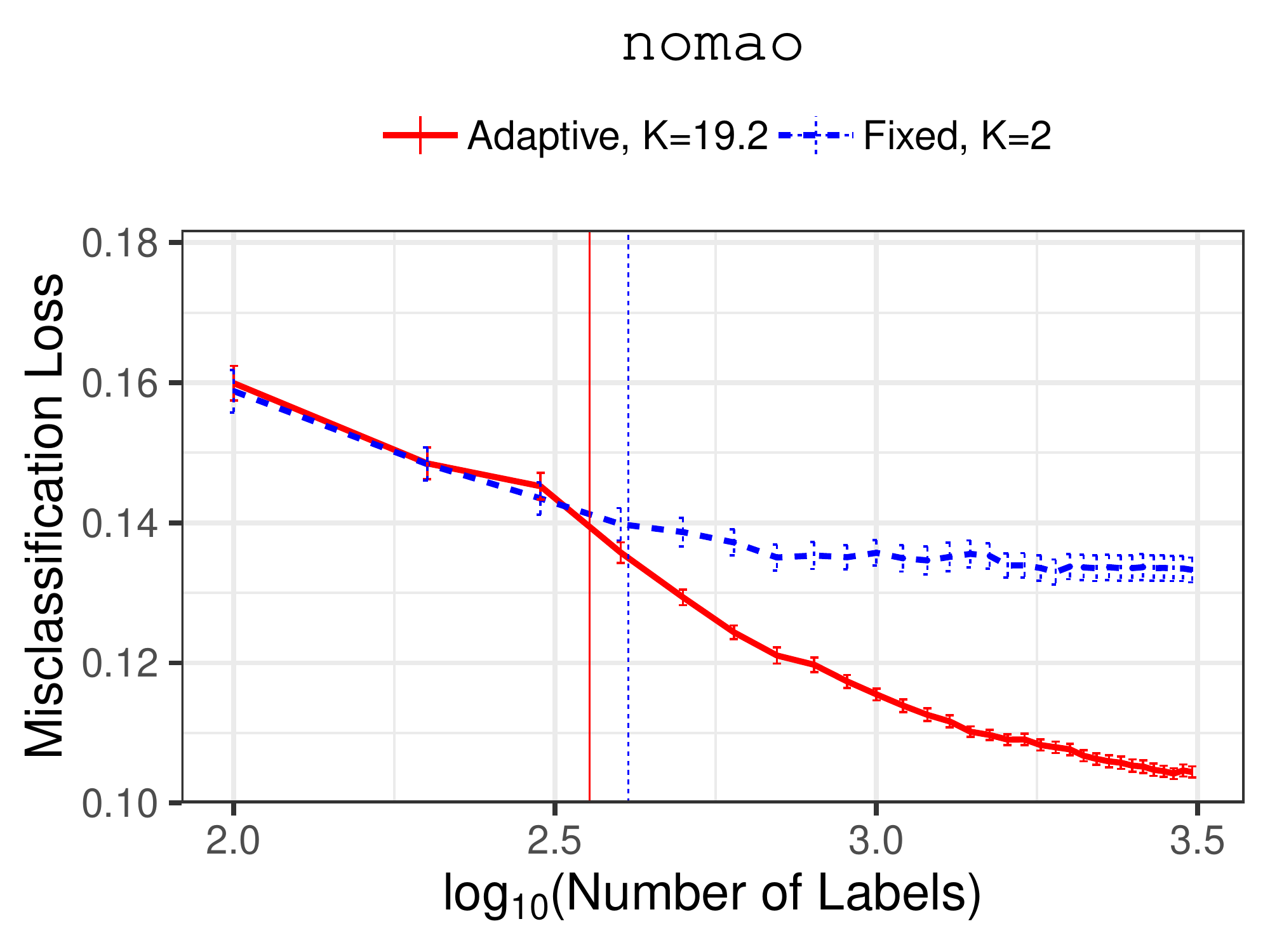}} 
  \end{center}
  \vskip -0.2in
  \caption{Misclassification loss of \arbal\ with fixed and adaptive
    threshold $\gamma$ on held out test data vs. number of labels
    requested ($\log_{10}$ scale), with $\kappa=20$ and $\tau=800$.  
     The vertical lines indicate the end of the first (split) phase.}
\label{fig:expmis_gamma_tau800k20_four}
\vskip -0.1in
\end{figure*}

\begin{figure*}[ht]
  \begin{center}
  \subfigure{\centering\includegraphics[width=0.24\textwidth,,trim= 5 10 10 5,clip=true]{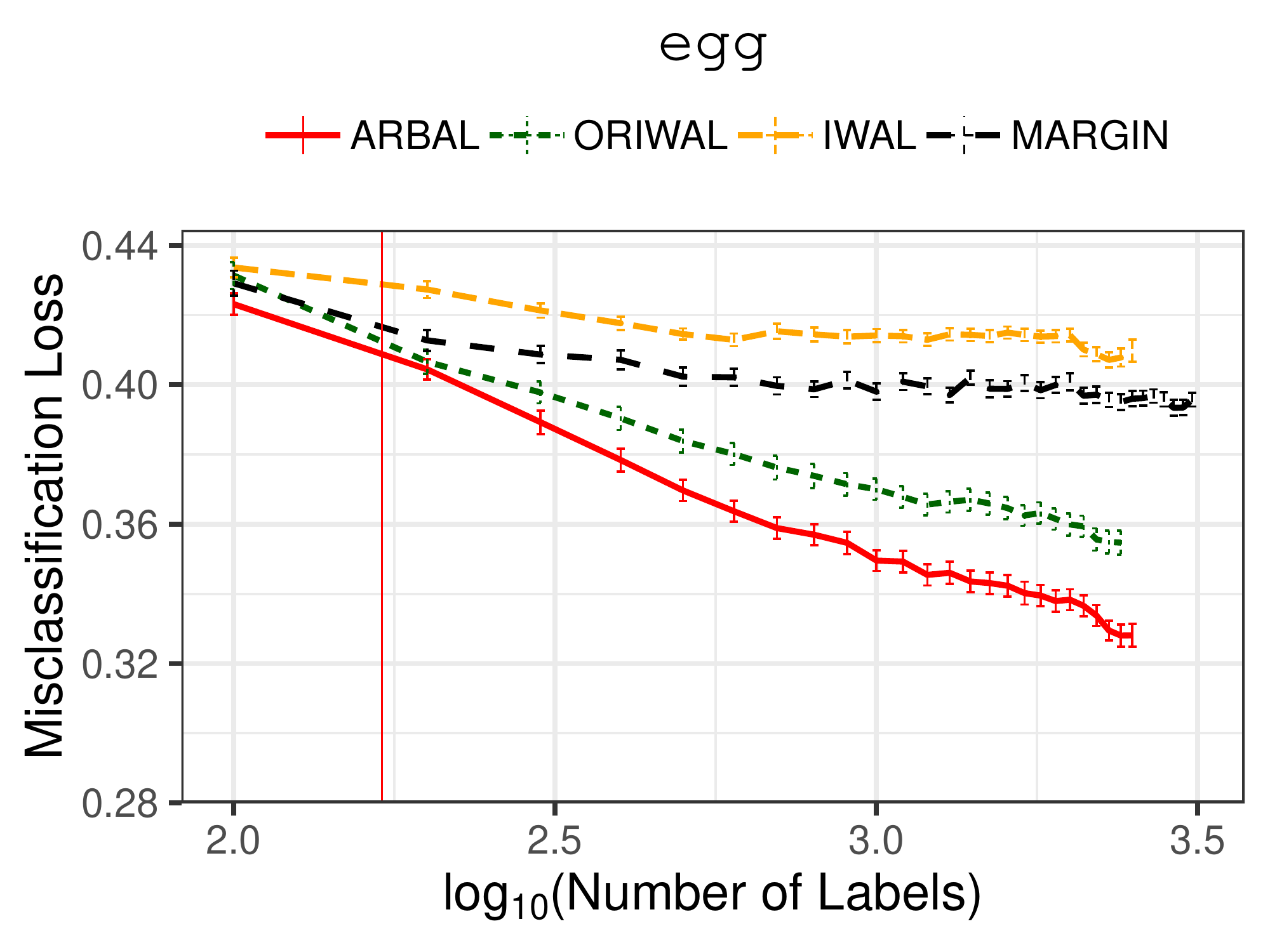}}
  \subfigure{\centering\includegraphics[width=0.24\textwidth,,trim= 5 10 10 5,clip=true]{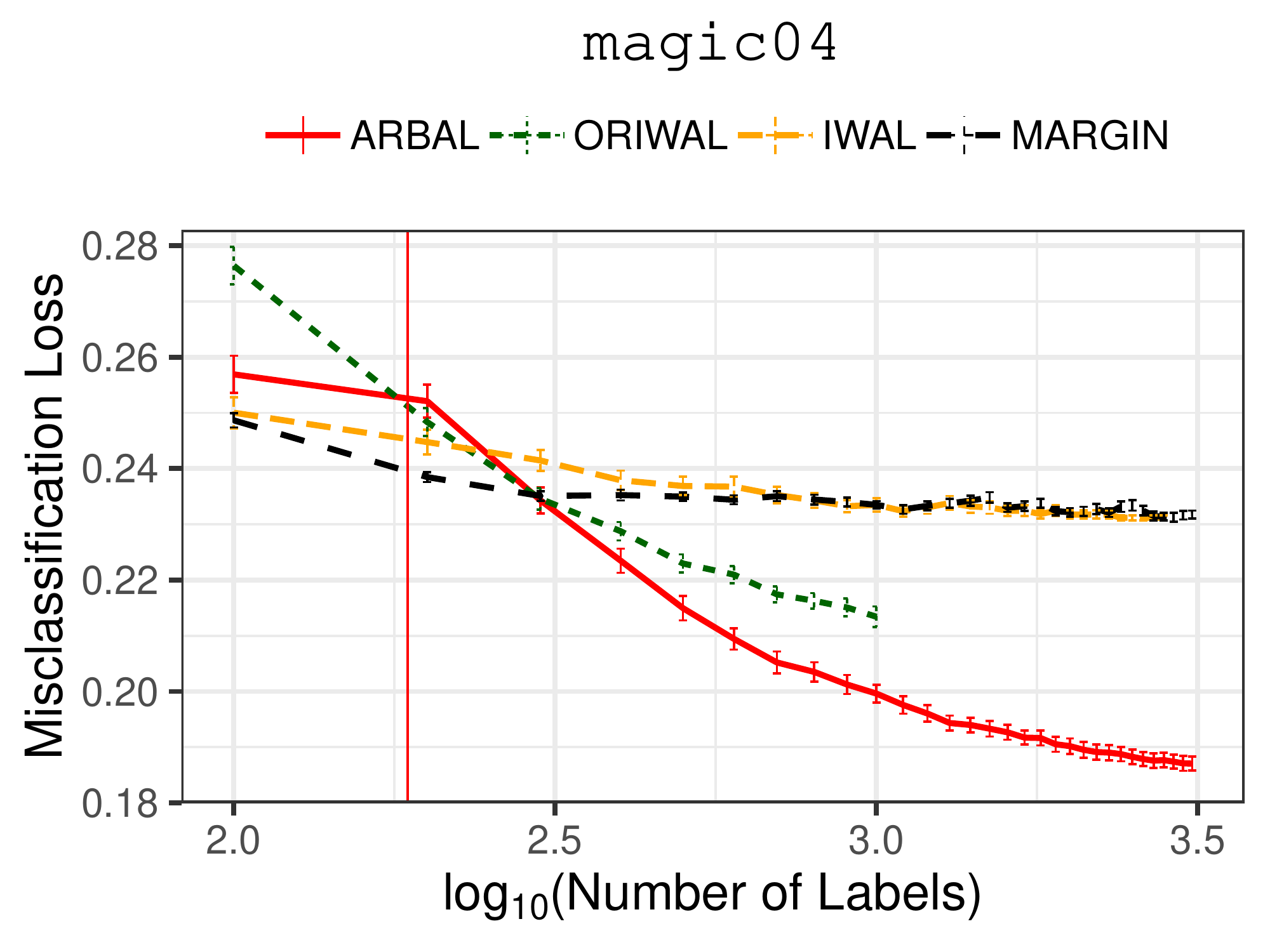}}
  \subfigure{\centering\includegraphics[width=0.24\textwidth,,trim= 5 10 10 5,clip=true]{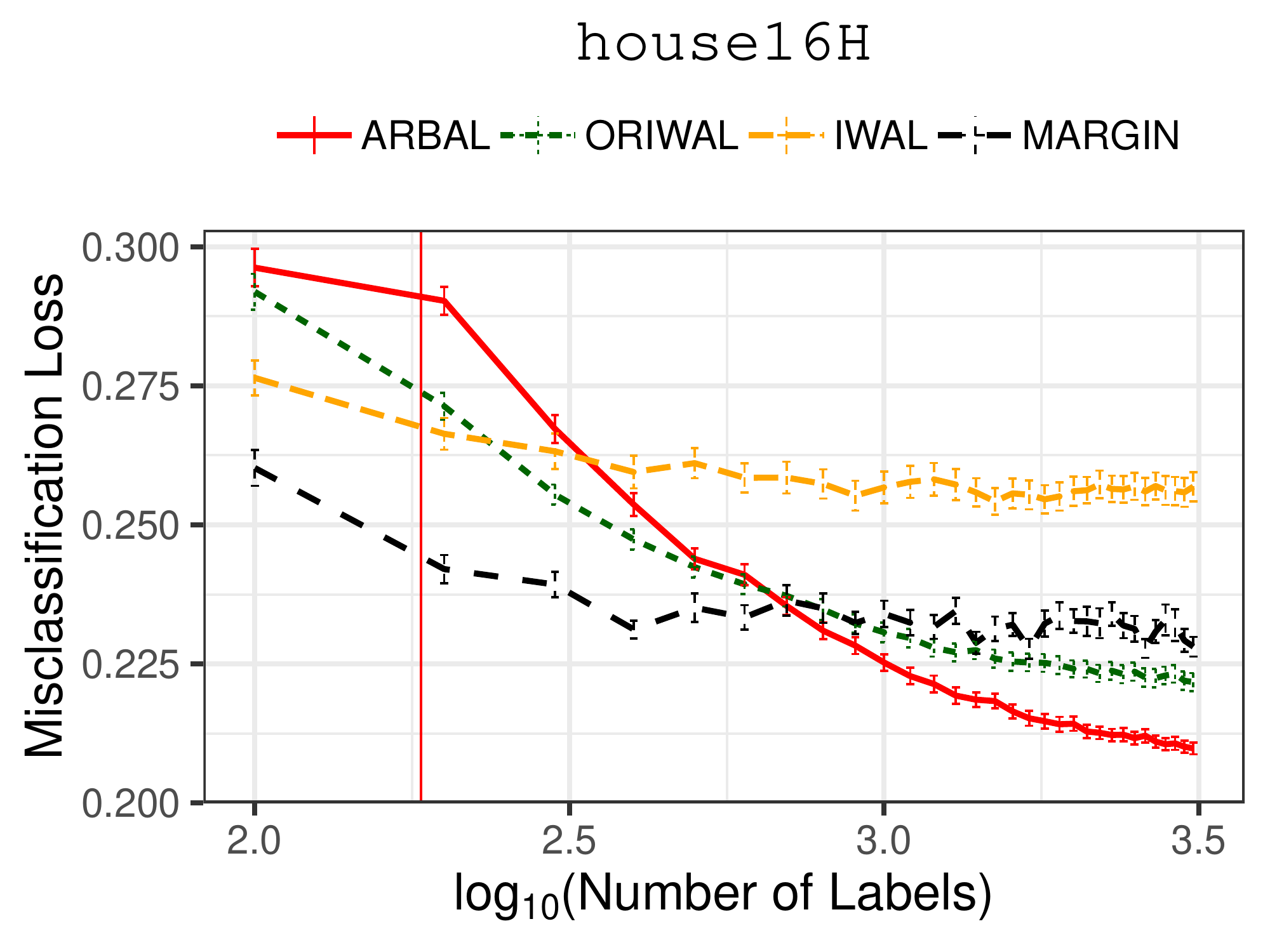}}
  \subfigure{\centering\includegraphics[width=0.24\textwidth,,trim= 5 10 10 5,clip=true]{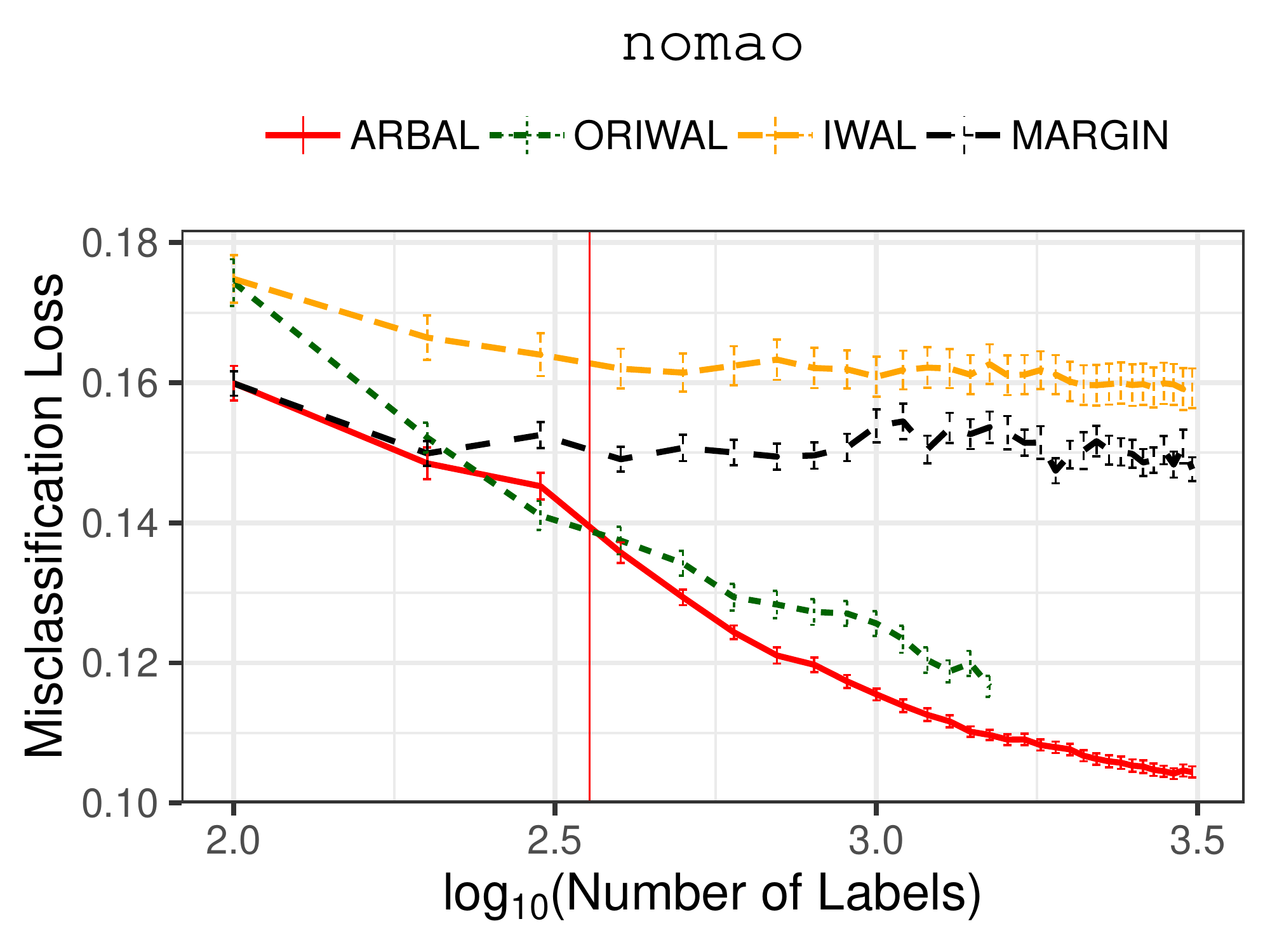}}\\
  \vskip -0.24in
  \subfigure{\centering\includegraphics[width=0.24\textwidth,,trim= 5 10 10 5,clip=true]{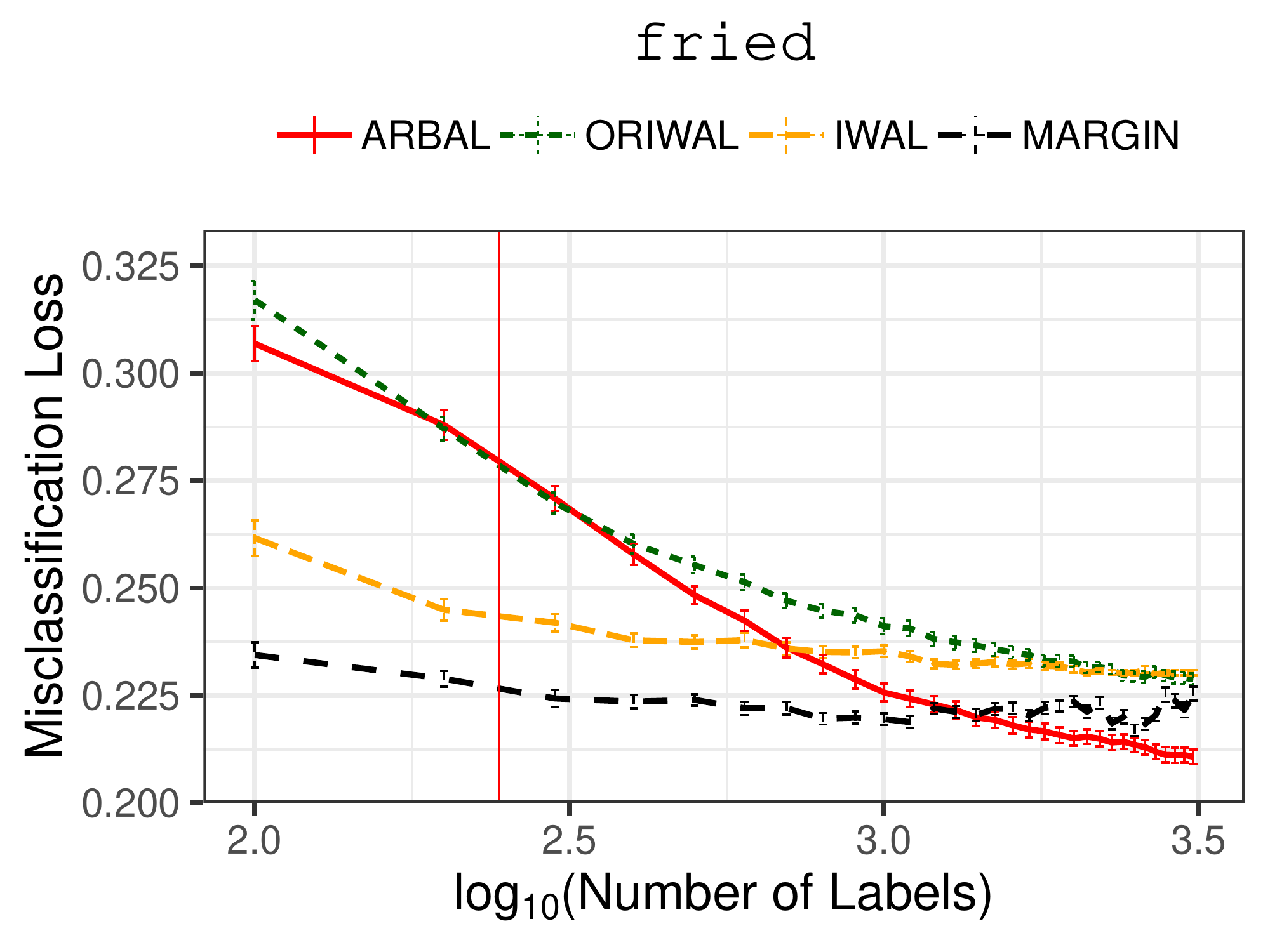}}
  \subfigure{\centering\includegraphics[width=0.24\textwidth,,trim= 5 10 10 5,clip=true]{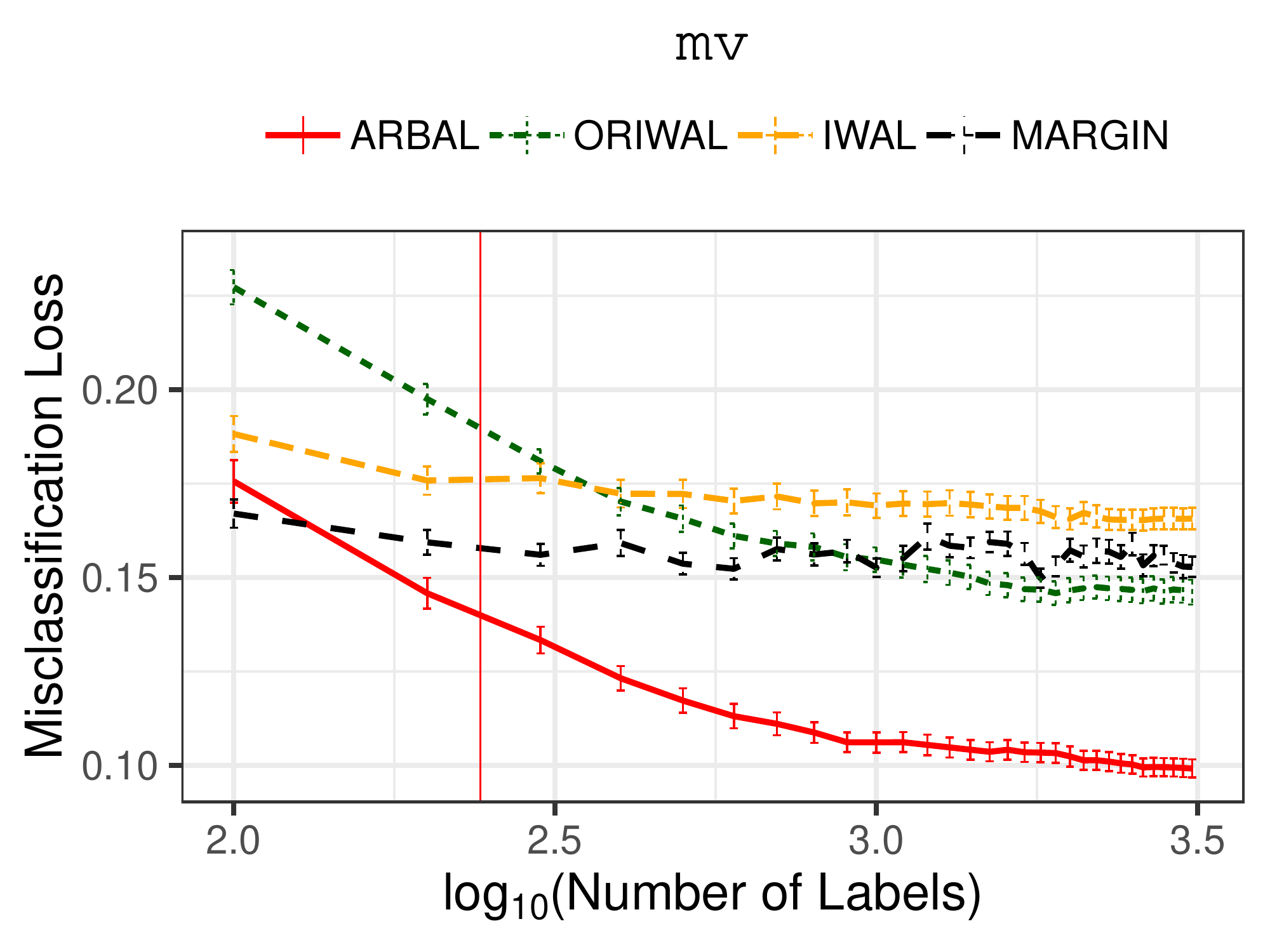}}
  \subfigure{\centering\includegraphics[width=0.24\textwidth,,trim= 5 10 10 5,clip=true]{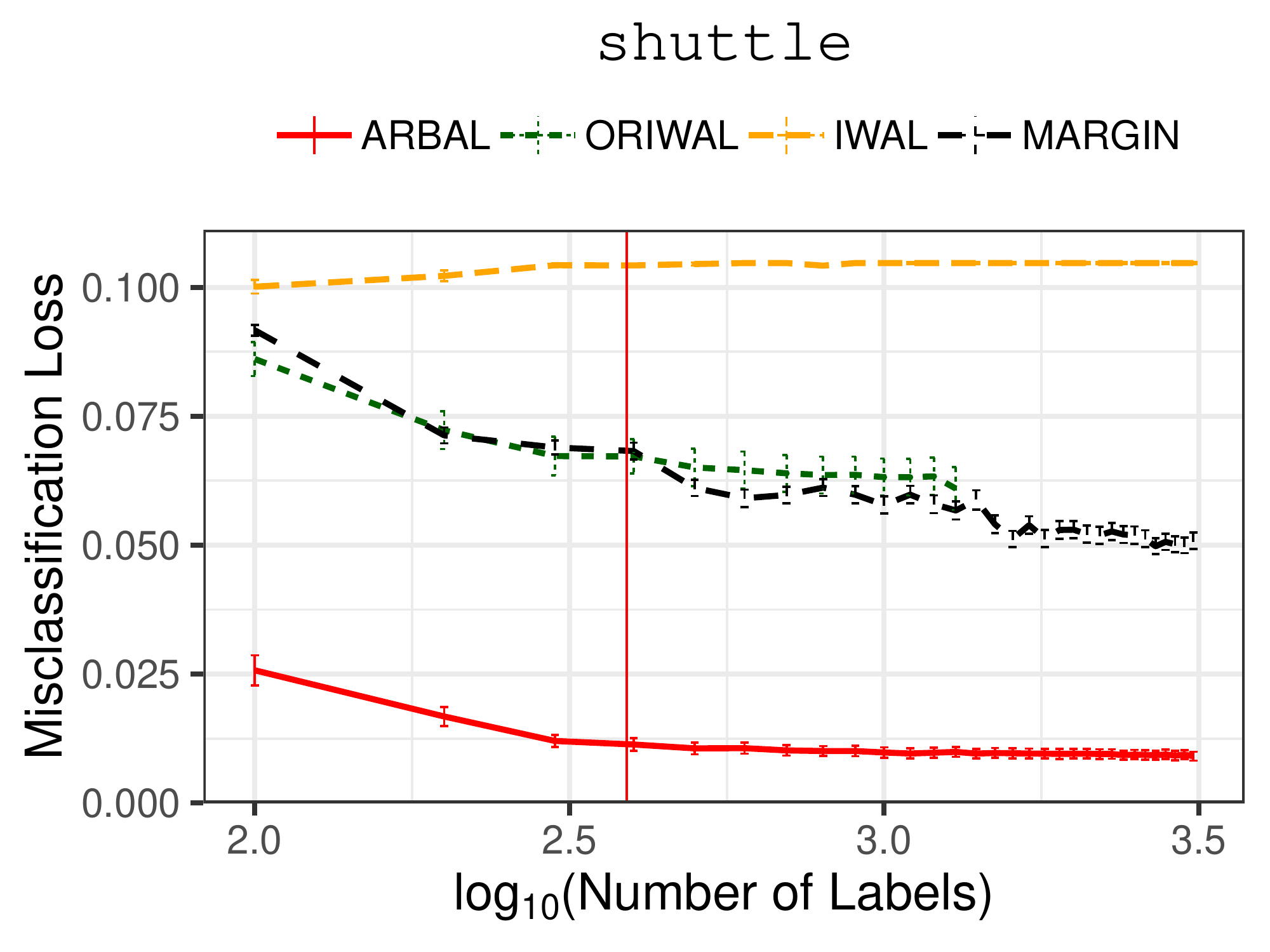}}
  \subfigure{\centering\includegraphics[width=0.24\textwidth,,trim= 5 10 10 5,clip=true]{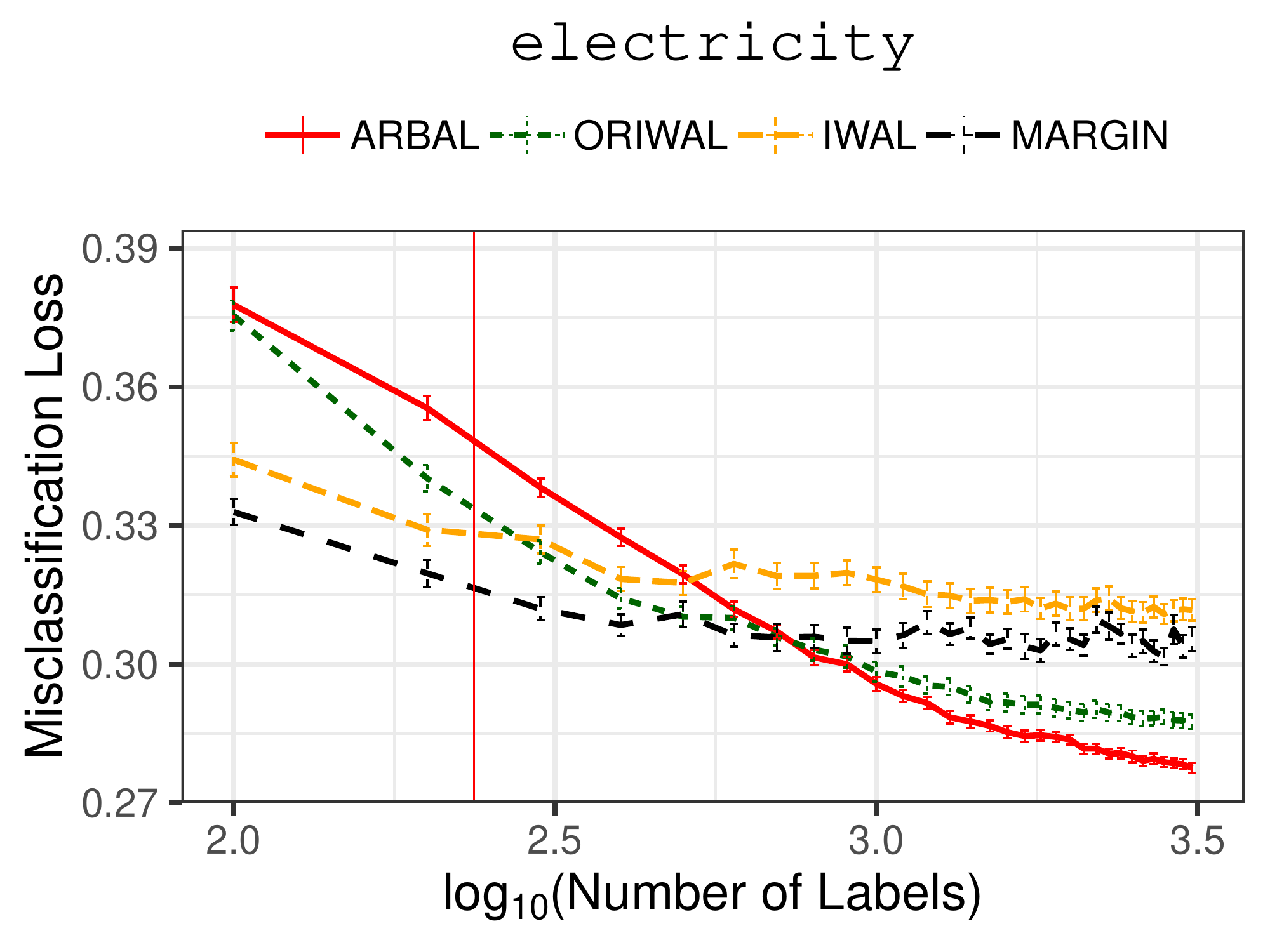}}
  \end{center}
  \vskip -0.2in
  \caption{Misclassification loss of \arbal (with adaptive
    $\gamma_t$), \oriwal, \iwal, and \margin\ on hold out test data
    vs. number of labels requested ($\log_{10}$ scale), with
    $\kappa=20$ and $\tau=800$.  The \arbal\ curves are repetitions
    from Figure~\ref{fig:expmis_gamma_tau800k20_four}. }
\vskip -0.1in
\label{fig:expmis_tau800k20_four}
\end{figure*}
\textbf{\arbal\ with fixed or adaptive $\gamma$. }
We first compare \arbal\ with fixed $\gamma$ to \arbal\ with
an adaptive $\gamma_t$.  Figure~\ref{fig:expmis_gamma_tau800k20_four} plots the
misclassification loss versus the number of labels requested on four datasets. 
The vertical lines indicate the label counts when \arbal\
transitions from the first to the second phase, and the legends give the average 
number of resulting regions $K$ the algorithms produce.\ignore{Note 
that all vertical lines are at label counts often far less than 
$\log_{10}(800)\approx 2.9$, signifying that in the
first phase both versions of \arbal, but especially adaptive $\gamma_t$, selectively
request much fewer labels than passive learning.}
Adaptive $\gamma_t$ tends to split into more regions and to exit the split phase earlier, 
and hence often results in superior prediction performance over fixed $\gamma$. 
Thus, in the rest of this section, we show the performance with adaptive $\gamma_t$.
Results on other datasets (see Appendix~\ref{app:moreexp}) show similar patterns.
During the active learning split phase, even though \arbal\ does not shrink the hypothesis set(s),
both versions are observed to request labels in only 50\% - 90\% of the rounds, 
which is far less than passive learning.

\textbf{\arbal\ vs. \oriwal. }
Since the key idea of \arbal\ is the informed adaptive splitting criterion,
we compare \arbal\ with the \oriwal\ algorithm of \citet{cortes2019rbal}, 
a ``non-adaptive splitting'' algorithm that first randomly generates $\kappa$ regions, 
and then runs region-based active learning on these regions.
The regions of \oriwal\ are obtained from terminal nodes of
random binary trees, that is, binary trees with random splitting
coordinates and thresholds (hence, they are axis-aligned rectangles, as for \arbal). 
Figure~\ref{fig:expmis_tau800k20_four} shows that \arbal\ quickly takes
over (recall that the $x$ axis is on $\log$ scale) and performs substantially better 
than \oriwal, on eight datasets covered by Figure~\ref{fig:expmis_tau800k20_four}.
Results on other datasets show similar 
patterns, even though \oriwal\ sometimes uses more regions
than \arbal, since \arbal\ may not always fully split into $\kappa$ regions. 
These results empirically verify the advantage of an adaptive splitting criterion. 
\ignore{\footnote{ \cite{cortes2019rbal} tested \oriwal\ on a subset of the datasets found in this paper 
showing similar figures as us, but note that the scaling of the axes of the figures between 
the two papers is slightly different.}}

\textbf{\arbal\ vs. non-splitting baselines. }
We also compare \arbal\ with the single-region \iwal\ algorithm,
and the single-region \margin\ algorithm, which is a standard uncertainty sampling algorithm 
that requests the label closest to the decision boundary of the current empirical risk minimizer
(note that \margin\ runs under a pool-based setting and thus sees more information than on-line algorithms).
Figure~\ref{fig:expmis_tau800k20_four} shows that \margin\ is a strong baseline that outperforms \iwal\
on almost all the datasets, sometimes even \oriwal\, (e.g.\ \texttt{house16H}),
but \arbal\ is still more favorable than \margin.  
The difference of errors observed in these plots after consuming much of the sample is 
essentially due to the difference of the split-region and single-region best-in-class errors, 
that is, $R_U$ vs. $R^*$, which further corroborates our theory. 
The results for most other datasets show similar patterns. 
In fact, \arbal\ can also be used with the \margin\ algorithm as a subroutine, which is likely to lead to even better performance but, 
as with the \margin\ algorithm, that extension would not benefit from any general theoretical guarantee and might actually underperform in 
some instances where the \margin\ technique can fail.

Finally, as mentioned in Section~\ref{subsec:split}, \arbal\ 
is agnostic to the shape of subregions, thus any hierarchical partitioning 
method could be used in the splitting phase. For instance, we can split 
a region via an arbitrary separating hyperplane or via hierarchical clustering, 
that is, determine two new centers and assign points to the closest center. 
In Appendix~\ref{app:moreexp}, we compare axis-aligned binary tree 
splitting method with hierarchical clustering splitting, using adaptive $\gamma_t$.
Our results suggest that, for most datasets, splitting via
binary trees is more favorable than via hierarchical clustering.

\vskip -0.1in
\section{Conclusion}
\label{sec:concl}
\vskip -0.05in
We presented a novel algorithm for adaptive region-based active
learning, and proved that it benefits from favorable generalization and
label complexity guarantees. We also studied the extent to which
splitting the input space is likely to lead to improved prediction
performance. We complemented our theoretical findings by reporting the
results of several experiments with our algorithm on standard benchmarks.
Our extensive experiments 
demonstrate substantial performance improvements over existing active learning
algorithms such as \iwal\, and margin-based uncertainty sampling, 
as well as other region-based baselines that do not rely on adaptively
splitting of the input space.
Our techniques have been showcased through \iwal-like algorithms~\citep{beygelzimer2009importance,cortes2019rbal}, but they can be 
straightforwardly combined with other base active learning algorithms, such as the DHM algorithm from \cite{dasgupta2008general}, achieving 
similar generalization and label complexity guarantees. 
\ignore{Recently in \cite{cd19}, \oriwal\ has been combined with the margin algorithm and, although the algorithm does not currently benefit from theoretical guarantees, it is substantially improving upon the margin algorithm. 
Given the close connection between \oriwal\ and \arbal, this suggest that similar benefits could be attained by \arbal. }

Altogether, our theory, algorithms, and empirical results provide a
new promising solution to active learning, with very important
practical benefits. These results also suggest further investigation
of the general idea of adaptively refining and enriching the
hypothesis set for active learning.

\bibliography{arbal}
\bibliographystyle{icml2020}

\clearpage
\appendix
\section{Guarantees for \iwal}
\label{app:iwal}
\ignore{
We first need definitions and concepts from \citet{beygelzimer2009importance}. 
Define the distance $\rho(f, g)$
between two hypotheses $f, g \in \sH$ as follows:
\begin{equation}\label{eq:iwal_rho}
  \rho(f, g) = \E_{(x,y) \sim \sD} \left| \ell(f(x), y) - \ell(g(x), y) \right|.
\end{equation}
Note that, this definition of 
$\rho(f,g)$ slightly differs from the original definition 
in \citet{beygelzimer2009importance}.
We give more details near the end of this section.
Given $r > 0$, let $B(f, r)$ denote the ball of radius $r$ centered in
$f \in \sH$: $B(f, r) = \set{g \in \sH \colon \rho(f, g) \leq r}$.
The generalized disagreement coefficient $\theta(\sD, \sH)$ of a class of 
functions $\sH$ with respect to distribution $\sD$ is defined as follows:
\begin{align*}
  & \theta(\sD, \sH) = \inf_{\theta > 0} \set[\bigg]{\forall r \geq 0, \\
    &  \E_{x \sim \sD_\cX}\bigg[\sup_{h \in B(h^{*}, r)}\sup_{y 
    \in \cY}
\big| \ell(h(x), y) - \ell(h^{*}(x), y) \big| \bigg] \leq \theta r }\,.
\end{align*}
}
The disagreement coefficient
$\theta$ is a complexity measure widely used in disagreement-based
active learning problems. In particular, \citet{hanneke2007bound}
proved upper bounds for the label complexity for the $\textsc{a}^2$
algorithm in terms of %the disagreement coefficient
$\theta$. \citet{dasgupta2008general} also gave an upper bound for their
DHM algorithm using $\theta$. The reader is referred to \citet{Hanneke2014} for a more
extensive analysis of the disagreement coefficient as related to active learning.

In \citet{beygelzimer2009importance}, the distance $\rho$ is defined as
\[  
\rho(f, g) = \E_{x \sim \sD_\cX} \sup_{y\in\cY}
\left| \ell(f(x), y) - \ell(g(x), y) \right|~,
\]
while the distance $\rho$ in Section~\ref{sec:theory} is defined in a slightly 
different manner as
\begin{equation}\label{eq:iwal_rho}
  \rho(f, g) = \E_{(x,y) \sim \sD} \left| \ell(f(x), y) - \ell(g(x), y) \right|.
\end{equation}
\cite{cortes2019disgraph} showed that the new definition of 
$\rho$ in Eq.~\eqref{eq:iwal_rho}
removes a constant $K_{\ell}$ from the label complexity bound of 
\iwal, where $K_{\ell}$ depends on the loss function
and is always greater than 1. Thus this new definition of $\rho$
improves the label complexity bound of \iwal.

\begin{theorem}[\cite{beygelzimer2009importance}]
\label{thm:iwal}
Let $\h h_T$ be the hypothesis output by \iwal\ after $T$ rounds.
For all $\delta>0$, with probability at least $1 - \delta$, 
for any $t\in[T]$,
\begin{align}
& R(\h h_T) 
\leq 
R^* + 2 \sqrt{\frac{8 \log \big[ \frac{2T (T + 1) |\sH|^2 }{\delta} \big]}{T}},\label{eq:iwalgen}\\
& \E_{x_t\sim \sD_\cX} \!\big[p_t | \cF_{t-1}\big] \!
\leq 
\!4\theta \bigg[ \!R^* \!+ \!\sqrt{\frac{8\log \big[ \frac{2(t - 1) t |\sH|^2}{\delta}\big]}{t-1}} \bigg]~,\label{eq:iwallabel}
\end{align}
where $\cF_t$ denotes the $\sigma$-algebra generated by
$(x_1, y_1, Q_1), \ldots, (x_t, y_t ,Q_t)$. 
\end{theorem}
Thus, 
the generalization error of the returned hypothesis $\h h_T$ is 
close to that of the best-in-class, while
the expected number of labels requested after $T$ rounds is in
$O(R^*T)$.

\section{Region-Based Active Learning}
\label{app:disagreement}

The following results are adapted from \cite{cortes2019rbal}.
Lemma~\ref{lemma:dis-coef} relates the region-specific disagreement
coefficients $\theta_k = \theta(\sD_k, \cH)$ to the overall
disagreement coefficient $\theta(\sD,\prodH)$, where
$\prodH= \big\{ \sum_{k = 1}^{\kappa} \one_{x \in \cX_k} h_k(x)\,
\colon\, h_k \in \sH \big\}$. Theorem~\ref{thm:compare_iwal} compares
the learning guarantees of running with $\prodH$ and running \iwal\
within each region separately.

\begin{lemma}\label{lemma:dis-coef}
  The generalized disagreement coefficient $\theta(\sD,\prodH)$ satisfies 
  $\theta(\sD,\prodH)\leq\sum_{k = 1}^{\kappa}\theta(\condsD_k,\sH)$.
\end{lemma}
\begin{proof}
  Denote $h^*=\argmin_{h\in\prodH}R(h)$, and $h^*_k = \argmin_{h\in\sH} R_k(h)$.
  Recall that $\condsD_k = \sD|\cX_k$ denotes the conditional distribution 
  of $x$ on $\cX_k$, and that
  $h^*$ is defined as $h^* = \sum_{k = 1}^{\kappa}\one_{x\in X_{k}}h_{k}^{*}$.
  Extending the definitions in Section~\ref{sec:theory}, we define
  $$\rho_k(f,g)=\E_{(x,y)\sim\condsD_k} |\ell(f(x),y)-\ell(g(x),y)|.$$
  Given the hypothesis set $\sH$ and any real $r>0$, define 
  $$B_k(f,r)=\big\{ g\in \sH \colon \rho_k(f,g)\le r\big\}.$$
  For a set of non-negative values $\lambda=\{ \lambda_{1},\dots,\lambda_{\kappa}\}$ , let
  $$G_{\lambda}(h^{*},r)=\Big\{ \sum_{k = 1}^{\kappa}
  1_{x\in X_{k}}g_{k} \colon g_{k}\in B_{k}(h_{k}^{*},\lambda_{k}r)\Big\} .$$

  We first show that, for any $\lambda$ satisfying 
  $\sum_{k = 1}^{\kappa}\tp_k\lambda_{k}\leq1$, $G_{\lambda}(h^{*},r)\subseteq B(h^{*},r)$. 
  Let $g=\sum_{k = 1}^{\kappa}1_{x\in X_{k}}g_{k}$, where $g_k \in B_k(h_k^*, \lambda_k r)$. 
  Then,
  \begin{align*}
     \rho\left(h^{*},g\right) 
     &=\E_{(x,y)\sim \sD}|\ell(h^{*}(x),y)-\ell(g(x),y)| \\
     & =\sum_{k = 1}^{\kappa}\tp_k\E_{(x,y)\sim \condsD_k}
    |\ell(h_{k}^{*}(x),y)-\ell(g_{k}(x),y)|  \\
    & \leq\sum_{k = 1}^{\kappa}\tp_k\lambda_{k}r \leq r. 
  \end{align*}
  Thus, $\set[\big]{ \cup_{\lambda \colon \sum_{k = 1}^{\kappa}\tp_k\lambda_{k}\leq1}
  G_{\lambda}(h^{*},r)} \subseteq B(h^{*},r)$.
  On the other hand, if there exits a hypothesis $h$ such that 
  $$
  h\in B(h^{*},r)\Big\backslash \set[\big]{ \cup_{\lambda \colon \sum_{k = 1}^{\kappa}
  \tp_k\lambda_{k}\leq1}G_{\lambda}(h^{*},r)}~,
  $$
  let this $h$ be of the form $h=\sum_{k = 1}^{\kappa}1_{x\in X_{k}}h_{k}$. Then,
  \begin{align*}
  \rho(h^{*},h) =\sum_{k = 1}^{\kappa}\tp_k\rho_k(h_k^*,h_k)\leq r 
   \Rightarrow \sum_{k = 1}^{\kappa} \tp_k\frac{\rho_k(h_{k}^{*},h_{k})}{r}\leq1.
  \end{align*}
  Obviously, $h_k \in B_k(h_k^*, \rho_k(h_k^*,h_k))$. Thus, let 
  $\lambda=\{ \frac{\rho_1(h_{1}^{*},h_{1})}{r},\dots,
  \frac{\rho_p(h_{\kappa}^{*},h_{\kappa})}{r}\} $, then
  $\sum_{k = 1}^{\kappa}\tp_k\lambda_{k}\leq1$, and $h\in G_{\lambda}(h^{*},r)$ by definition. 
  We have a contradiction. Therefore, 
  \[
    \set[\Big]{ \cup_{\lambda \colon \sum_{k = 1}^{\kappa}\tp_k \lambda_{k}\leq1}
    G_{\lambda}(h^{*},r)} = B(h^{*},r)~.
  \]
  Given the equivalence above, for any $k\in[\kappa]$,
  \begin{align}
    \sH \cap B(h^{*},r) 
    & = \sH \cap \{ \cup_{\lambda \colon \sum_{k = 1}^{\kappa}\tp_{k}\lambda_{k}\leq1}
    G_{\lambda}(h^{*},r)\} \nonumber \\
    & = \sH \cap \{ \cup_{\lambda_k \leq 1/\tp_k} B_k(h_k^*, \lambda_k r)\} 
    \label{eq:H_cap_B2}\\
    & = B_{k}(h_{k}^{*},r/\tp_k)\label{eq:H_cap_B3}~.
  \end{align}
  Equation~\eqref{eq:H_cap_B2} follows from the definition of $G_\lambda(h^*,r)$.
  Putting everything together, we have for any $r\ge0$, 
  \begin{align}
    & \E_{x\sim D}\sup_{h\in B(h^{*},r)}\sup_{y}
    |\ell(h(x),y)-\ell(h^{*}(x),y)| \nonumber \\
    &=  \sum_{k = 1}^{\kappa}\tp_k\E_{x\sim \condsD_k}\sup_{h\in B(h^{*},r)}\sup_{y}
    |\ell(h(x),y)-\ell(h^{*}(x),y)| \nonumber \\
    &=  \sum_{k = 1}^{\kappa}\tp_k\E_{x\sim \condsD_k}
    \sup_{y, h_k\in B_k(h_k^{*},\frac{r}{\tp_k} )} 
    |\ell(h_k(x),y)-\ell(h_k^{*}(x),y)| \label{eq:discoef1}\\
    &\leq  \sum_{k = 1}^{\kappa}\tp_k\theta(\condsD_k, \sH)r/\tp_k \label{eq:discoef2}\\
    &= \Big(\sum_{k = 1}^{\kappa}\theta(\sD_k ,\sH)\Big)r\nonumber .
  \end{align}
  Equation~\eqref{eq:discoef1} holds due to the equivalence in~\eqref{eq:H_cap_B3}, 
  and inequality~\eqref{eq:discoef2} follows from the definition of 
  $\theta(\condsD_k,\sH)$.

  Finally, recall the definition of $\theta(\sD,\prodH)$:
  \begin{align*}
    \theta&(\sD,\prodH) = \inf \Bigl\{\theta \colon \forall r\ge 0,\\
          &     \E_{x\sim\sD}\sup_{h\in B(h^{*},r)}\sup_{y}
    |\ell(h(x),y)-\ell(h^{*}(x),y)|\le\theta r\Bigl\}.
  \end{align*}
  Therefore $\theta(\sD,\prodH)\leq\sum_{k = 1}^{\kappa}\theta(\condsD_k,\sH_k)$, 
  which concludes the proof.
\end{proof}

Combining Lemma~\ref{lemma:dis-coef} with the learning guarantee of \iwal\ (Theorem \ref{thm:iwal}), 
we obtain the following result.

\begin{theorem}\label{thm:compare_iwal}
Assume $\theta_k = \theta(\condsD_k,\sH)$ is the same across all regions $\cX_k, k \in [\kappa]$.
Consider running with \iwal\ with $\prodH$ (Method 1) and running \iwal\ with $\cH$ on
each region separately (Method 2). 
Then, the hypothesis returned 
by both methods admit comparable generalization error guarantees, 
but on average running with $\prodH$ would request up to $\kappa$ times more labels.
\end{theorem}

\begin{proof}
  Denote $h^*=\argmin_{h\in\prodH}R(h)$.
  Let $N = |\sH|$, and $\theta_0 = \theta(\condsD_k, \sH)$, for all $k\in[\kappa]$,
  so that $|\prodH| = N^\kappa$ and, from Lemma~\ref{lemma:dis-coef}, $\theta(\sD,\prodH)\leq \kappa\theta_0$. 
  According to the learning guarantee of \iwal, with probability at least 
  $1 - \delta$, Method 1 (running with $\prodH$) satisfies
  \begin{align}
    & R(h_{T}^{(1)})
    \leq R(h^{*}) + O\Big(\sqrt{\frac{\log (T N^{2\kappa}/\delta)}{T}}\Big), 
    \label{eq:naiwal_reg}\\
    & \tau_T^{(1)}  \leq 4 \kappa \theta_0
    \Big[R(h^{*})T + O\Big(\sqrt{T\log(TN^{2\kappa}/\delta)}\, \Big) \Big] .\label{eq:naiwal_label}
  \end{align}
  In addition, with probability at least $1 - \delta$, Method 2 
  (running \iwal\ within each region separately) satisfies
  \begin{align}
    & R(h_{T}^{(2)}) 
     \leq R(h^*) + 
    \sum_{k = 1}^\kappa \tp_k \, O\Big(\sqrt{\frac{\log (T|N|^2
      \kappa/\delta)}{T_k}}\Big), \mspace{-5mu}
    \label{eq:riwal_reg}\\
    & \tau_T ^{(2)} 
     \leq \sum_{k = 1}^\kappa 4 \theta_0 \Big[R_k(h^{*})T \tp_k 
    + O(\sqrt{2T \tp_k \log (2TN^2 \kappa/\delta})\Big] \nonumber \\
    &\quad \ \  = 4\theta_0 \Big[R(h^*)T + \sum_{k = 1}^\kappa O(\sqrt{2T \tp_k 
    \log (2TN^2 \kappa/\delta})\Big].\label{eq:riwal_label}
  \end{align}
  Replacing $T_k$ with $ T \tp_k + O(\sqrt{T})$ in the RHS of~\eqref{eq:riwal_reg}, 
  and using the fact that $\sum_{k = 1}^\kappa \sqrt{\tp_k} \leq \sqrt{\kappa}$, we obtain 
  \begin{align}\label{eq:riwal_reg2}
    R(h_{T}^{(2)}) &\leq R(h^*) + 
    O\Big(\sqrt{\frac{\kappa \log (T|N|^2 \kappa/\delta)}{T}}\Big).
  \end{align}
  Comparing the upper bound on the generalization error of 
  Method 2~(Eq. \eqref{eq:riwal_reg2}) to that of Method 1~(Eq. \eqref{eq:naiwal_reg}), 
  we conclude that the two algorithms admit comparable learning guarantees.

  On the other hand, comparing the proportion of labels requested per round, we have
  \begin{align*}
    \tau_T^{(1)} /T &\leq 4\kappa \theta_0 R(h^*) + 
    O\bigg(\frac{1}{\sqrt{T}}\bigg),\\
    \tau_T^{(2)} /T &\leq 4\theta_0 R(h^*) + 
    O\bigg(\frac{1}{\sqrt{T}}\bigg).
  \end{align*}
  Thus, Method 1 may request up to $\kappa$ times more labels than Method 2.
\end{proof}

\section{Proofs}
\label{app:proof}

For simplicity of presentation, all results are stated and proven under the assumption that the loss function $\ell$ is $\mu$-Lipschitz with $\mu \leq 1$. 
This is the case, e.g., for hinge loss and logistic loss.

\begin{replemma}{lemma:split_concentrate}
  With probability at least $1 - \delta/4$, for all binary trees with (at most) $\kappa$ leaf nodes,
  the improvement in the minimal empirical error by splitting 
  concentrates around the improvement in the best-in-class error:
   \begin{align*}
       & \Big|
     \big[R_k(h_k^*) - R_k(h_{lr}^*)\big] - 
     \big[L_{k, t}(\h h_{k, t}) - L_{k, t}(\h h_{lr,t})\big]\Big| \\
     & \leq \sqrt{\frac{2\slack}{T_{k, t}}}.
     \ignore{     & \leq \sqrt{\frac{2\kappa D \log(2|\sH|^3 T_{k, t} (T_{k, t}+1) 
     \kappa TD/\delta)}{T_{k, t}}}.}
   \end{align*}
\end{replemma}
\begin{proof}

    We first assume that the splitting threshold $c$ only takes values 
    in pre-specified sets. To be more concrete, when splitting along coordinate 
    $d$, the threshold only takes one of the $C$ pre-specified values:
    $c\in \Theta^d = \{\theta^d_1, \cdots, \theta^d_C\}$, 
    where $\theta^d_1,\cdots,\theta^d_C\in\Rset$ discretize the $d$-th coordinate.
    Given this assumption, we can upper bound the number of possible binary trees with 
    at most $\kappa$ regions.  Note that for there to be $\kappa$ regions, there must be $\kappa-1$ splits.
    Also note that, at each split, there are at most $C\times D$ possible splitting tuples of $(d,c)$,
    where $D$ is the number of features and $C$ is the number of possible thresholds.
    At the $k$-th split, $k\leq \kappa-1$, one first chooses which leaf node 
    to split on (there are $k$ of them), and then picks a splitting tuple,
    thus there are $kCD$ possible splitting outcomes.
    By the multiplication rule in probability, there are a total of 
    $$
    (CD) \times (2CD) \times (3CD) \cdots \times ((\kappa-1)CD) \leq  (\kappa CD)^{\kappa}
    $$
    possible binary trees with $\kappa$ regions.

    We prove this Lemma as follows. We first fix a binary tree and prove 
    concentration inequalities that hold for every internal node of that tree. 
    Next, we take a union bound over 
    the $(\kappa CD)^{\kappa}$ trees to extend these inequalities to hold every 
    node of every tree with the given splitting thresholds.
    Finally, we extend that to trees with arbitrary thresholds using 
    a standard covering number argument.

    Fix a binary tree as well as an intermediate region $k$ during the split phase.
    Furthermore, fix a $T_k>0$ and condition on the event $T_{k, t}=T_k$.
    Then we can drop the time subscript $t$ from notation, 
    since the tail probability will be determined by $T_k$ only. 
    To avoid clutter in the notation, we re-index the sample 
    points $x_1,\cdots, x_T$
    in such a way that the first $T_k$ of them fall in region $\cX_k$.

    Define the composite hypothesis set
    $\sH^2 = \big\{1_{x\in\cX_l} h_1 + 1_{x\in\cX_r}h_2\colon
    h_1,h_2\in\sH\big\}$; then $|\sH^2| = |\sH|^2$. Moreover, 
    \[
        h_{lr}^* = \argmin_{h\in\sH^2} R_k(h), \quad
        \h h_{lr} = \argmin_{h\in\sH^2} L_k(h)\,.
    \]

    Fix a pair of hypotheses $f\in \sH$, $g \in \sH^2$.
    For brevity, define $\ell(h,h'; x,y) = \ell(h(x),y) - \ell(h'(x),y)$, and
    then the random variable 
    \[
        %    Z_t = \frac{Q_t}{p_t}\big[ \ell(f(x_t),y_t) - \ell(g(x_t),y_t)\big].
        Z_t = \frac{Q_t}{p_t} \ell(f,g; x_t, y_t). \,
    \]
    Then, $\{Z_t, t\in[T_k]\}$ are i.i.d.\ random variables, 
    since when the hypothesis set $\sH$ is fixed we have
    $p_t = p(x_t)$, being $p(x)$ the average disagreement of $\sH$ on $x$.
    Thus $p_t$ only depends on $x_t$ and $\sH$, and is independent of the past 
    (unlike the standard \iwal).

    By definition, $|Z_t|\leq 1$ since at point $x_t$, 
    \[
        \max_{f\in\sH, g\in\sH^2} \ell(f,g; x_t, y_t) 
        =\max_{f\in\sH, g\in\sH} \ell(f,g; x_t, y_t) \leq p_t,
    \]
    where recall the label request probability 
    \[
        %p_t = \max_{h,h' \in \sH,\,y\in\cY} \ell(h(x_t),y_t) - \ell(h'(x_t),y_t).
        p_t = \max_{h,h' \in \sH,\,y_t\in\cY} \ell(h,h'; x_t, y_t).
    \]
    Furthermore,
    \[
        \E_{\substack{Q_t \sim p_t \\ (x_t,y_t)\sim\sD|\cX_k}} [Z_t ]= 
        R_k(f) - R_k(g).
    \]
    Applying Hoeffding's inequality to $Z_t$ yields
    \begin{align*}
        & \Pr\bigg(\bigg| \sum_{t=1}^{T_k} Z_t - \E[Z_t] \bigg|
        \geq T_k \Delta_{T_k}\bigg) 
        \leq 2e^{-\frac{T_k \Delta_{T_k}^2}{2}}  \\
        & = \frac{\delta}{4T_k (T_k+1)|\sH|^3}\,,
    \end{align*}
    where
    $\Delta_{T_k}=\sqrt{\frac{2\log(8|\sH|^3 T_k
    (T_k+1)/\delta)}{T_k}}$. A union bound over all possible
    values of $T_k$ and all pairs of $(f,g) \in \sH \times \sH^2$
    allows us to conclude that, with probability at least $1 - \delta/4$, for all
    $T_k$ %(thus for all $t$) 
    and all $(f,g)$,
    \begin{equation}\label{eq:azuma2}
        | R_k(f) - R_k(g) - L_k(f) + L_k(g)|\leq \Delta_{T_k}.
    \end{equation}
    Thus, 
    \begin{align*}
        R_k(h_k^*) - R_k(h_{lr}^*) 
        &\geq R_k(h_k^*) - R_k(\h h_{lr})  \\
        & \geq L_k(h_k^*) - L_k(\h h_{lr}) - \Delta_{T_k} \\
        &\geq L_k(\h h_k) - L_k(\h h_{lr}) - \Delta_{T_k}\,,
    \end{align*}
    where the first inequality follows from the definition of $h_{lr}^*$,
    the second inequality follows from \eqref{eq:azuma2}
    (since $h_k^*\in\sH$ and $\h h_{lr} \in \sH^2$), and
    the last inequality follows from the definition of $\h h_{k}$.

    Similarly, 
    \begin{align*}
        R_k(h_k^*) - R_k(h_{lr}^*) 
        &\leq R_k(\h h_k) - R_k( h_{lr}^*) \\
        &\leq L_k(\h h_k) - L_k( h_{lr}^*) + \Delta_{T_k} \\
        &\leq L_k(\h h_k) - L_k(\h h_{lr}) + \Delta_{T_k}\,.
    \end{align*}

    To take a union bound over at most $\kappa$ regions, as well as over the $(\kappa CD)^{\kappa}$ 
    possible binary trees, we replace $\delta$ with $\frac{\delta}{\kappa (\kappa CD)^{\kappa}}$
    in the expression of $\Delta_{T_k}$. Thus, we have the concentration results of Eq.~\eqref{eq:azuma2} 
    hold uniformly over all $\kappa$ leaf nodes and over all binary trees constructed from the pre-specified 
    thresholds, with
    $\Delta_{T_k}=\sqrt{\frac{2\kappa\log(8|\sH|^3 T_{k, t} (T_{k, t}+1) \kappa CD/\delta)}{T_{k, t}}}$. 

    \ignore{Finally, we set $C = 1/\e$, and thus the covering number is $O(1/{\e^D})$. 
    Further replacing $\delta$ with $\delta \e^D$, and setting $\e = O(1/\sqrt{T})$, we have 
    $\Delta_{T_k}=\sqrt{\frac{2\kappa D \log(2|\sH|^3 T_{k, t} (T_{k, t}+1) \kappa TD/\delta)}{T_{k, t}}} $.
    By the covering number argument, extending from the pre-specified splitting thresholds to arbitrary binary trees
only adds an $\e = O(1/\sqrt{T})$ term to the right-hand side of the concentration result, which is thus omitted.}

    Finally, by a standard argument, with the $\mu$-Lipschitzness of the loss, with $\mu \leq 1$,
    the family of losses of trees with any threshold can be covered by the losses 
    of those with thresholds in $\Theta$, where $\Theta$ is defined by values
    separated by $\e$ for each dimension, which has cardinality $(1/\e)^D$.
    Further replacing $\delta$ with $\delta \e^D$, 
    and setting $\e = 1/T$, we have 
    $\Delta_{T_k}=\sqrt{\frac{2\kappa D \log(8|\sH|^3 T_{k, t} (T_{k, t}+1) \kappa TD/\delta)}{T_{k, t}}} $.
    Upper bounding $T_{k, t} (T_{k, t}+1)$ with $T^2$ yields 
    Lemma~\ref{lemma:split_concentrate}.
\end{proof}

\begin{repcorollary}{thm:split}
  With probability at least $1 - \delta/4$, for all splits made by \arbal,
  the improvement in the best-in-class error is at least $\gamma_t$,
  where $\gamma_t$ is the threshold at the time of split.
\end{repcorollary}

\begin{proof}
  Let \arbal\ split at time $t$ region $\cX_k$ with threshold $\gamma_t$. 
  From Lemma~\ref{lemma:split_concentrate}, with probability at least 
  $1 - \delta/4$, for any split that \arbal\ makes, 
  \begin{align*}
      & \tp_k  \big(R_k(h_k^*) - R_k(h_{lr}^*) \big) \\
    & \geq  \tp_k
    \Big(L_{k, t}(\h h_{k, t}) - L_k(\h h_{lr,t}) 
    - \eta_{k,t}\Big)
%    &    - \sqrt{\frac{2\kappa D \log(8|\sH|^3 T_{k, t} (T_{k, t}+1)\kappa TD/\delta)}{T_{k, t}}}\Big) 
     \geq \gamma_t\,,
  \end{align*}
  where the last inequality follows from the definition of $\gamma_t$
  and the splitting criterion.
\end{proof}

We now proceed to proving Theorem \ref{thm:splitiwal2}, but first
we show a version of Theorem~\ref{thm:splitiwal2} with random
quantities in it.
\begin{theorem}
\label{thm:splitiwal}
 Assume \arbal\ runs with a fixed $\gamma$ and has split the input 
 space into $K$ regions. Then, with probability at least $1 - \delta/2$,
  \begin{align}
    R(\h h_T) 
   &  \leq \sum_{k = 1}^{K} \tp_k \bigg[ R_k^* 
+ 4\sqrt{\mfrac{2\slack}{T_k}}\bigg] \label{err1}\\
  & \leq R^* - \gamma (K-1) + \sum_{k = 1}^{K} 4 \tp_k \sqrt{\mfrac{2\slack}{T_k}},\label{err2}
\end{align}
where $T_k$ is the total number of unlabeled samples in region $k$ up
to round $T$. Moreover, with probability at least $1 - \delta/2$,
\begin{align}
  & \sum_{t=1}^T \E_{x\sim\sD_{\cX}} \big[p_t |\cF_{t-1}\big] 
  \leq \min \{2\theta r_0, 1\}\tau \, \nonumber \\
  & +  \sum_{k = 1}^{K} 4\theta_k \bigg[R_k^* T_k' + 
  4\sqrt{2T'_k \slack}\bigg],\label{eq:splitlabel}
  \end{align}
  where $T_k'$ is the total number of unlabeled samples in region
  $k$ from round $\tau+1$ to $T$.
\end{theorem}

\begin{proof}
  By Eq~\eqref{eq:iwalgen} in Theorem~\ref{thm:iwal}, for 
  a fixed binary tree and for a fixed region $k$ resulting 
  from the binary tree, with probability at least 
  $1 - \delta/4$, for all $T_k>0$,
  \[
    R(\h h_{k, T}) - R(h_k^*) \leq 
    2\sqrt{\frac{8\log(8T_k(T_k+1)|\sH|^2 /\delta)}{T_k}}\,.
  \]
  Using the same technique as in the proof of 
  Lemma~\ref{lemma:split_concentrate}, 
  we take a union bound over the (at most) $\kappa$ regions
  as well as over the $(\kappa CD)^{\kappa}$ 
  possible binary trees and use the Lipschitzness argument,
  we obtain the first inequality~\eqref{err1} 
  with probability at least $1 - \delta/4$.
  To simplify the notation, we have upper bounded the term in the log,
  $T_k(T_k+1)|\sH|^2 \leq T^2|\sH|^3$,
  to match that in $\slack$.

  The second inequality~\eqref{err2} follows from
  Corollary~\ref{thm:split}, that
  each split gives at least $\gamma$ improvement in the error
  of best-in-class predictors. It follows that with probability at
  least $1 - \delta/4$,
  \[
    \sum_{k = 1}^{K} \tp_k R_k^* \leq R^* - \gamma(K-1).
  \]
  A union bound over \eqref{err1} and~\eqref{err2} gives
  the first statement of Theorem~\ref{thm:splitiwal}.

  The statement about label complexity~\eqref{eq:splitlabel} follows
  from the analysis of \iwal's label complexity. Recall that in
  \iwal, with probability at least $1 - \delta$,
  \begin{align*}
    \E_{x\sim \sD_\cX} \big[p_t | \cF_{t-1}\big] 
    \leq 4\theta \Biggl(R^* + \sqrt{\frac{8\log(2T^2|\sH|^2 /\delta)}{t-1}}\Biggl),
  \end{align*}
  where we upper bound $(t-1)t$ by $T^2$ since $t\in[T]$. Within
  the split phase, the space of hypotheses remains $\sH$. It is
  easy to show that, among the first $\tau$ rounds, by the
  definition of the disagreement coefficient $\theta$ and the triangle
  inequality,
  \begin{equation}\label{eq:splitlabel1}
    \sum_{t=1}^\tau \E_{x\sim\sD_\cX} \big[p_t | \cF_{t-1}\big] 
    \leq \min\{ 2 \theta r_0, 1\}\, \tau .
  \end{equation}
  After $\tau$ rounds, the label complexity follows directly 
  from Eq. \eqref{eq:iwallabel} in Theorem~\ref{thm:iwal}: 
  by the same argument for the generalization bound \eqref{err1},
  with probability at least $1 - \delta/4$, we have
   \begin{align}\label{eq:splitlabel2}
     & \sum_{t=\tau+1}^T \E_{x\sim\sD_\cX} \big[p_t | \cF_{t-1}\big] \nonumber \\
     & \leq \sum_{k = 1}^{K} 4\theta_k \bigg[R_k^* T_k' + 
       \sum_{s=T_k-T'_k+1}^{T_k} 
       \sqrt{ \mfrac{8 \slack}{s-1}}
     \bigg] \nonumber \\
     & \leq \sum_{k = 1}^{K} 4\theta_k \bigg[R_k^* T_k' + 
     2\sqrt{8T'_k\,\slack}
   \bigg],
   \end{align}
   where the last inequality uses
   $\sum_{i=a}^b 1/\sqrt{i} \leq 2 (\sqrt{b} - \sqrt{a-1}) \leq
   2\sqrt{b-a+1}$. Combining~\eqref{eq:splitlabel1}
   with~\eqref{eq:splitlabel2} concludes the proof.
\end{proof}

The learning guarantees in Theorem~\ref{thm:splitiwal} depend on the
random quantities $T_k$ and $T_k'$. We can further apply Chernoff's
inequality, and relate these random quantities to their expectations.

\begin{theorem}[Chernoff]
\label{thm:chernoff}
Let $X_1,\cdots, X_m$ be independent random variables drawn according
to some distribution $\sD$ with mean $p$ and support included in
$[0,1]$. Then, for any $\gamma \in [0,\frac{1}{p}-1]$, the following
holds for $\h p=\frac{1}{m}\sum_{i=1}^m X_i$:
  \begin{align*}
    \Pr[\h p \geq (1+\gamma)p] & \leq e^{-\frac{mp\gamma^2}{3}},\\
    \Pr[\h p \leq (1-\gamma)p] & \leq e^{-\frac{mp\gamma^2}{2}}.
  \end{align*}
\end{theorem}

Now we prove Theorem~\ref{thm:splitiwal2}.
\begin{reptheorem}{thm:splitiwal2}
Assume that a run of \arbal\ over $T$ rounds has split the input space into
$K$ regions.  Then, for any $\delta > 0$, with probability at least
$1 - \delta$, the following inequality holds:
  \begin{align*}
    R(\h h_T)
    \leq R_U +  \sqrt{\frac{32 K \slack}{T} }
    + \frac{16K\slack}{T},
  \end{align*}
where $R_U = R^* - \gamma (K-1)$ 
is an upper bound on the best-in-class error obtained by \arbal.
Moreover, with probability at least $1 - \delta$, the expected number of 
labels requested, $\tau_T= \sum_{t=1}^T \E_{x_t\sim\sD_{\cX}} \big[p_t |\cF_{t-1}\big]$,
satisfies
  \begin{align*}
    \tau_T
    & \leq \min \set{2\theta r_0, 1} \tau 
    + 4\theta_{\max}(T - \tau) 
    \Big[ R_U 
    + 8 \sqrt{\mfrac{K \slack}{T - \tau}} \Big] \\
& +  \sqrt{32}K\slack .
  \end{align*}
\end{reptheorem}
\begin{proof}
 Given a total of $T$ samples and a fixed partition, we have $\E[T_k] = T \tp_k$. By
 Theorem~\ref{thm:chernoff}, with probability $1 - \delta/4$, for all
 $k\in[\kappa]$,
  \[
      \frac{T_k}{T} \geq \tp_k \Big(1- \sqrt{\frac{2\log(\frac{4\kappa}{\delta})}{T\tp_k}}\Big).
  \]
  By the same covering number argument in Lemma~\ref{lemma:split_concentrate},
  with probability at least $1-\delta/4$, for all partitions and all $k\in[\kappa]$,
  \[
      \frac{T_k}{T} \geq \tp_k \Big(1- \sqrt{\frac{2\kappa D\log(\frac{4\kappa TD}{\delta})}{T\tp_k}}\Big)
      \geq \tp_k \Big(1- \sqrt{\frac{2\slack}{T\tp_k}}\Big).
  \]
  It follows that when
  $T\geq \frac{4\slack}{\min_{k\in[K]} \tp_k}$ (or
  equivalently there are at least $4\slack$ points in each
  region, which can be easily satisfied), we have
  \[
    \frac{\tp_k}{\sqrt{T_k}} \leq \sqrt{\frac{\tp_k}{T}} 
    + \frac{2\sqrt{2\slack}}{T}.
  \]
  Plugging into Theorem~\ref{thm:splitiwal}, a union bound implies
  that with probability at least $1 - \delta$,
   \begin{align*}
    R(\h h_T) 
   & \leq R^* - \gamma (K-1)
   + \sum_{k = 1}^{K} 4 \tp_k \sqrt{\frac{2\slack}{T_k}}\\
   & \leq R^* - \gamma (K-1)
   + \sum_{k = 1}^{K} 4 \sqrt{\frac{2 \tp_k \slack}{T} } 
   + \frac{16K\slack}{T}.
   \end{align*}
   Furthermore, by Theorem~\ref{thm:chernoff}, with probability at least $1 - \delta/4$,
   for all $k\in[\kappa]$,
   \begin{align*}
       T_k' \leq (T-\tau) \tp_k + \sqrt{3(T-\tau)\tp_k \slack} ,
   \end{align*}
   which implies that (using the inequality $\sqrt{x+y}\leq \sqrt{x}+y/(2\sqrt{x})$) 
   \begin{align*}
       \sqrt{T_k'} \leq \sqrt{(T-\tau) \tp_k} + \sqrt{\slack}.
   \end{align*}
   Plugging back into Theorem~\ref{thm:splitiwal}, with probability at least $1 - \delta$, we have
   \begin{align*}
     & \sum_{t=1}^T \E_{x\sim\sD_{\cX}} \big[p_t |\cF_{t-1}\big] \\
     & \leq \min \{2\theta r_0, 1\} \tau  + \sum_{k = 1}^{K} 4\theta_k 
       \bigg[R_k^* T_k' + 
   4\sqrt{2T'_k \slack}\bigg] \\
   & \leq \min \{2\theta r_0, 1\}\tau 
   + \sum_{k = 1}^{K} 4\theta_k 
   \bigg[R_k^* (T-\tau)\tp_k  \\
       & + R_k^*\sqrt{3(T-\tau)\tp_k \slack}  \\
       &\ \ + 4\sqrt{2(T-\tau)\tp_k \slack} + 4\sqrt{2}\slack \bigg]\\
   & \leq \min \{2\theta r_0, 1\}\tau 
   + \sum_{k = 1}^{K} \! 4\theta_k \!
   \bigg[\!R_k^* (T-\tau)\tp_k \!+ \\
   &  8\sqrt{(T-\tau)\tp_k \slack}\bigg] 
   + 4K \sqrt{2}\slack \,,
 \end{align*}
 where the last inequality uses fact that $R_k^*\leq 1$ and the assumption that $4\theta_k \leq 1$,
since otherwise the label complexity bound~\eqref{eq:splitlabel} is vacuous.

Using the inequality that $\sum_{k = 1}^{K} \sqrt{\tp_k} \leq \sqrt{K}$, and 
$\sum_{k = 1}^{K} \tp_k R_k^* \leq R^* - \gamma(K-1)$, we further upper bound
the above results. For the generalization error,
   \begin{align*}
       R(\h h_T) & \leq R^* - \gamma (K-1)
       +  \sqrt{\frac{32 K \slack}{T} } 
       + \frac{16K\slack}{T}.
   \end{align*}
For the expected number of labels,
\begin{align*}
  & \sum_{t=1}^T \E_{x\sim\sD_{\cX}} \big[p_t |\cF_{t-1}\big] \\
  & \leq \min \{2\theta r_0, 1\}\tau 
  + \sum_{k = 1}^{K} \!4\theta_k \!
  \bigg[\! R_k^* (T-\tau)\tp_k \!+ \\
  & \ \ 8\sqrt{(T-\tau)\tp_k \slack}\bigg] 
  + 4K \sqrt{2}\slack \\
  & \leq \min \{2\theta r_0, 1\}\tau 
  + 4\theta_{\max} \bigg[ \Big(\sum_{k = 1}^{K}  \tp_k R_k^*\Big)(T-\tau)  \\
  & \ \   + 8\sum_{k = 1}^{K} \sqrt{(T-\tau)\tp_k \slack} \bigg]
  + 4K \sqrt{2} \slack\\
  & \leq \min \{2\theta r_0, 1\}\tau  
  + 4\theta_{\max}(T-\tau) 
   \bigg[ R^* - \gamma(K-1) + \\
   & 8\sqrt{\frac{K \slack}{T-\tau}} \bigg] 
   +  \sqrt{32} K\slack .
\end{align*}
\end{proof}

\paragraph{Determining a fixed $\gamma$.}
The natural question arises as to how to set the value of threshold $\gamma$.
Comparing Theorem~\ref{thm:splitiwal2} to the generalization bound
of \iwal\ (Theorem~\ref{thm:iwal}), in order for \arbal\ to achieve
improved guarantees whenever it decides to split ($K\geq 2$), we need to have
\begin{align*}
    & R^* - \gamma (K-1)
 + \sum_{k = 1}^{K} 2 \sqrt{\frac{8 \tp_k \slack}{T} } \\
 & \leq R^* + 2\sqrt{\frac{8\log(2T(T+1)|\sH|^2 /\delta)}{T}}~,
\end{align*}
where we have dropped the lower order term $O(1/T)$. Neglecting
the small differences in the log terms, this turns out to be equivalent to the
following condition:
\[
2 \Delta_T \big[ \sum_{k = 1}^{K} \sqrt{\tp_k} -1 \big] \leq \gamma (K-1)~,
\]
where $\Delta_T=\sqrt{8 \slack/{T} }$.
Since $\sum_{k = 1}^{K}\sqrt{\tp_k} \leq \sqrt{K}$, 
and since this analysis applies only when \arbal\ decides to split ($K \geq 2$),
in order to guarantee improvement over \iwal\ in the
generalization bound for any possible value of $K$ that \arbal\ may select, 
it is sufficient to impose 
\begin{align*}
  \gamma & \geq 2 \Delta_T \bigg[ \frac{\sqrt{K}-1}{K- 1} \bigg]
      = \frac{2\Delta_T}{\sqrt{K}+1} \\
      &      \Rightarrow \gamma \geq \frac{2\Delta_T}{\sqrt{2}+1} = \frac{2}{\sqrt{2}+1}\,\sqrt{\frac{8 \slack}{T} }~.
\end{align*}

We first prove a more general version of Proposition~\ref{lemma:first_split_time} as follows.
\begin{lemma}\label{lemma:split_time}
  Let \arbal\ run with $\gamma_t = \Pr(\cX_{k_t}) \gain/2$. 
  With probability at least $1 - \delta/2$, for any region $\cX_k$
  created during the split phase, it will be split before round 
  $\Big\lceil{2\Slack{2} \big(\frac{4}{\gain} +1 \big)^2}/{\tp_k}\Big\rceil$
  unless \arbal\ has reached the end of the split phase.
\end{lemma}

\begin{proof}
  Fix a binary tree and a region $\cX_k$ that is split during the split phase, 
  and assume that splitting $\cX_k$ into $\cX_l \cup\cX_r$ 
  satisfies the assumption above. Recall that by assumption, such a split always exists. 
  Then by Lemma~\ref{lemma:split_concentrate}, with probability at least $1 - \delta/4$, 
  the corresponding empirical improvement satisfies
  \begin{align*}
      &L_{k, t}(\h h_{k, t}) - L_{k, t}(\h h_{lr,t}) + \Delta(T_{k, t}) \\
      & \geq R_k(h_k^*) - R_k(h_{lr}^*) \geq \gain ,
\end{align*}
where $\Delta(T_{k, t}) = \sqrt{{2\slack}/{T_{k, t}}}$.
Thus, for \arbal\ to split $\cX_k$ into $\cX_l,\cX_r$, it is sufficient to have
\begin{align*}
  & L_{k, t}(\h h_{k, t}) - L_{k, t}(\h h_{lr,t}) - \Delta(T_{k, t}) \geq \gain/2\\
  & \Leftarrow L_{k, t}(\h h_{k, t}) - L_{k, t}(\h h_{lr,t}) - \Delta(T_{k, t})  \\
  & \geq \gain - 2\Delta(T_{k, t}) \geq \gain/2\\
  & \Leftarrow \Delta(T_{k, t}) \leq \frac{\gain}{4} 
    \Leftarrow T_{k, t} \geq \frac{32\slack}{\gain^2}.
\end{align*}
Furthermore, by Theorem~\ref{thm:chernoff} and the covering number argument, 
with probability at least $1 - \delta/4$, 
for all $t \in [T]$ and all possible partitions with at most $\kappa$ regions,
\begin{align*}
    T_{k, t} \geq t\tp_k \Big(1- \sqrt{\frac{2\slack}{t\tp_k}}\Big),
\end{align*}
where $\tp_k = \Pr(\cX_k)$.
Thus, to split $\cX_k$ into $\cX_r\cup \cX_r$, it is sufficient to have 
\begin{align*}
    T_{k, t} \geq t\tp_k \Big(1- \sqrt{\mfrac{2\slack}{t\tp_k}}\Big) 
    \geq \mfrac{32\slack}{\gain^2}.
\end{align*}
Solving the quadratic inequality and using the fact that $\sqrt{x+y} \leq \sqrt{x} + \sqrt{y}$ when $x,y>0$,
we can write
\begin{align*}
  & t\tp_k \Big(1- \sqrt{\mfrac{2\slack}{t\tp_k}}\Big) 
    \geq \mfrac{32\slack}{\gain^2}\\
    & \Leftarrow \bigg(\sqrt{t\tp_k} - {\sqrt{\slack/2}} \bigg)^2 
    \geq \mfrac{32\slack}{\gain^2} 
    + \mfrac{\slack}{2}\\
  & \Leftarrow \sqrt{t\tp_k} 
    \geq \sqrt{\mfrac{\slack}{2}} + 
    \sqrt{\mfrac{32\slack}{\gain^2} 
    + \mfrac{\slack}{2}} \\
  & \Leftarrow \sqrt{t\tp_k} 
    \geq \sqrt{\mfrac{\slack}{2}} + 
    \sqrt{\mfrac{32\slack}{\gain^2}} 
    + \sqrt{\mfrac{\slack}{2}} \\
    & \Leftarrow t\tp_k \geq \bigg(\sqrt{2\slack}+
    \sqrt{\mfrac{32\slack}{\gain^2}}\, \bigg)^2\\
  & \Leftarrow t\tp_k \geq 2\slack\Big(\mfrac{4}{\gain} +1 \Big)^2
\end{align*}
Thus, for \arbal\ to split $\cX_k$, it is sufficient to have 
\[
    t \geq \Big\lceil{2\slack \Big(\frac{4}{\gain} +1 \Big)^2}/{\tp_k}\Big\rceil~.
\]
In other words, with probability at least $1 - \delta/2$, \arbal\ splits
region $\cX_k$ before time 
$\Big\lceil{2\slack \big(\frac{4}{\gain} +1 \big)^2}/{\tp_k}\Big\rceil$.
The statement holds for all (at most $\kappa-1$) splits.
\end{proof}

Thus, Lemma~\ref{lemma:split_time} provides an upper bound on the split time
for each region created during the split phase. 
In particular, for the original input space $\cX_k = \cX$, $\tp_k = 1$,
and we recover the result of Proposition~\ref{lemma:first_split_time}.
Combining Lemma~\ref{lemma:split_time} with the assumption
on the minimal subregion size after each split,
we can derive a lower bound on the number of splits
\arbal\ makes. 
\begin{repcorollary}{cor:num_splits}
  Let \arbal\ run with $\gamma_t = \Pr(\cX_{k_t}) \gain/2$. Then, with probability 
  at least $1 - \delta/2$, \arbal\ splits more than 
  $ \min \set[\Big]{\log_{1/c} \Big[\frac{\tau}
          {2\Slack{2} (4/{\gain} +1 )^2}\Big], \kappa-1}$
  times by the end of the split phase.
\end{repcorollary}
\begin{proof}
Assume that \arbal\ has only split $S < \kappa-1$ times by the end of $\tau$ rounds.
Then by assumption, the size of any subregion $\cX_k$ satisfies $\Pr(\cX_k)\geq c^{S}$.
According to Lemma~\ref{lemma:split_time}, if
\begin{align*}
    & \frac{2\Slack{2} \big(4/{\gain} +1 \big)^2}{\tp_k} 
 \leq \frac{2\Slack{2} \big(4/{\gain} +1 \big)^2}{c^{S}}
\leq \tau,
\end{align*}
then \arbal\ must have split the smallest region, thus have more than $S$ splits.
To avoid the contradiction, the number of splits $S$ must be at least 
\begin{align*}
    S \geq \log_{1/c} \bigg[\frac{\tau}{2\Slack{2}
  \big(4/{\gain} +1 \big)^2}\bigg].
\end{align*}
  \ignore{Finally, since \arbal\ cannot split more than $\kappa-1$ times, we have 
$ S \geq \min \set[\Big]{\log_{1/c} \Big[\frac{\tau}
  {2\kappa\log[\frac{16|\sH|^3 T^3 \kappa D}{\delta}] 
  (4/{\gain} +1 )^2}\Big], \kappa-1}$.}

Finally, since \arbal\ cannot split more than $\kappa-1$ times, we have 
$
S \geq \min \set[\Big]{\log_{1/c} \Big[\frac{\tau}
        {2\Slack{2} (4/{\gain} +1 )^2}\Big], \kappa-1}.
$
\end{proof}

Finally, with Corollary~\ref{cor:num_splits} handy, we can derive
an upper bound on the final best-in-class error after \arbal's split phase,
or equivalently a lower bound on the improvement from a single region's
best-in-class error $R^*$. 
We present the full learning guarantees
of \arbal\ with adaptive $\gamma_t$ in the following Theorem~\ref{thm:splitiwal3}.
For simplicity, we assume that the lower bound on $S$ does not exceed 
the hard constraint of $\kappa-1$, 
so that we can get rid of the $\min\{\cdot\}$ operator.

\begin{theorem}\label{thm:splitiwal3}
Assume a run of \arbal\ over $T$ rounds with $\gamma_t = \Pr(\cX_{k_t}) \gain/2$.
Then, with probability at least $1 - 3\delta/2$,
\begin{align*}
  R(\h h_T) & \leq R_{U}
   + 2 \sqrt{\frac{8 K \slack}{T} }
   + \frac{16K\slack}{T}~,\\
  \tau_T & \leq \min \{2\theta r_0, 1\}\tau    
         + 4\theta_{\max}(T-\tau) \Big[R_U  
    + 8 \sqrt{\frac{K \slack}{T-\tau}}\Big]\\
&  +  \sqrt{32}K \slack.
\end{align*}
where 
$ R_U = R^*-\frac{\gain c }{2(1 - c)} 
\big(1-{2\Slack{2}
  (\frac{4}{{\gain}} +1 )^2}/{\tau}\big)$
is an upper bound on the best-in-class error obtained by \arbal.
\end{theorem}

\begin{proof}
  When \arbal\ splits region $\cX_k$, by Corollary~\ref{thm:split},
  with high probability, the global best-in-class error is 
  improved by at least 
  $\Pr(\cX_k)\big[R_k(h_k^*) - R_k(h_{lr}^*) \big] \geq \Pr(\cX_k)\gain/2$,
  which means the global improvement depends on the size of the splitting region. 
  Assume \arbal\ has made $S=K-1$ splits into $K$ regions. Again, by assumption, 
  at $s$-th split, $s\in[S]$, the size of the splitting region must be at least $c^{s}$. 
  Thus, with probability at least $1 - \delta/2$, the improvement in the 
  best-in-class error after $S$ splits is at least 
\begin{align*}
    & \frac{\gain}{2} \Big(\sum_{s=1}^{S} c^{s} \Big) 
  = \frac{\gain c}{2}\Big(\frac{1 - c^{S}}{1 - c}\Big) \\
  &  \geq \frac{\gain c }{2(1 - c)} \Big(1- 
  \frac{2\Slack{2} \big(\frac{4}{\gain} +1 \big)^2}
  {\tau}\Big),
\end{align*}
where the last inequality follows from the lower bound of $S$ in Corollary~\ref{cor:num_splits}.
Thus, the best-in-class error after the splits made by \arbal\ is upper bounded by
\begin{align*}
  R^*-\frac{\gain c }{2(1 - c)} 
  \Big(1 - \frac{2 \Slack{2} 
  \big(\frac{4}{\gain} +1 \big)^2}{\tau}\Big),
\end{align*}
where $R^*$ is the global best-in-class error on $\cX$ before splitting. 
The rest of the proof follows from the proof of Theorem~\ref{thm:splitiwal2}.
\end{proof}

Theorem~\ref{thm:splitiwal3} relates the improvement in the best-in-class error 
with $c, \rho$ and $\tau$. When $c$ increases, which means \arbal\ tends to split more evenly,
then there are more improvements in the best-in-class error, partially because 
there are likely to be more splits. Similarly, when $\tau$ or $\rho$ increases,
then there tends to be a larger improvement in the best-in-class error, as expected.

\clearpage
\section{More experimental results}
\label{app:moreexp}
This appendix contains all plots omitted from the main body of the paper.

\begin{table}[t]
  \caption{Binary classification dataset summary: $N$ denotes the number of samples, 
  $D$ the number of features (or input space dimension), 
  and $r$ the relative size of the minority class. 
  Datasets are ordered by increasing $N$.}
  \label{tb:data_summary}
  %\vskip -0.1in
  \centering
  \begin{tabular}{lrrr}\toprule
    Dataset & $N\ \ \ \ $ &\ \ $D\ \ $ & $r\ \ \ $\\ \midrule
      \texttt{\small kin8nm}        &8,192       &8      &0.491\\
      \texttt{\small bank8fm}       &8,192       &8      &0.404\\
      \texttt{\small puma8NH}       &8,192       &8      &0.498\\
      \texttt{\small visualizing\_soil} &8,641   &4      &0.450\\
      \texttt{\small delta\_elevators}  &9,517   &6      &0.497\\
      \texttt{\small jm1}           &10,880      &21     &0.193\\
      \texttt{\small phishing}      &11,055      &68     &0.443\\
      \texttt{\small mnist35}       &11,552      &784    &0.469 \\
      \texttt{\small egg}           &14,980      &14     &0.449\\
      \texttt{\small elevators}     &16,599      &18     &0.309\\
      \texttt{\small magic04}       &19,020      &10     &0.352\\
      \texttt{\small house16H}      &22,784      &16     &0.296\\
      \texttt{\small nomao }        &34,465      &118    &0.286\\
      \texttt{\small fried}         &40,768      &10     &0.499\\
      \texttt{\small mv}            &40,768      &12     &0.403\\
      \texttt{\small shuttle}       &43,500      &9      &0.216\\
      \texttt{\small electricity}   &45,312      &14     &0.425\\
      \texttt{\small a9a  }         &48,842      &123    &0.239\\
      \texttt{\small ijcnn1 }       &49,990      &22     &0.097\\
      \texttt{\small codrna }       &59,535      &8      &0.333\\ 
      \texttt{\small runorwalk}     &88,588      &6      &0.499\\
      \texttt{\small higgs}         &98,049      &28     &0.471\\
      \texttt{\small MiniBooNE}     &130,064     &50     &0.281\\
      \texttt{\small skin  }        &245,057     &3      &0.208\\
      \texttt{\small covtype}       &581,012     &54     &0.488\\
  \bottomrule
  \end{tabular}
  \vskip -0.2in
\end{table}

In Table~\ref{tb:data_summary}, we show summary statistics for all datasets used in our experiments.

In Figures~\ref{fig:expmis_1}-\ref{fig:expmis_5}, we present the following results for 25 datasets
under $\tau=800$ and $\kappa=20$:
the results of comparing fixed vs. adaptive $\gamma_t$,
the results of \arbal\ as contrasted to the
baselines described in the main text, that is, \oriwal, \iwal, and \margin,
and the results of \arbal, using binary tree vs. hierarchical clustering splitting method.

Adaptive $\gamma_t$ yields superior prediction performance over fixed $\gamma$
on most datasets, except for \texttt{\small {jm1}} and \texttt{\small {elevators}},
where fixed $\gamma$ rarely splits yet adaptive $\gamma_t$ fully splits into $20$ regions,
suggesting that $\gamma_t$ is overly aggressive due to the simplification of the slack term
in the splitting criterion.
In the remaining plots, we show the performance of $\gamma_t$ except for
the \texttt{\small {jm1}} and \texttt{\small {elevators}} where we use fixed $\gamma$.

\clearpage
\begin{figure*}[ht]
\begin{center}
\subfigure{\centering\includegraphics[width=0.3\textwidth,,trim= 5 10 10 5,clip=true]{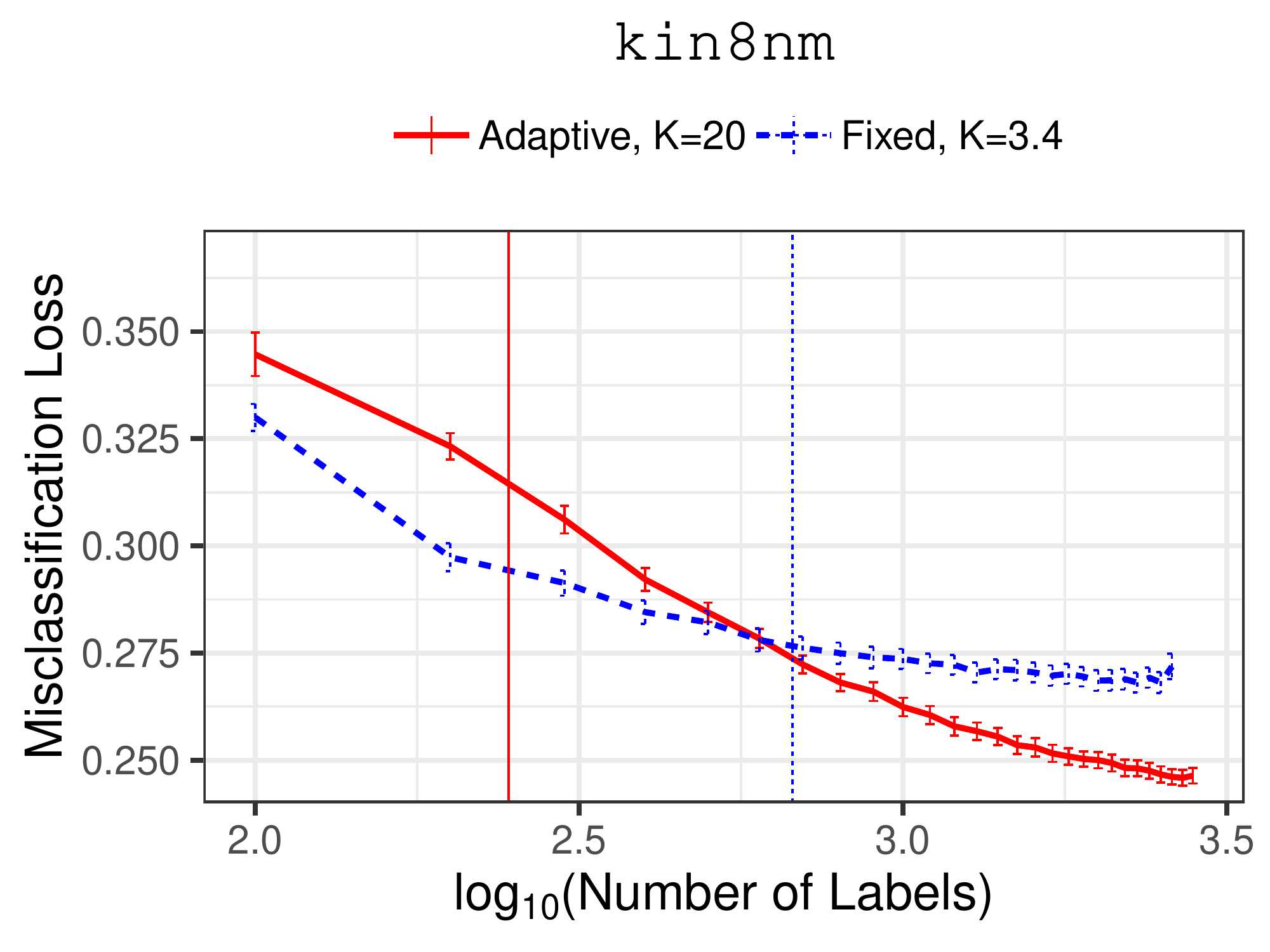}}
\subfigure{\centering\includegraphics[width=0.3\textwidth,,trim= 5 10 10 5,clip=true]{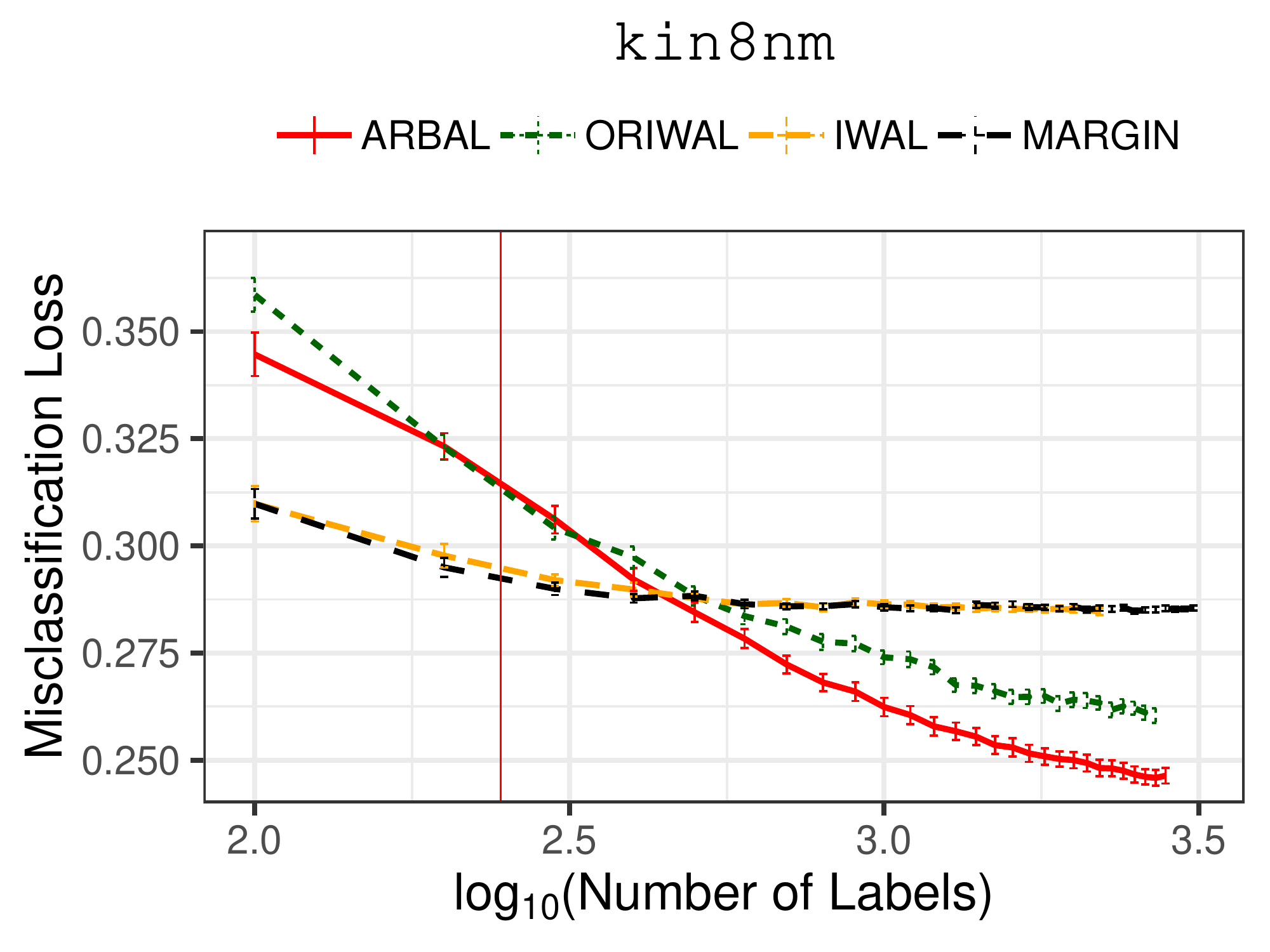}}
\subfigure{\centering\includegraphics[width=0.3\textwidth,,trim= 5 10 10 5,clip=true]{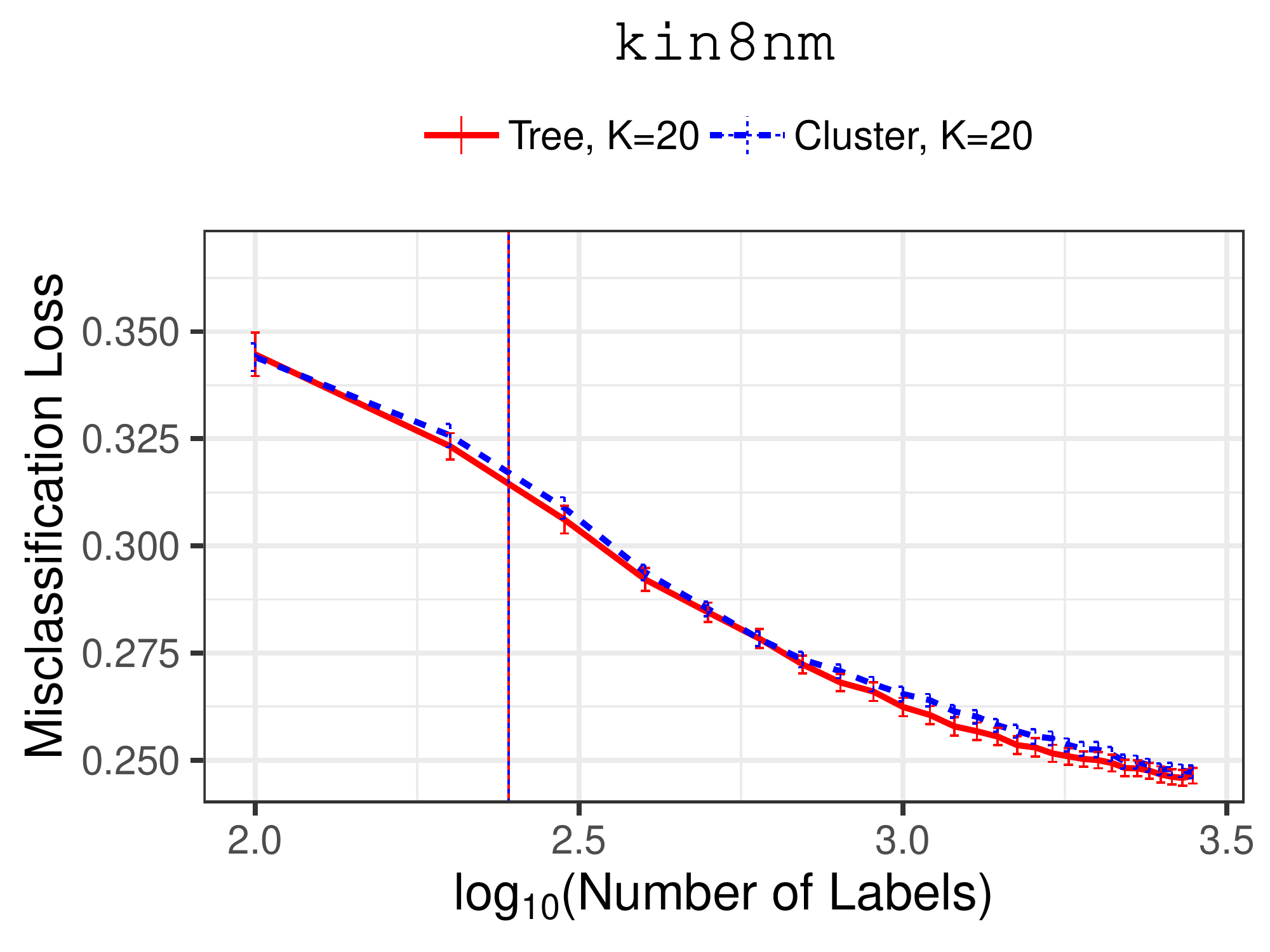}}\\
\subfigure{\centering\includegraphics[width=0.3\textwidth,,trim= 5 10 10 5,clip=true]{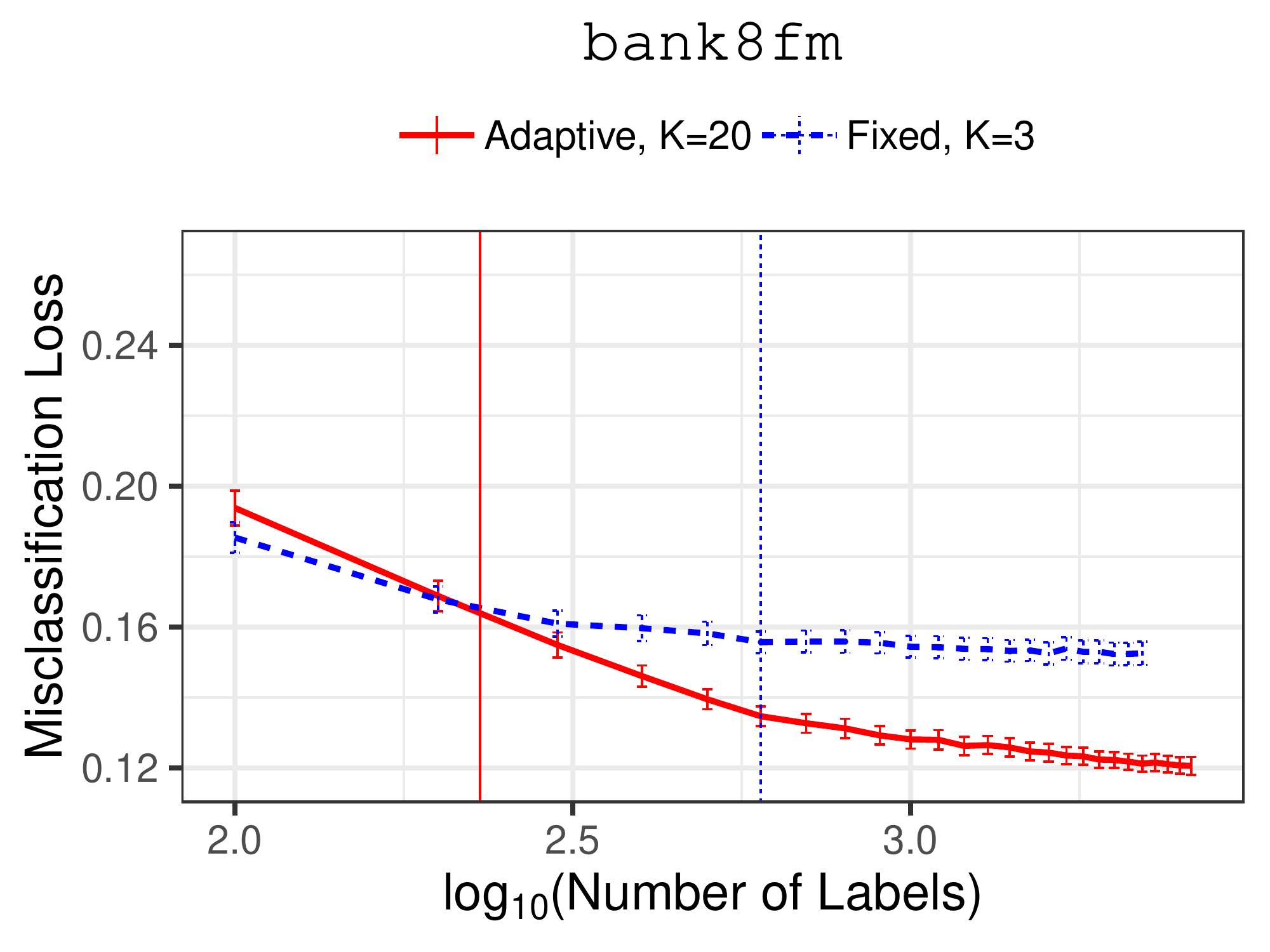}}
\subfigure{\centering\includegraphics[width=0.3\textwidth,,trim= 5 10 10 5,clip=true]{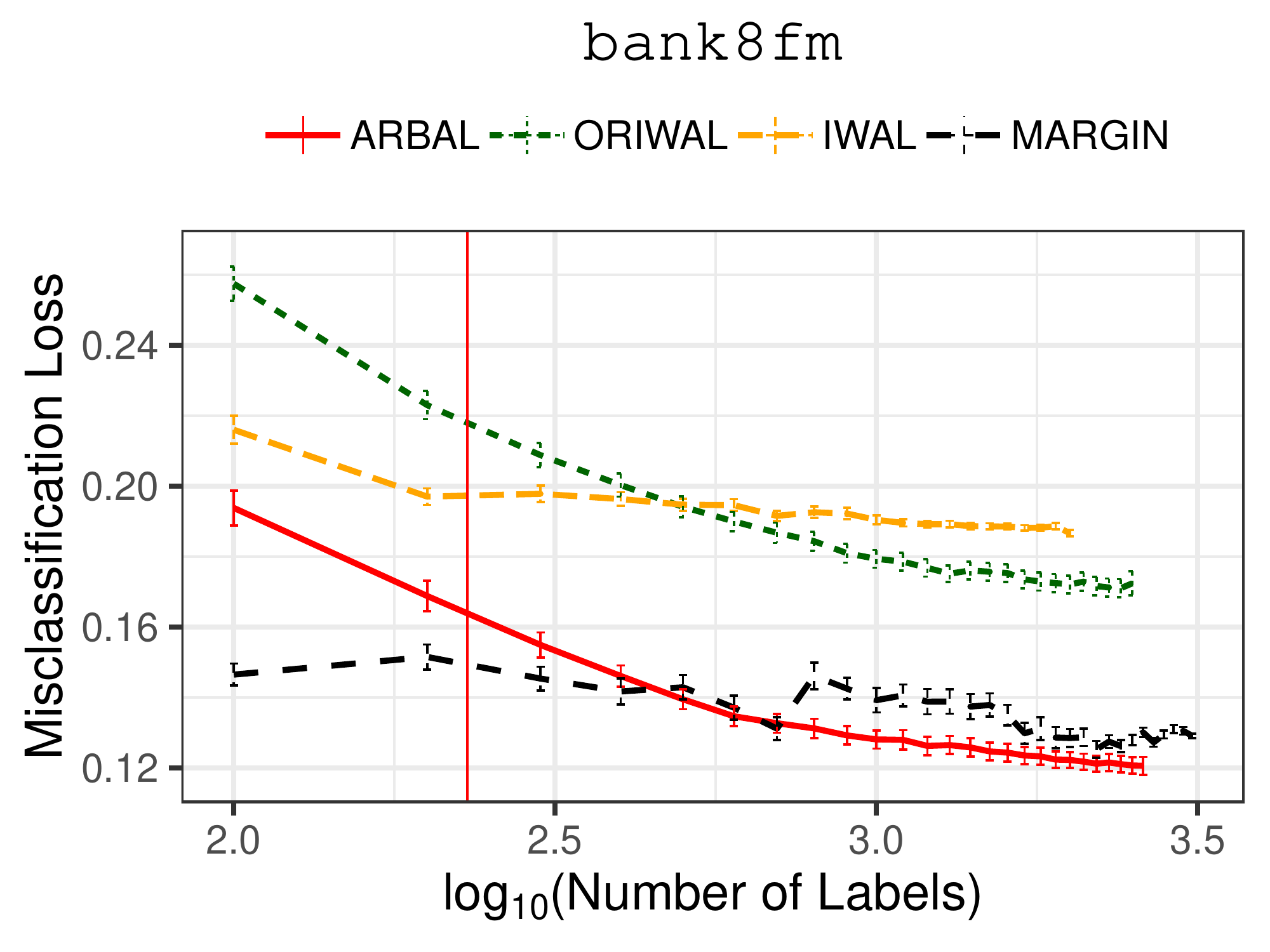}}
\subfigure{\centering\includegraphics[width=0.3\textwidth,,trim= 5 10 10 5,clip=true]{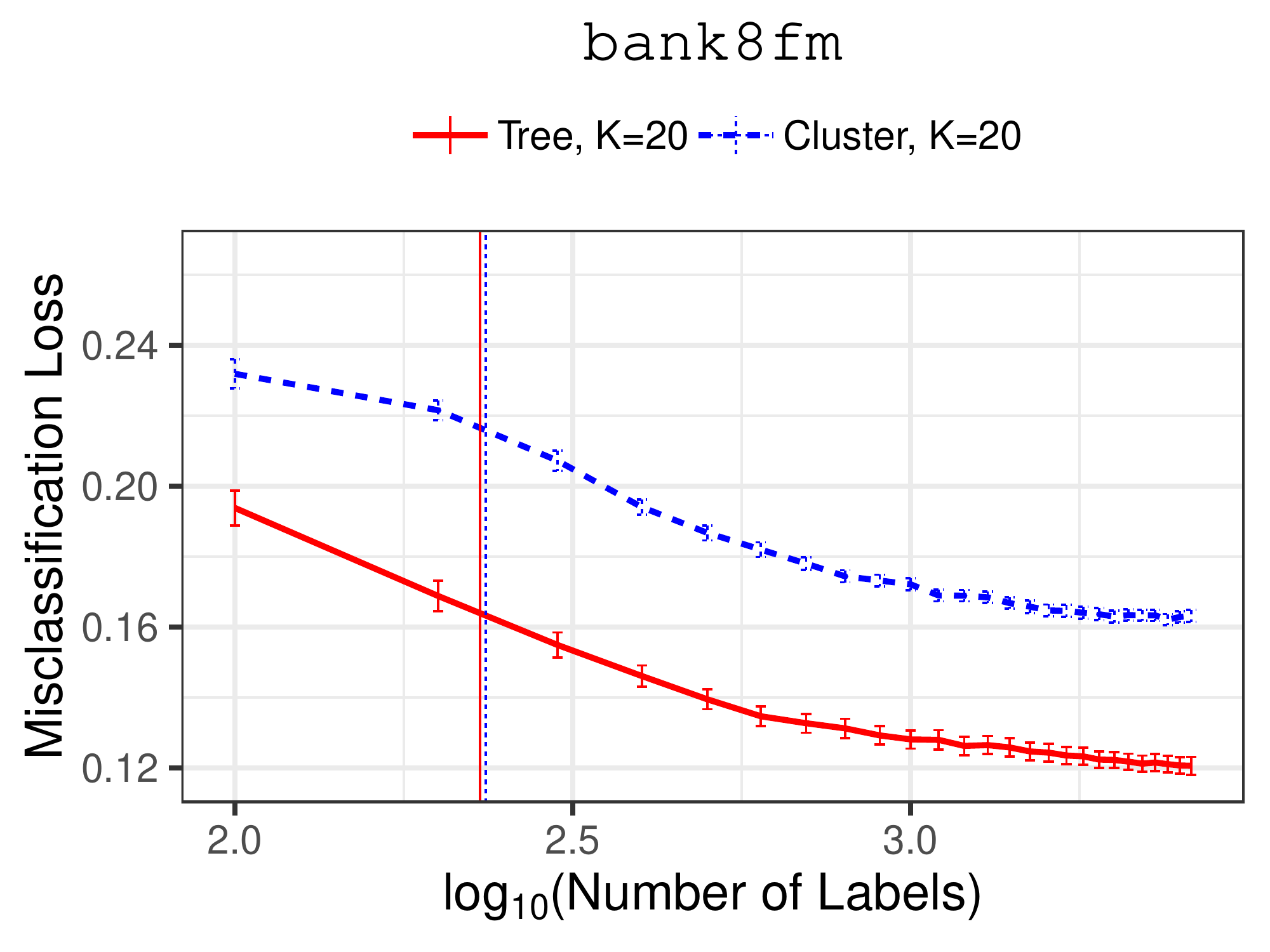}}\\
\subfigure{\centering\includegraphics[width=0.3\textwidth,,trim= 5 10 10 5,clip=true]{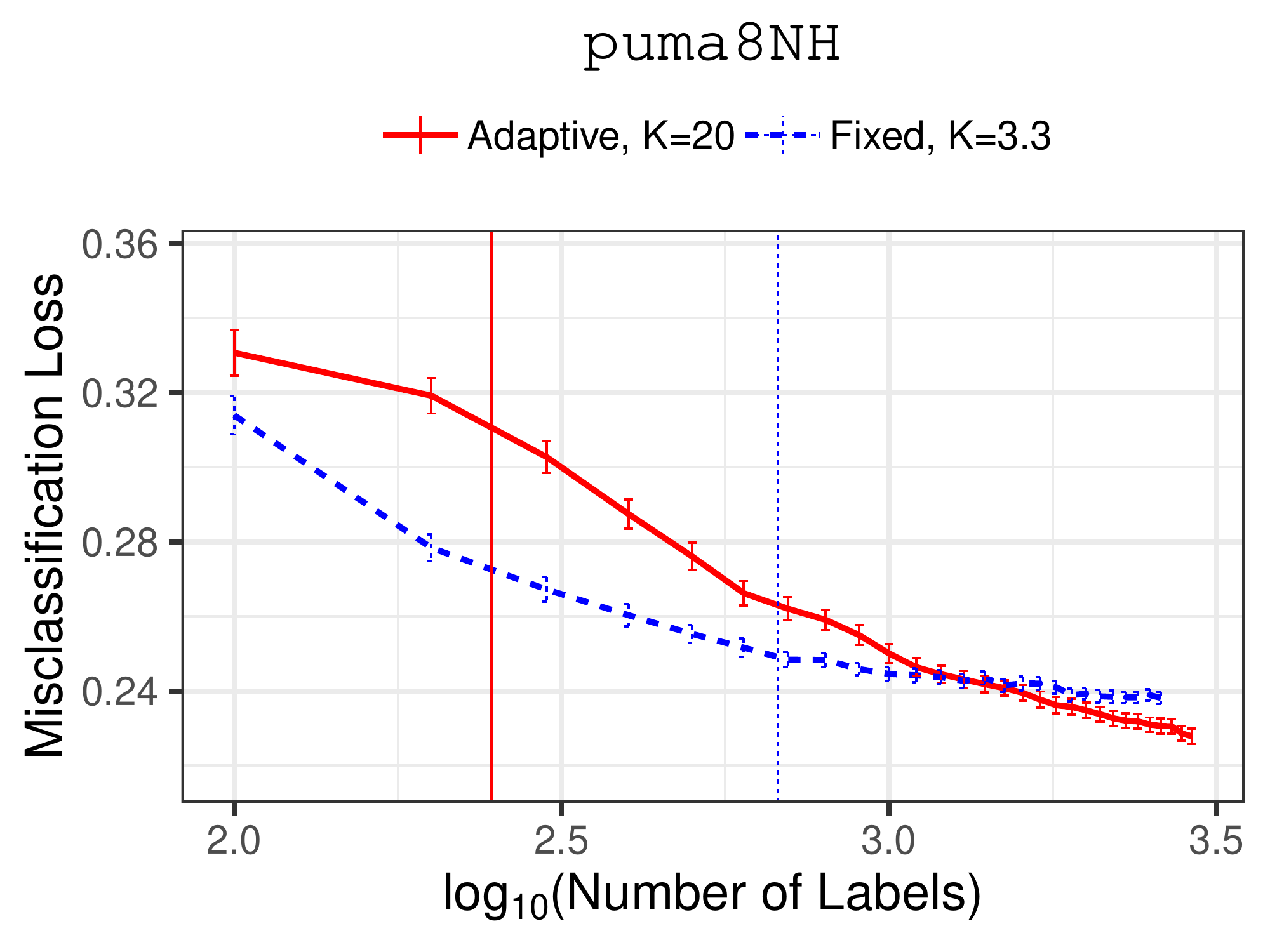}}
\subfigure{\centering\includegraphics[width=0.3\textwidth,,trim= 5 10 10 5,clip=true]{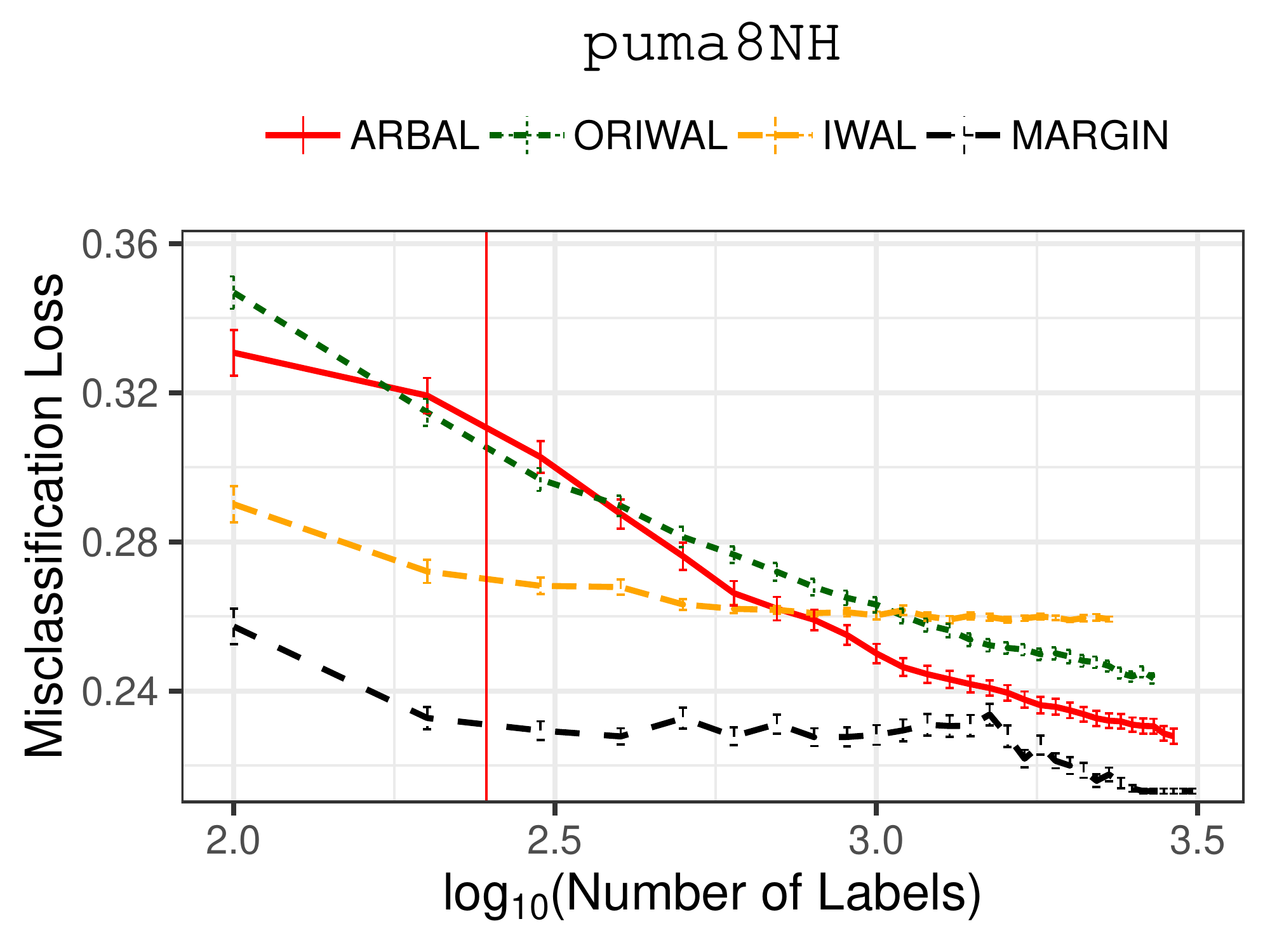}}
\subfigure{\centering\includegraphics[width=0.3\textwidth,,trim= 5 10 10 5,clip=true]{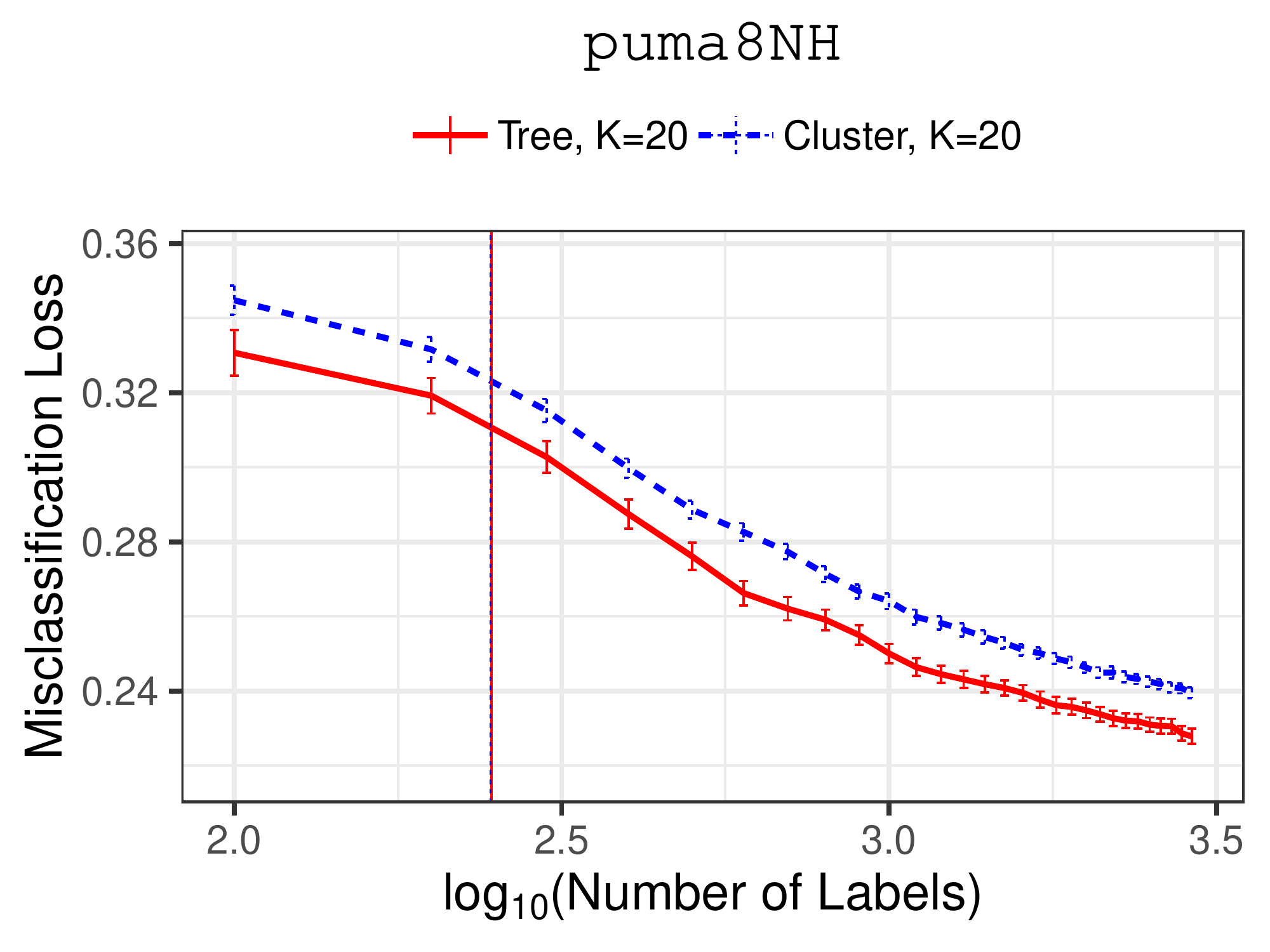}}\\
\subfigure{\centering\includegraphics[width=0.3\textwidth,,trim= 5 10 10 5,clip=true]{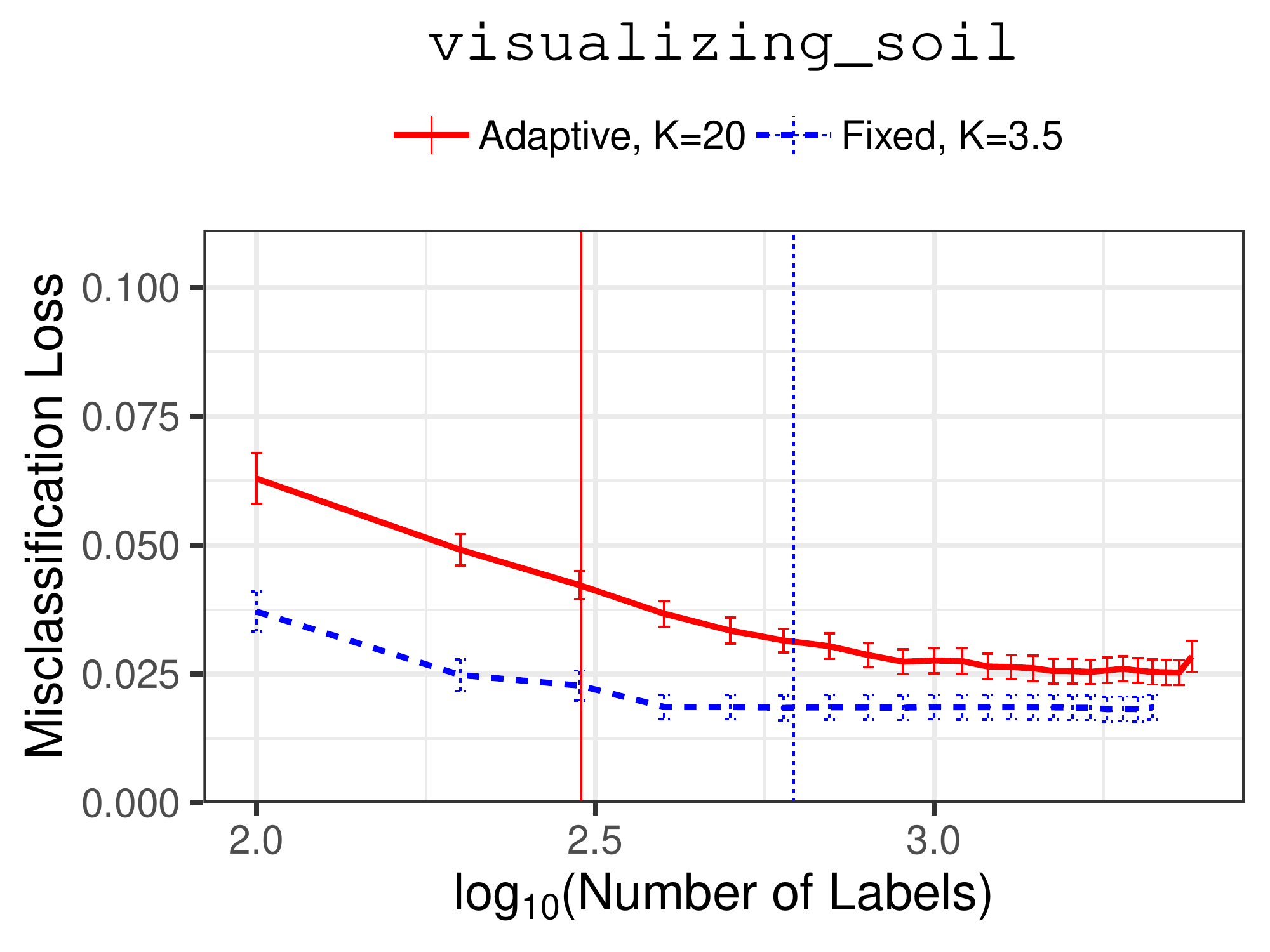}}
\subfigure{\centering\includegraphics[width=0.3\textwidth,,trim= 5 10 10 5,clip=true]{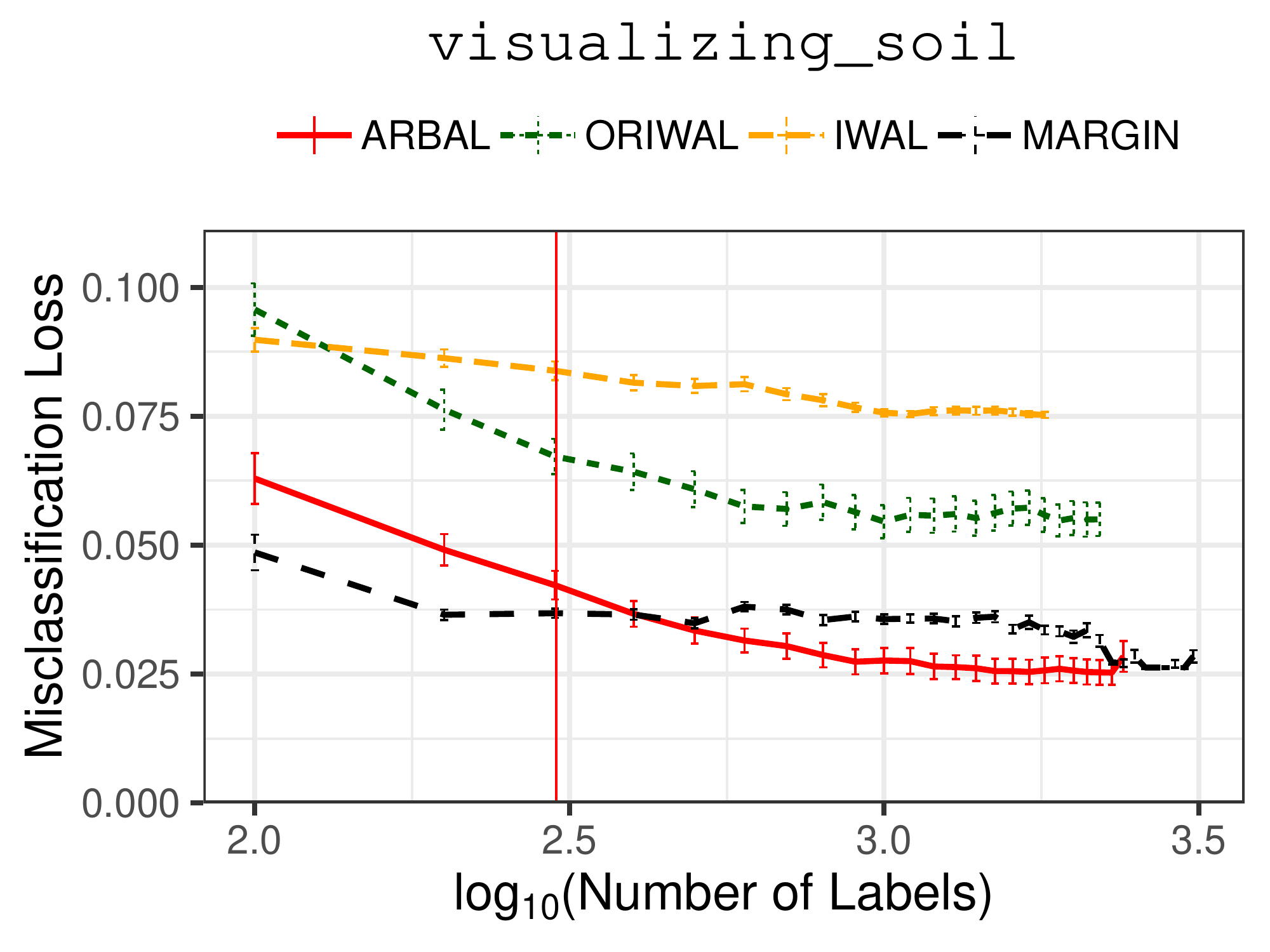}}
\subfigure{\centering\includegraphics[width=0.3\textwidth,,trim= 5 10 10 5,clip=true]{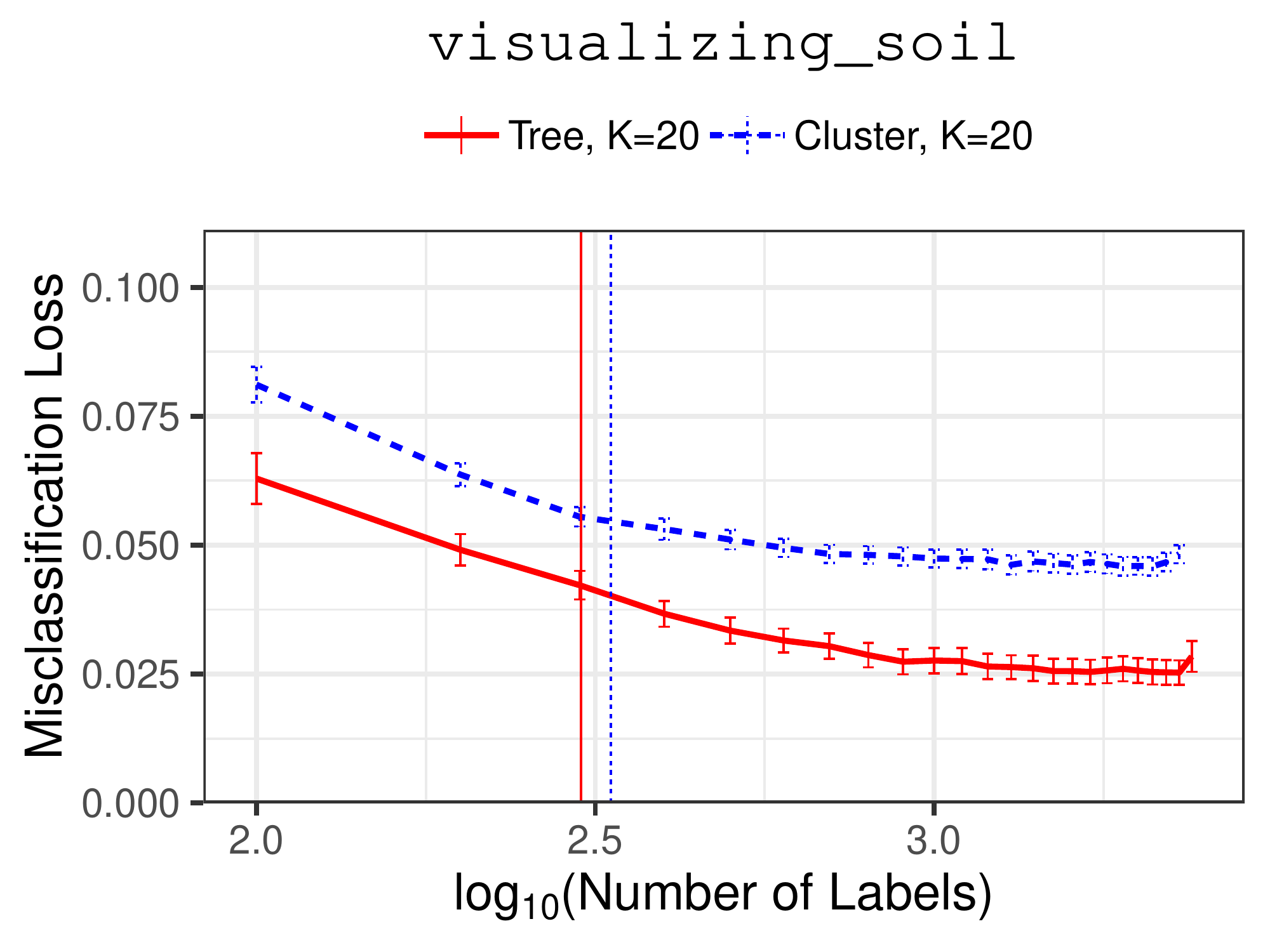}}\\
\subfigure{\centering\includegraphics[width=0.3\textwidth,,trim= 5 10 10 5,clip=true]{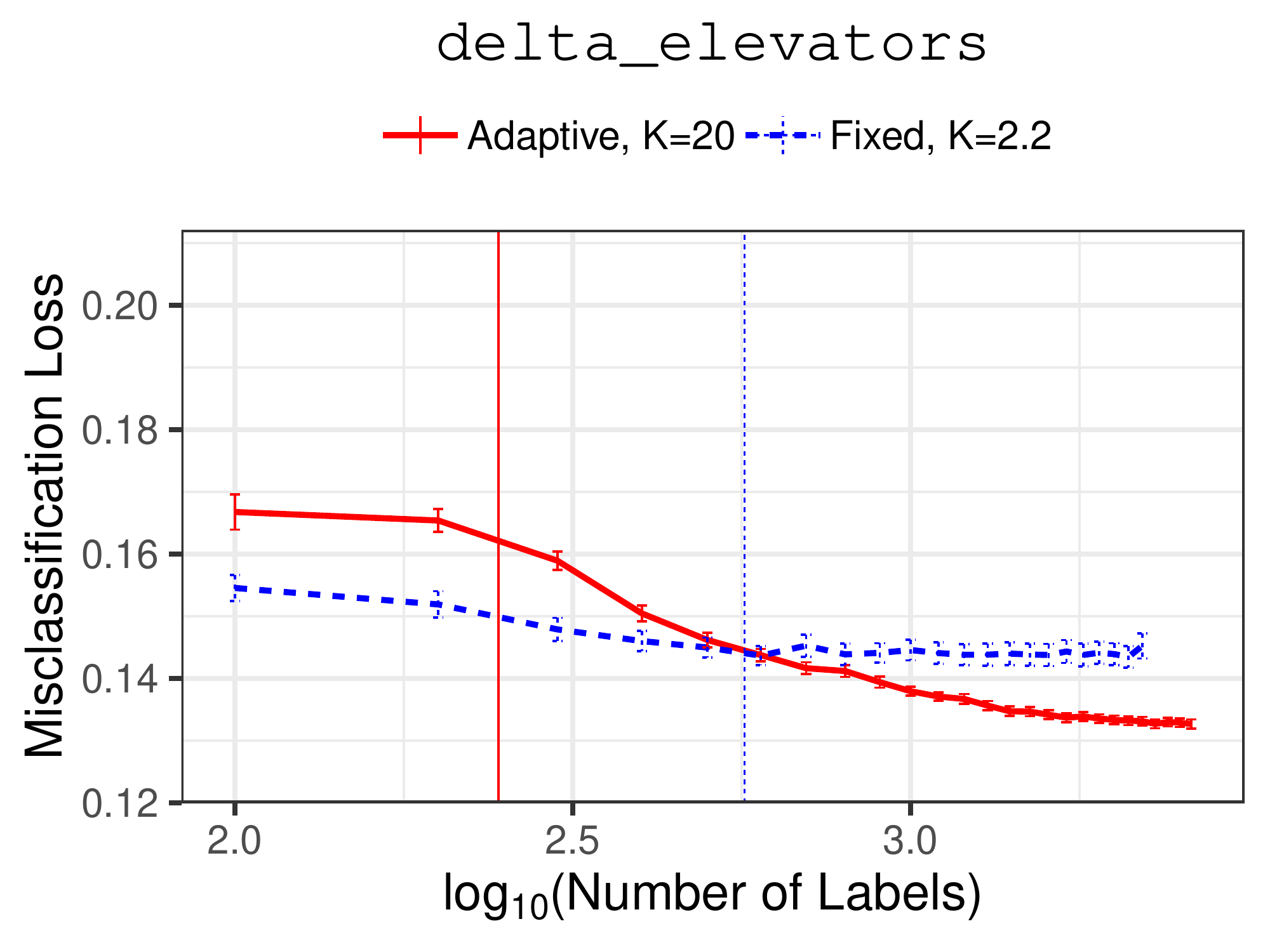}}
\subfigure{\centering\includegraphics[width=0.3\textwidth,,trim= 5 10 10 5,clip=true]{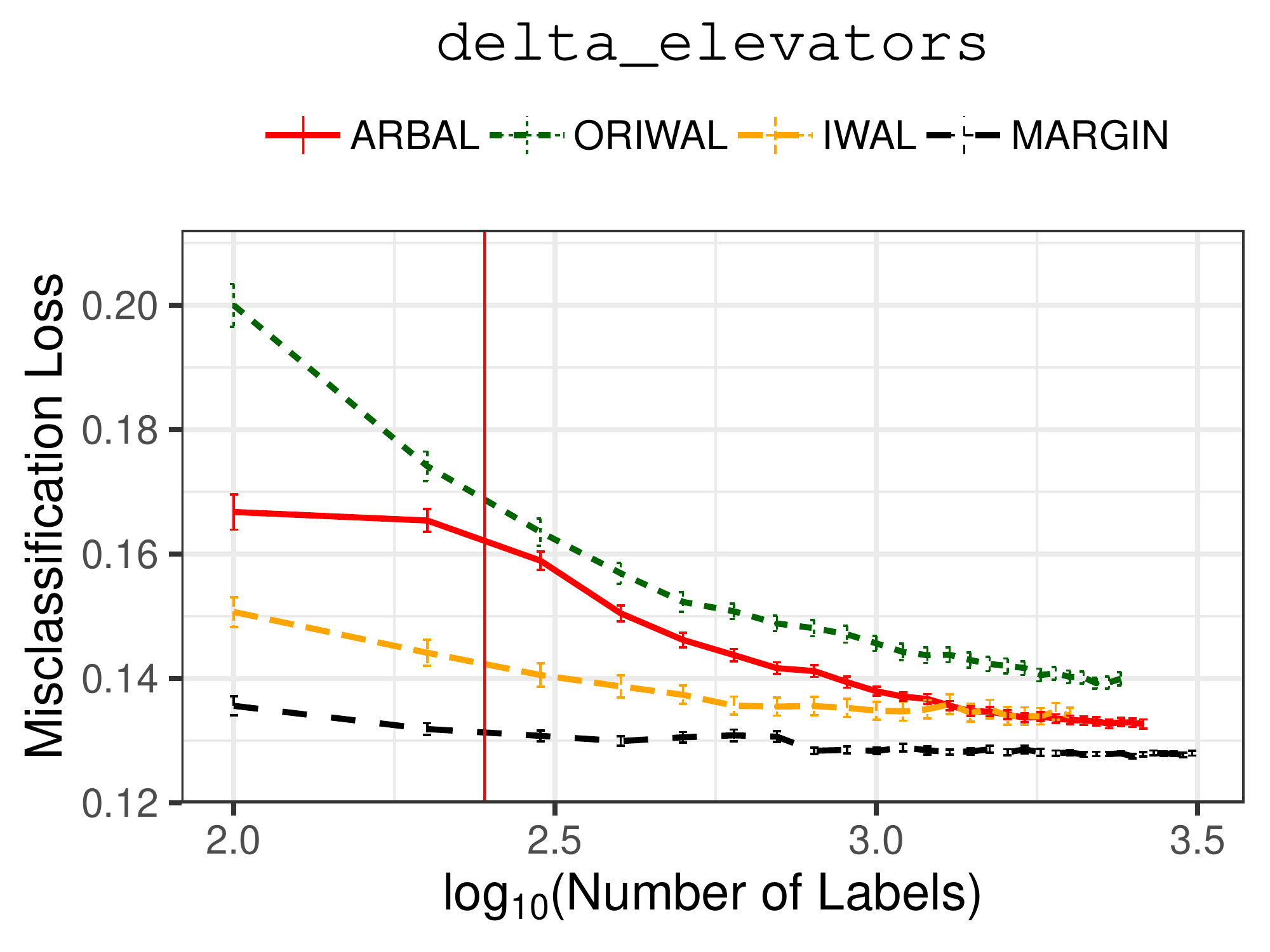}}
\subfigure{\centering\includegraphics[width=0.3\textwidth,,trim= 5 10 10 5,clip=true]{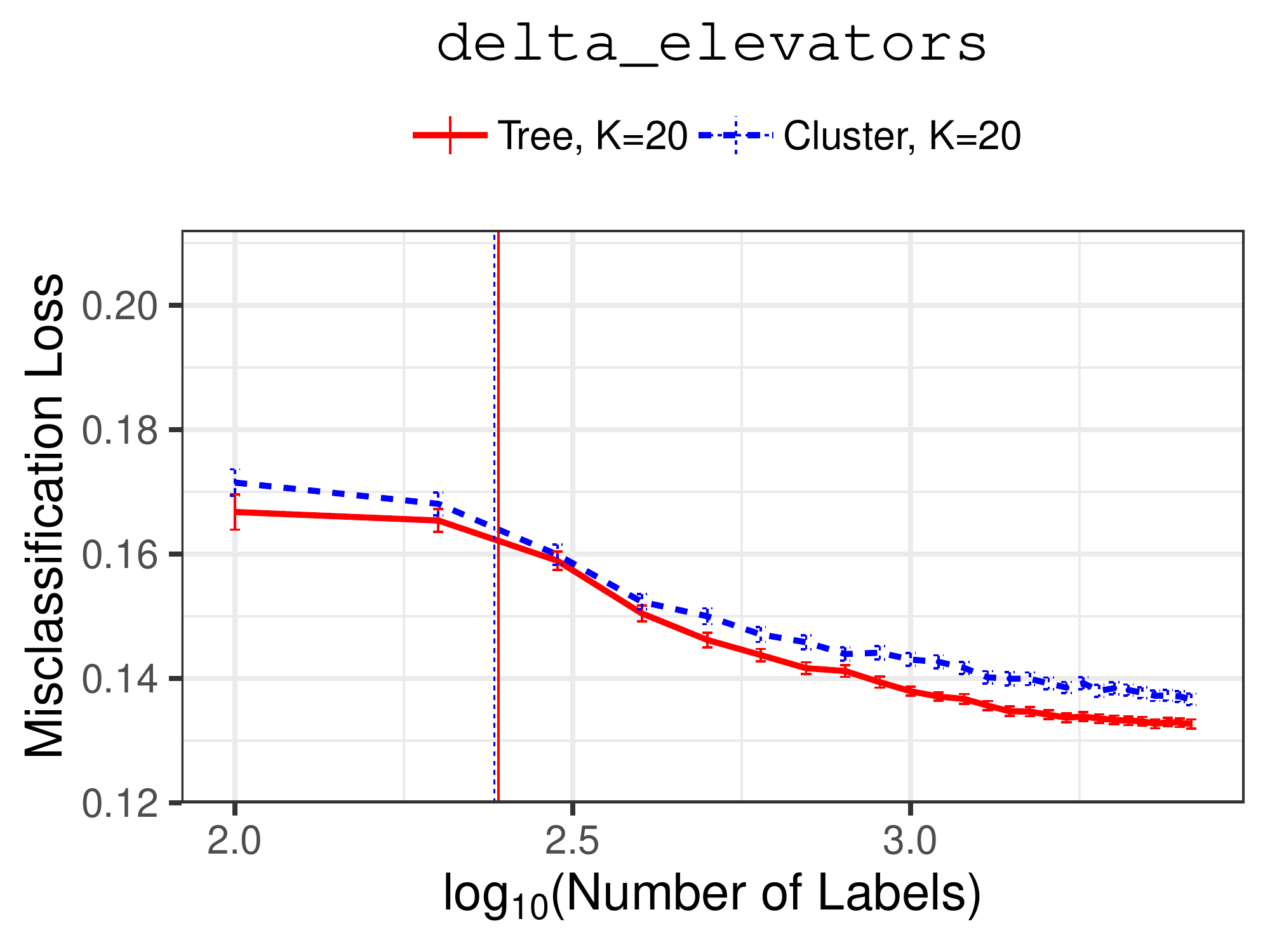}}\\
\end{center}
\caption{Misclassification loss on hold out test data versus number of labels requested ($\log_{10}$ scale). 
Left: \arbal\ with fixed and adaptive threshold $\gamma$.
Middle: \arbal, \riwal, \iwal, and \margin.
Right: \arbal\ with different partitioning methods: binary tree and hierarchical clustering.
For $\kappa=20$ and $\tau=800$,
dataset \texttt{\small{kin8nm}}, \texttt{\small{bank8fm}}, \texttt{\small{puma8NH}}, \texttt{\small{visualizing soil}}, \texttt{\small{delta elevators}}.
For left and right plots, we give the average number of resulting regions $K$ in the legend. The vertical lines indicate
when \arbal\ transits from the first to the second phase.}
\label{fig:expmis_1}
\vskip -0.2in
\end{figure*}
\begin{figure*}[ht]
\begin{center}
\subfigure{\centering\includegraphics[width=0.3\textwidth,,trim= 5 10 10 5,clip=true]{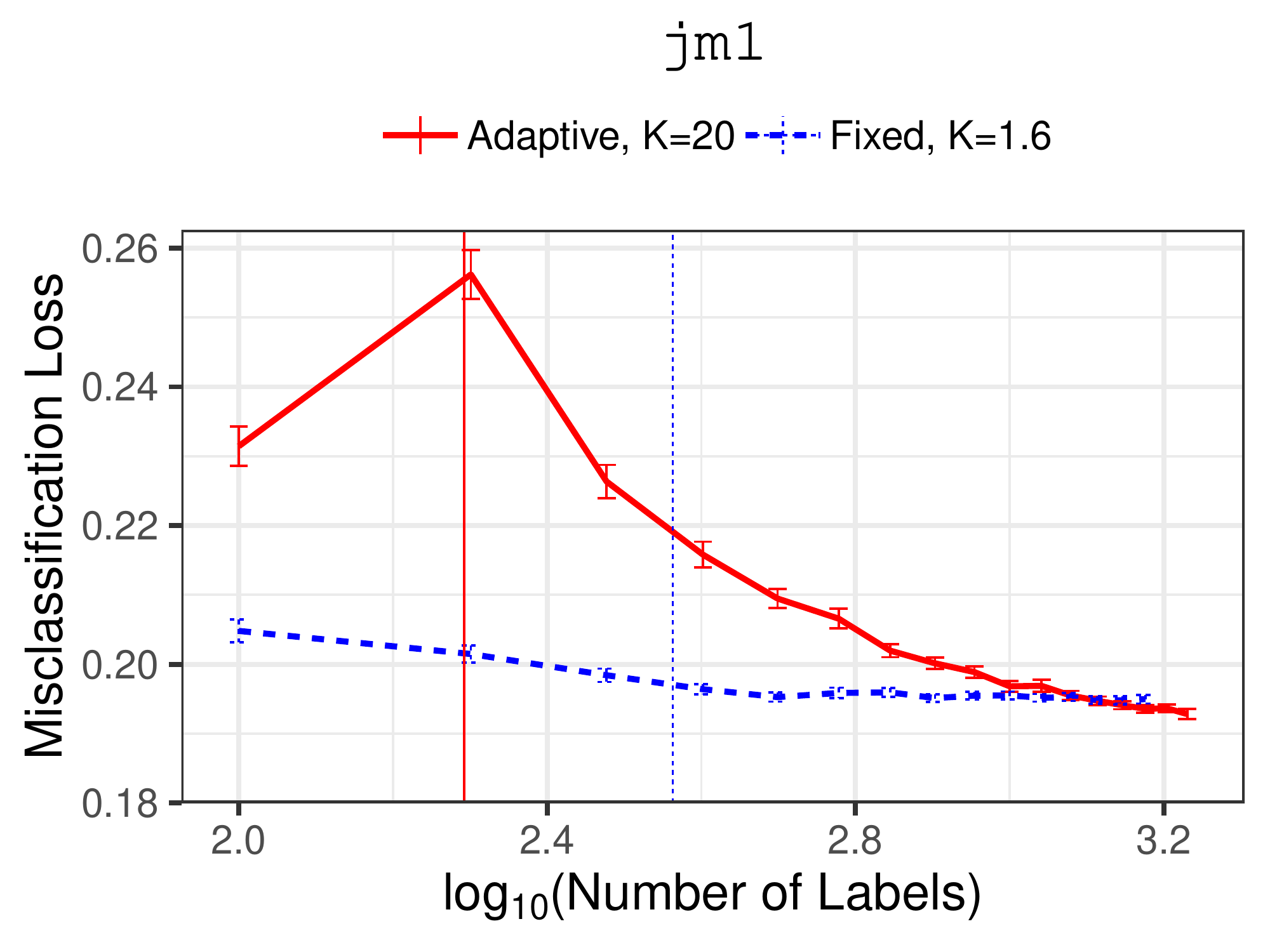}}
\subfigure{\centering\includegraphics[width=0.3\textwidth,,trim= 5 10 10 5,clip=true]{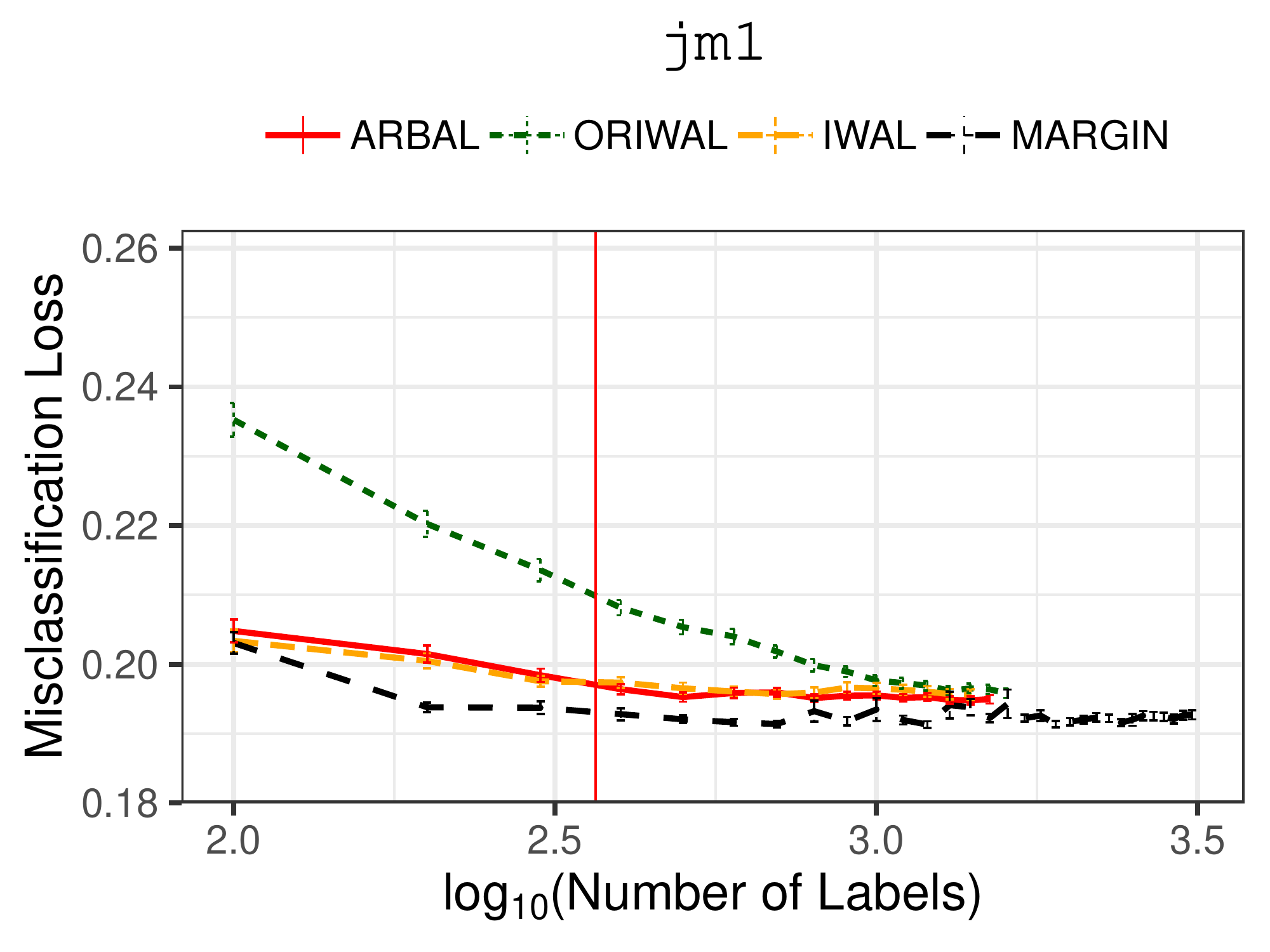}}
\subfigure{\centering\includegraphics[width=0.3\textwidth,,trim= 5 10 10 5,clip=true]{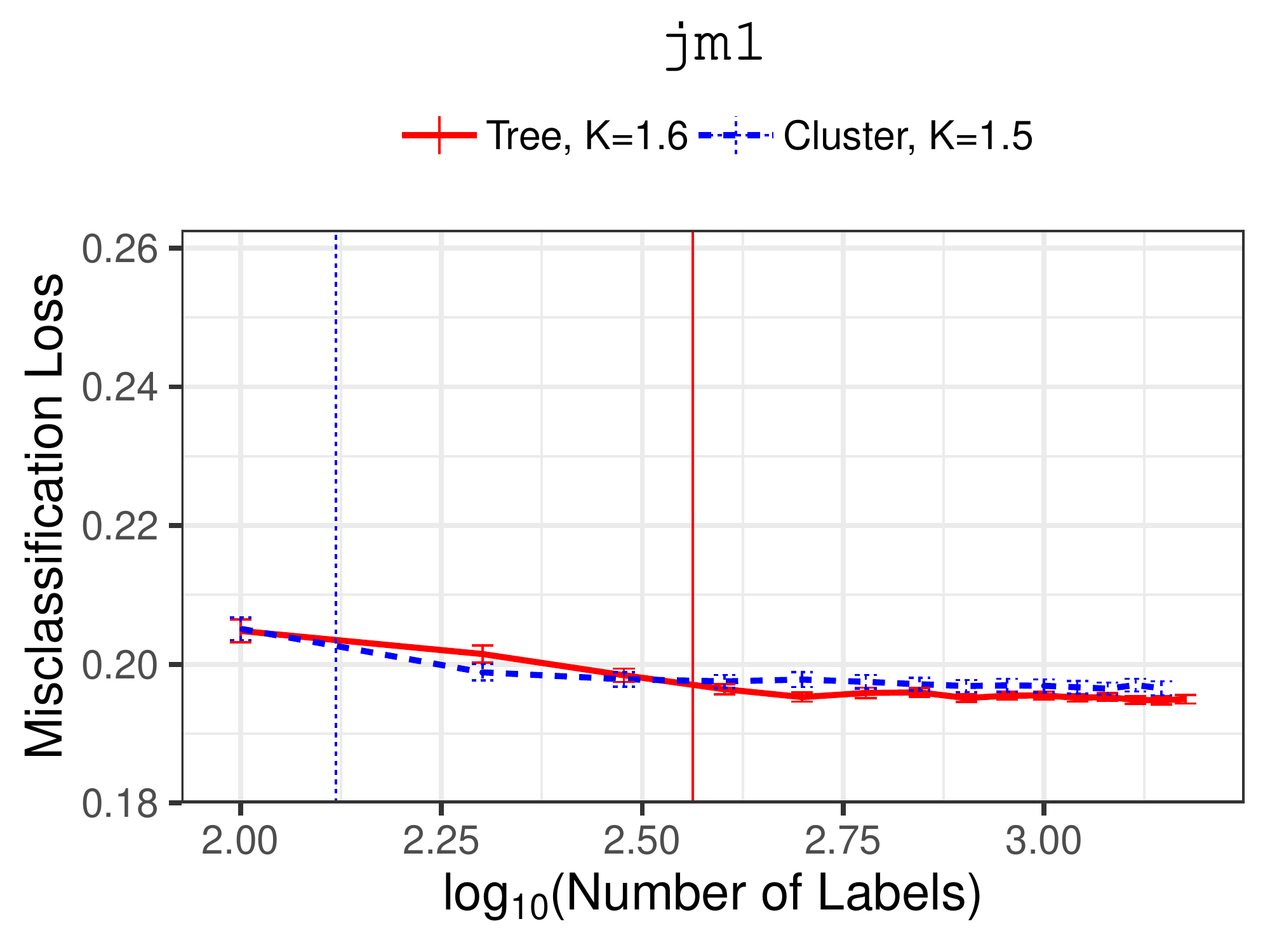}}\\
\subfigure{\centering\includegraphics[width=0.3\textwidth,,trim= 5 10 10 5,clip=true]{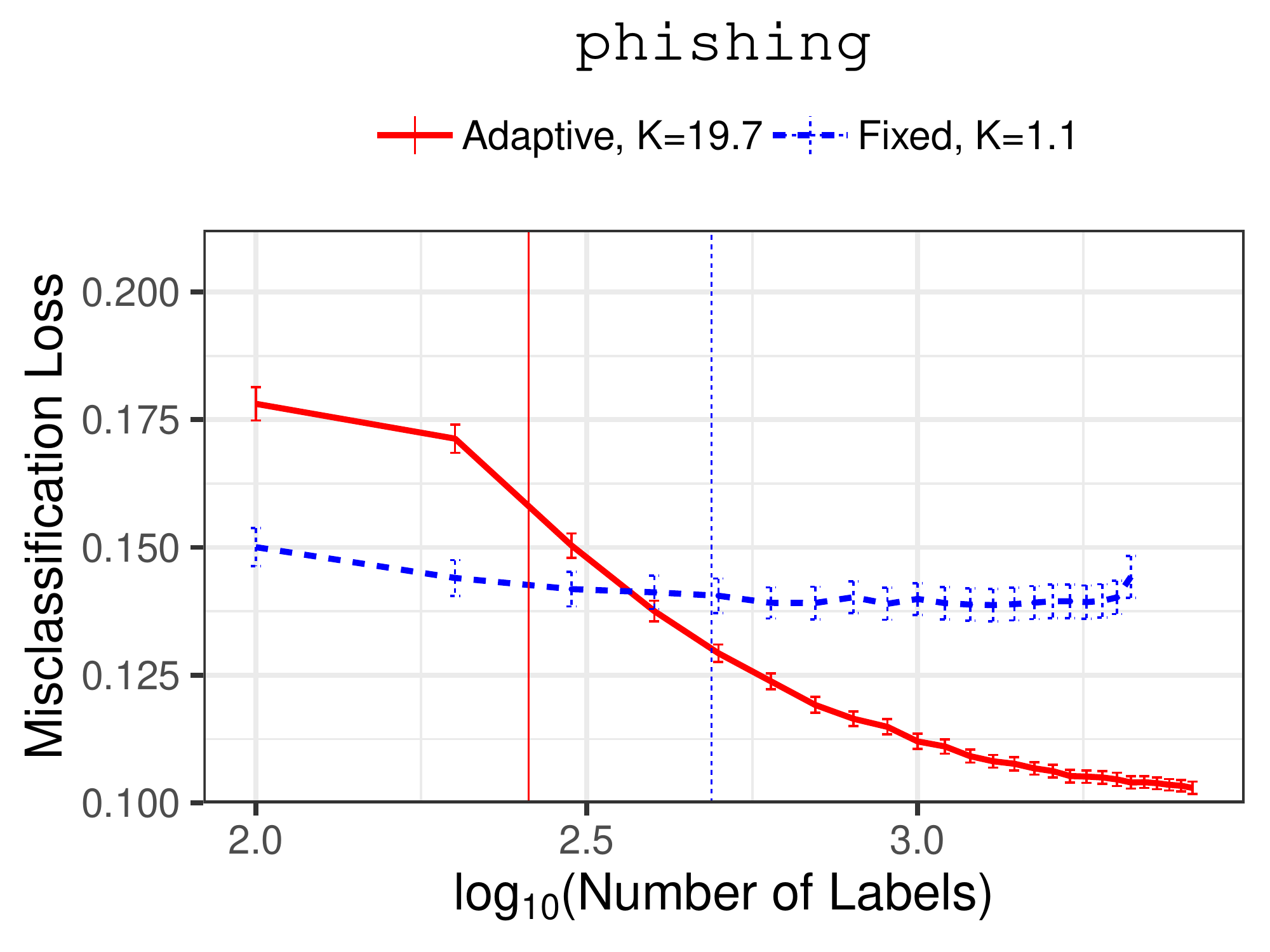}}
\subfigure{\centering\includegraphics[width=0.3\textwidth,,trim= 5 10 10 5,clip=true]{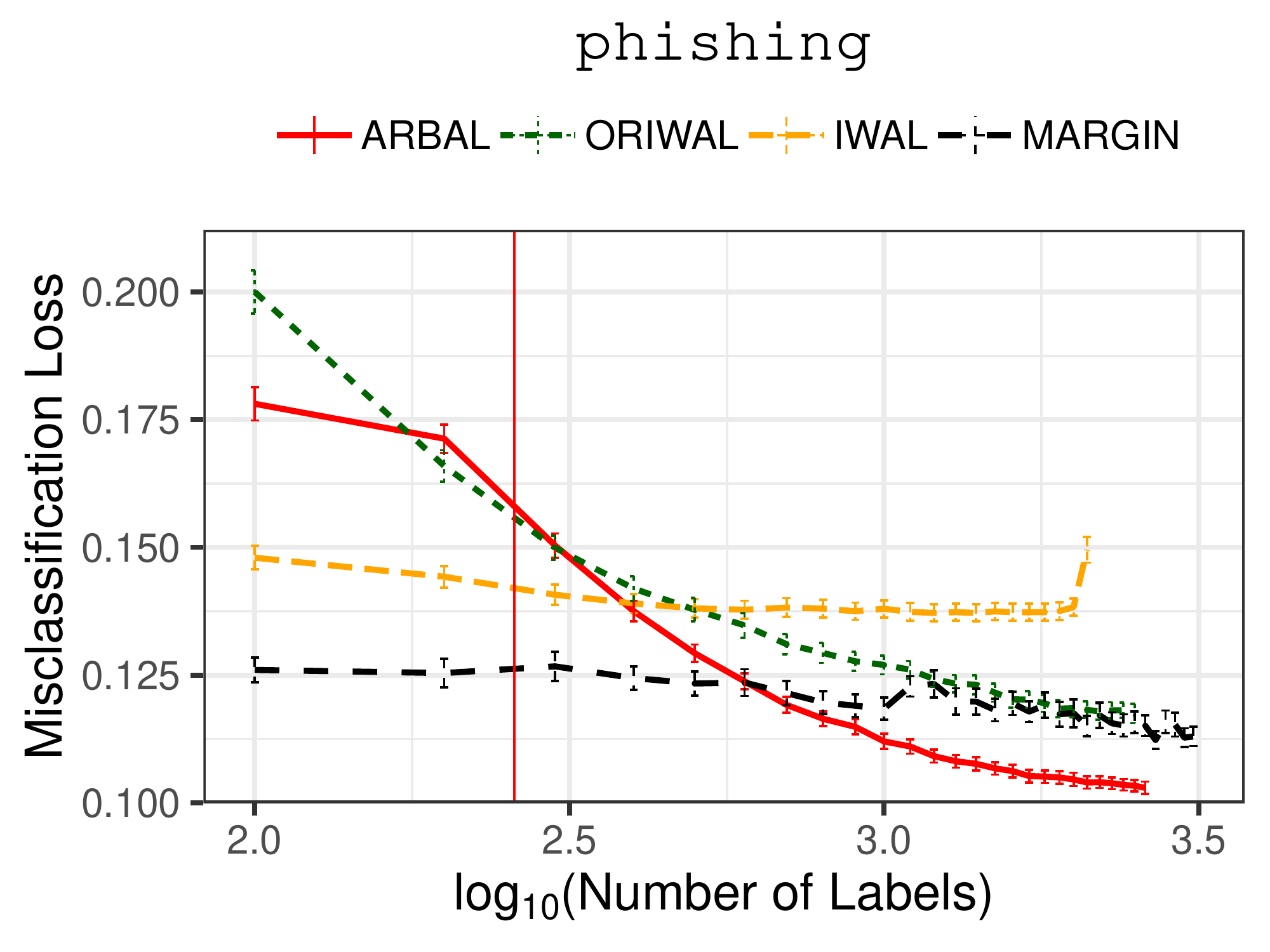}}
\subfigure{\centering\includegraphics[width=0.3\textwidth,,trim= 5 10 10 5,clip=true]{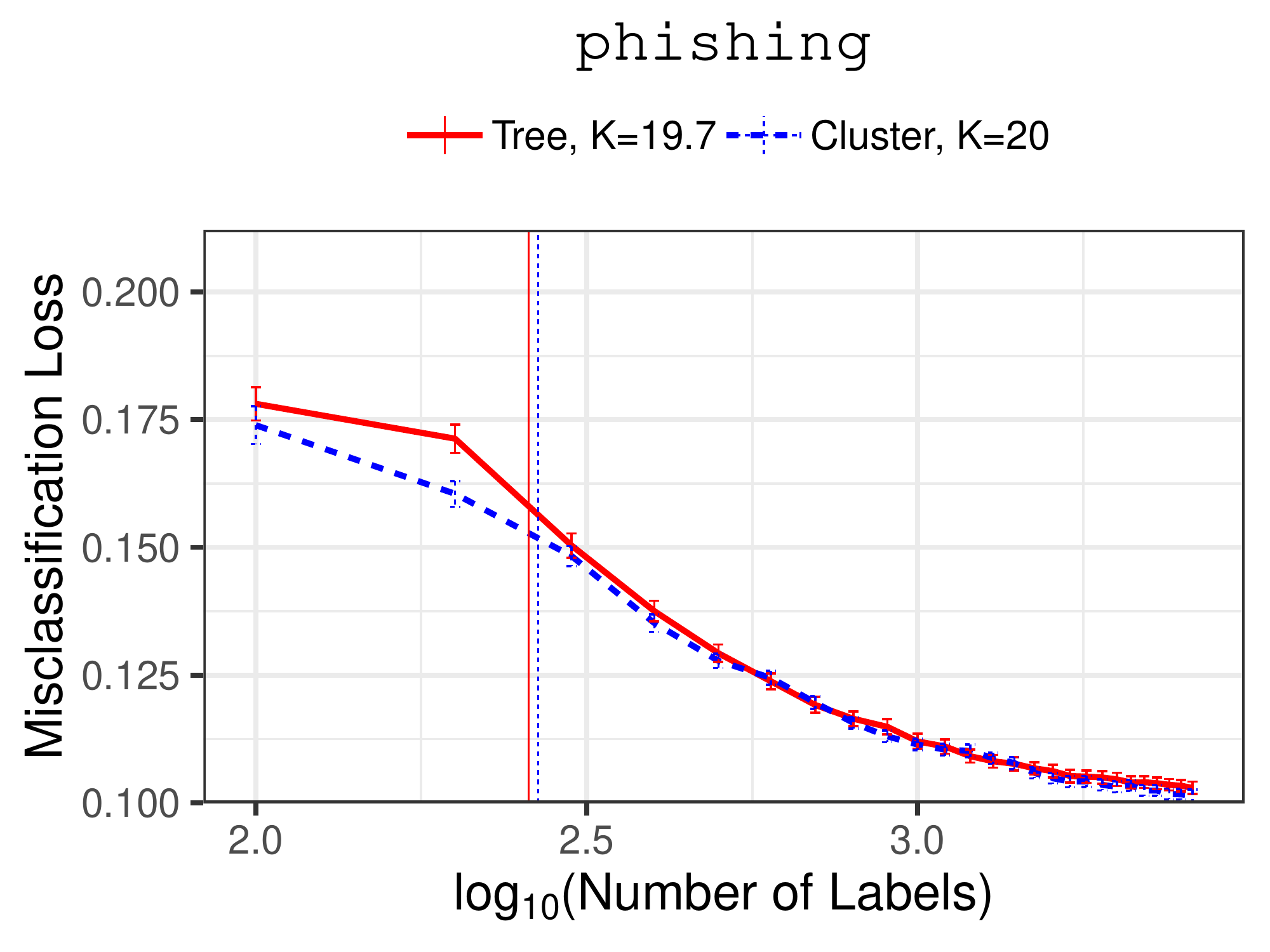}}\\
\subfigure{\centering\includegraphics[width=0.3\textwidth,,trim= 5 10 10 5,clip=true]{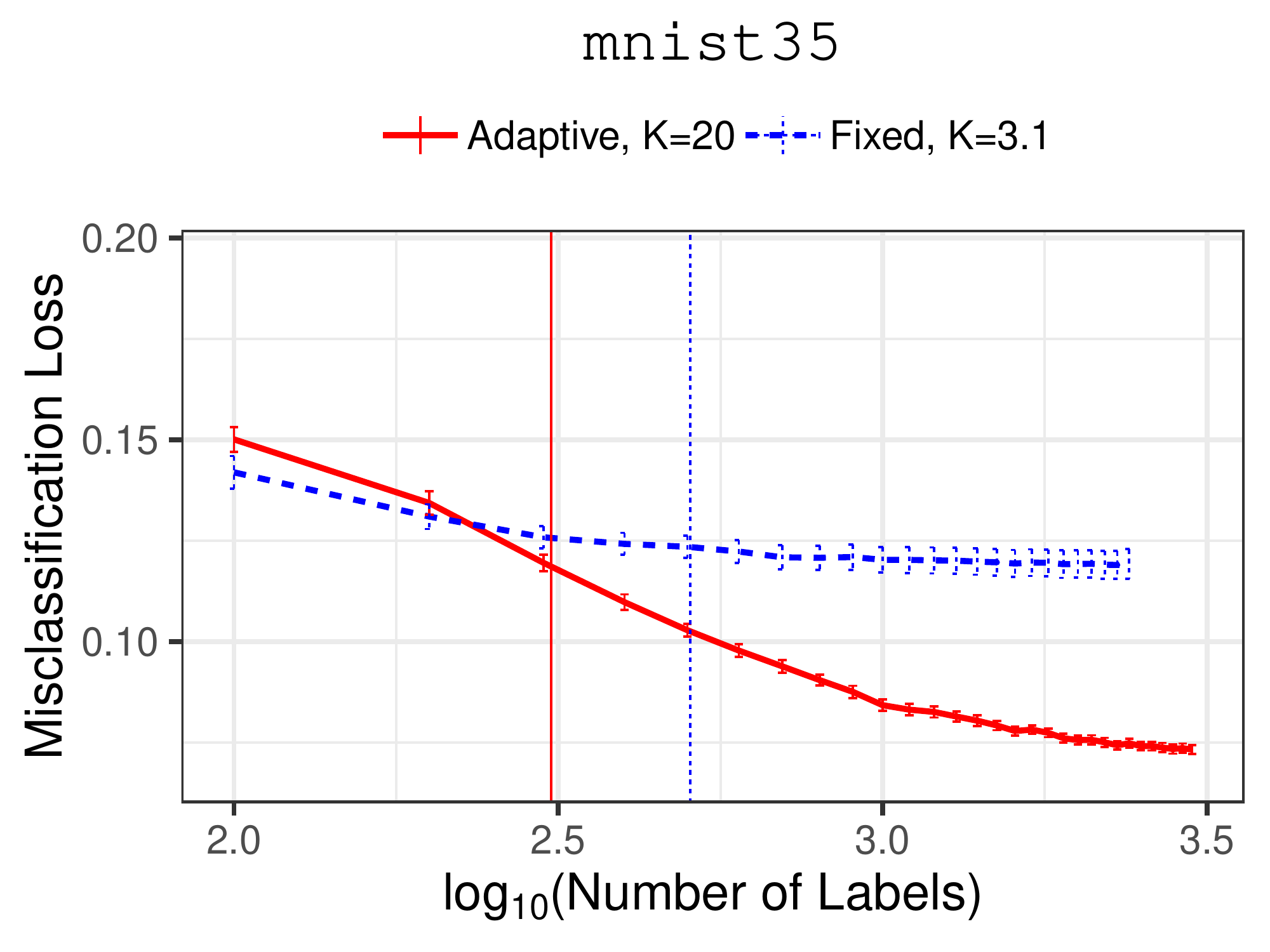}}
\subfigure{\centering\includegraphics[width=0.3\textwidth,,trim= 5 10 10 5,clip=true]{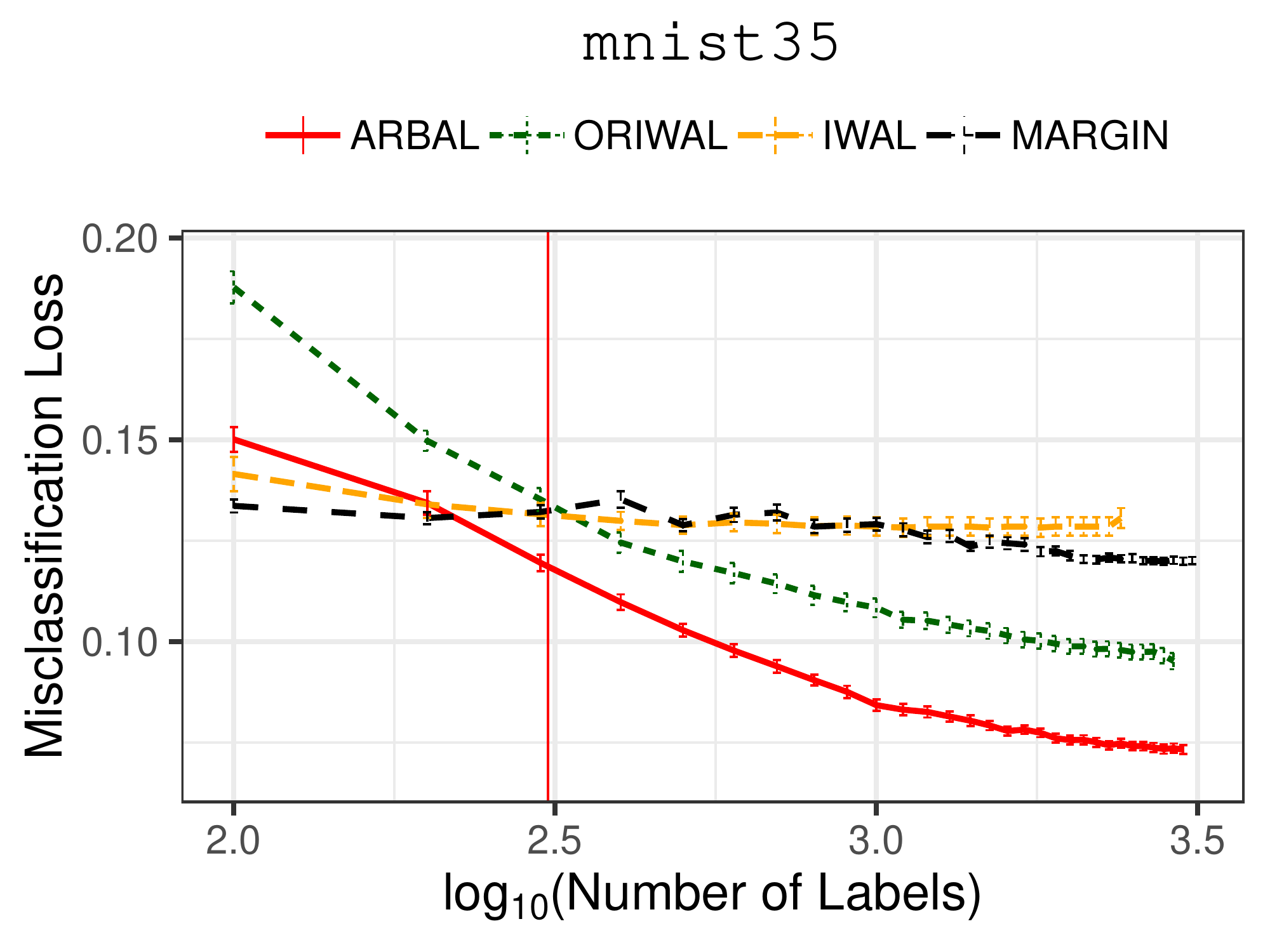}}
\subfigure{\centering\includegraphics[width=0.3\textwidth,,trim= 5 10 10 5,clip=true]{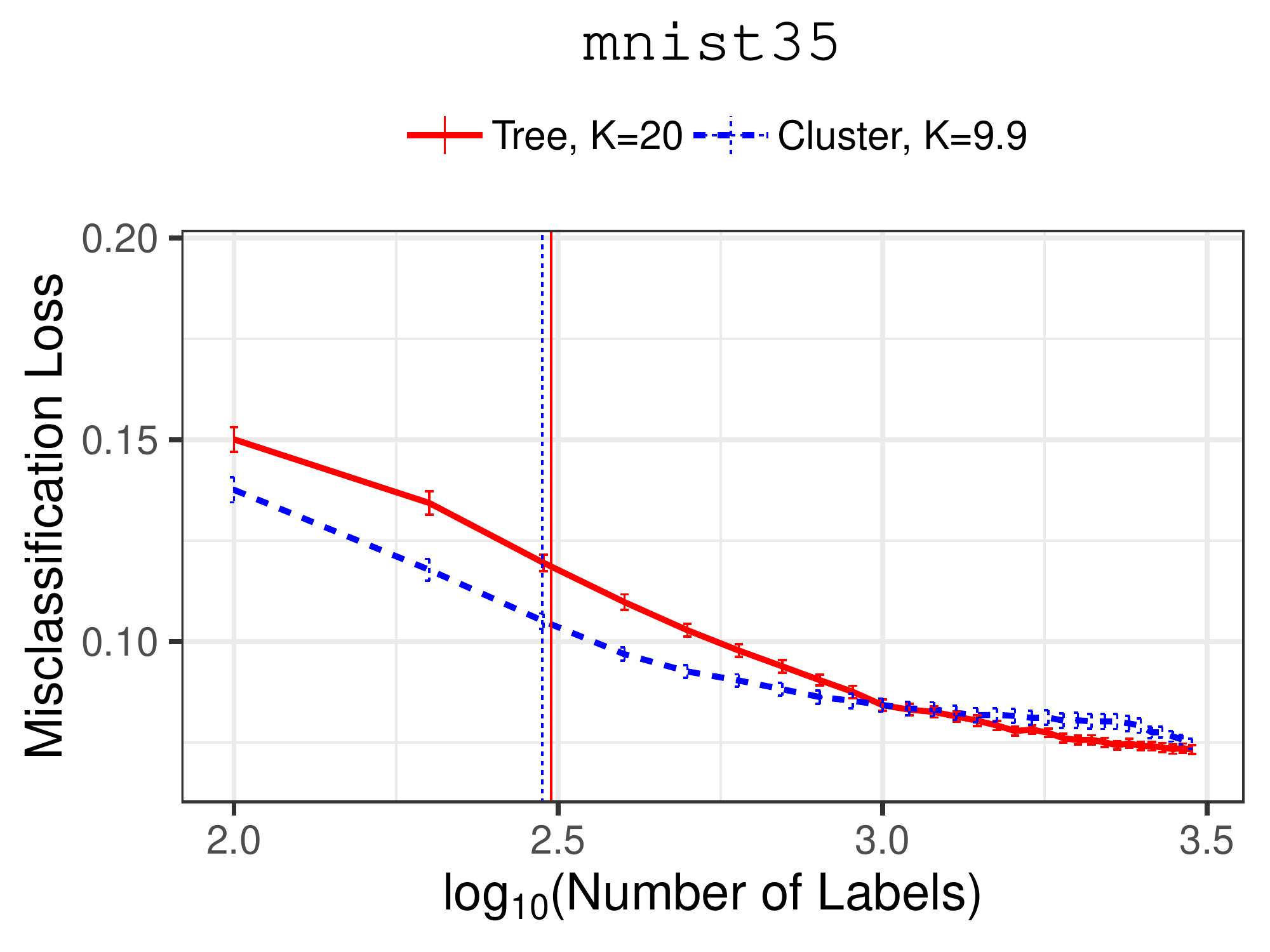}}\\
\subfigure{\centering\includegraphics[width=0.3\textwidth,,trim= 5 10 10 5,clip=true]{figures/compare_gamma/gamma_800_20_egg}}
\subfigure{\centering\includegraphics[width=0.3\textwidth,,trim= 5 10 10 5,clip=true]{figures/compare_baselines/loss_misclass_vs_labels_tau800_k20_egg}}
\subfigure{\centering\includegraphics[width=0.3\textwidth,,trim= 5 10 10 5,clip=true]{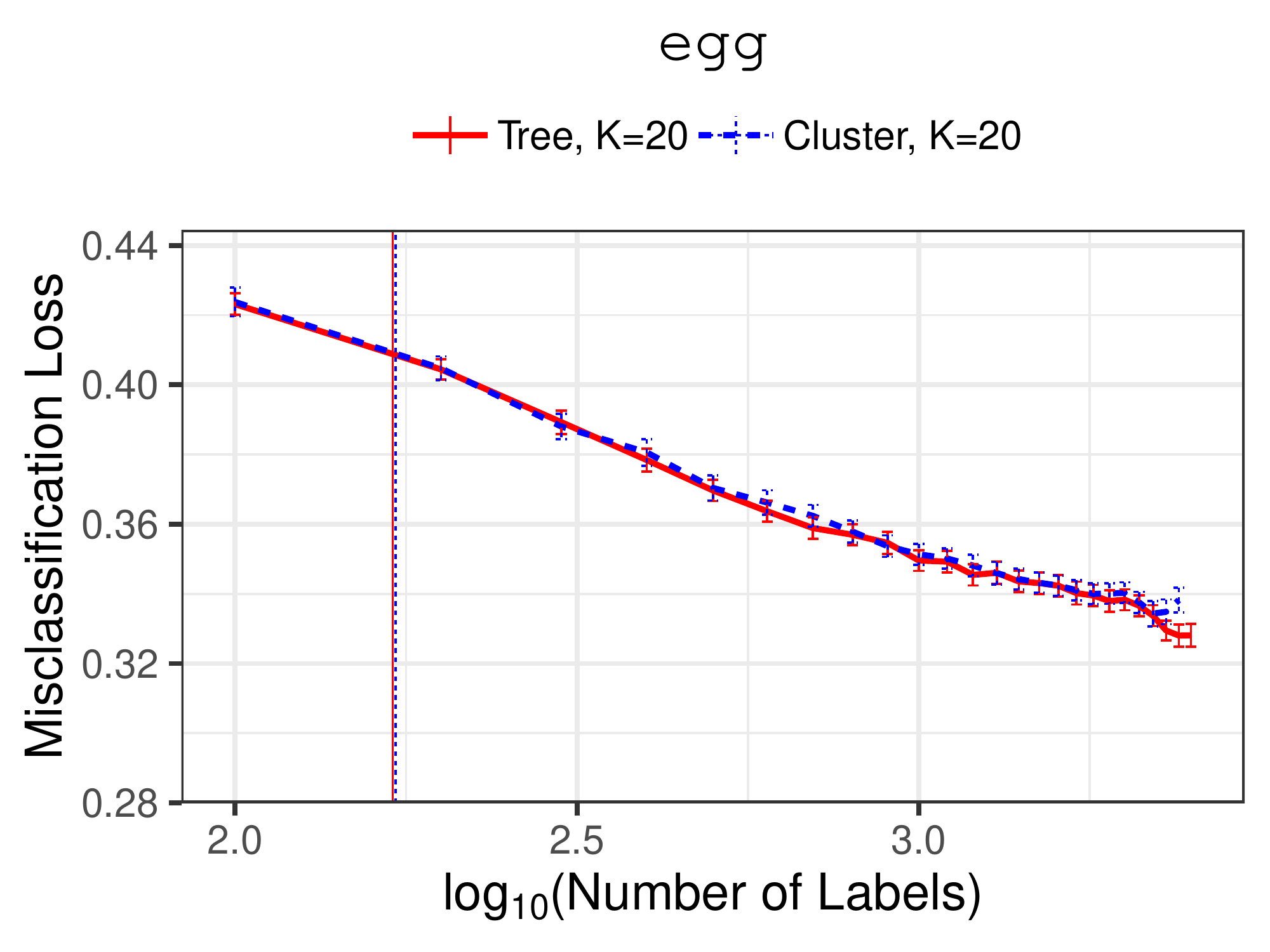}}\\
\subfigure{\centering\includegraphics[width=0.3\textwidth,,trim= 5 10 10 5,clip=true]{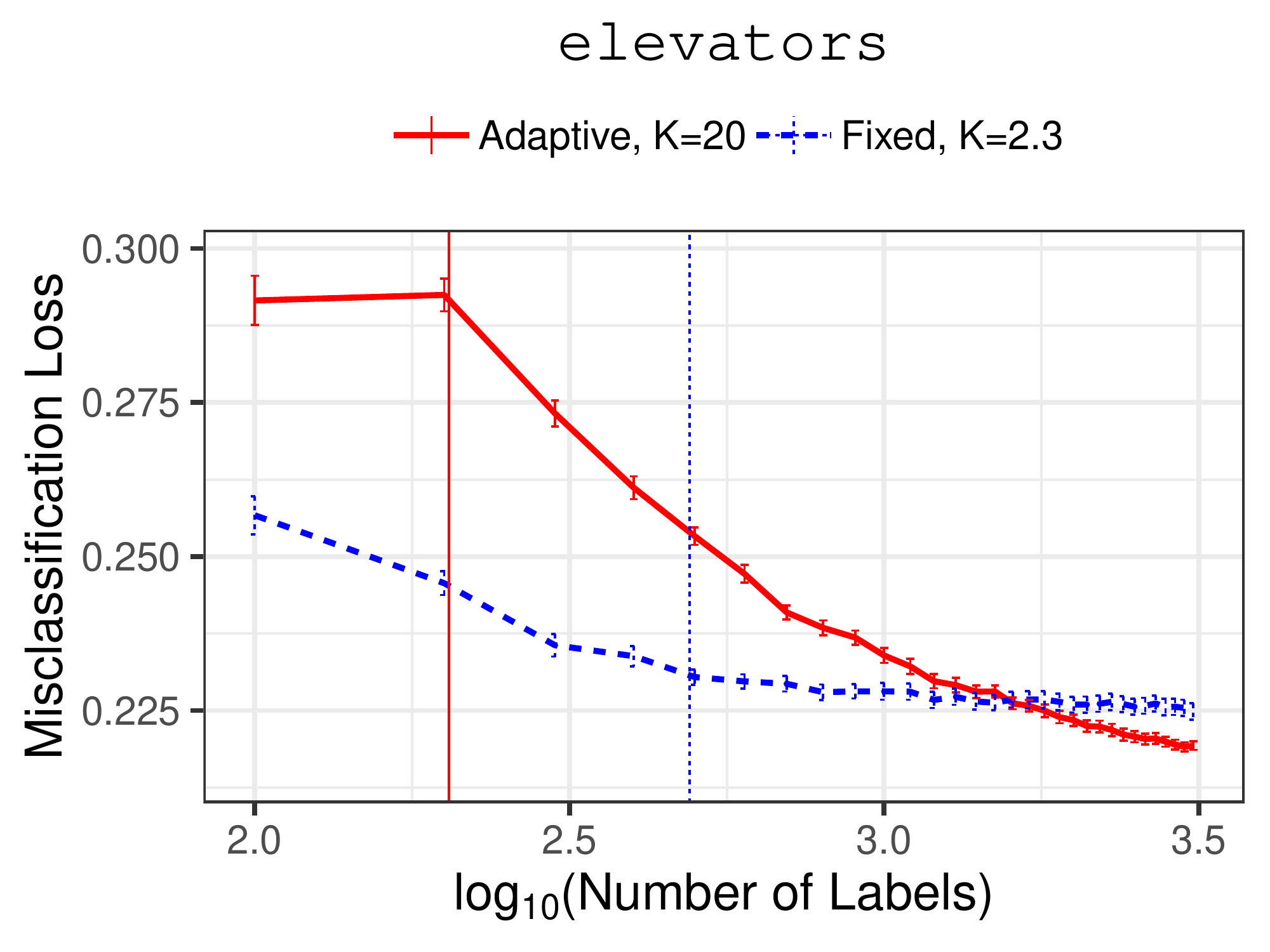}}
\subfigure{\centering\includegraphics[width=0.3\textwidth,,trim= 5 10 10 5,clip=true]{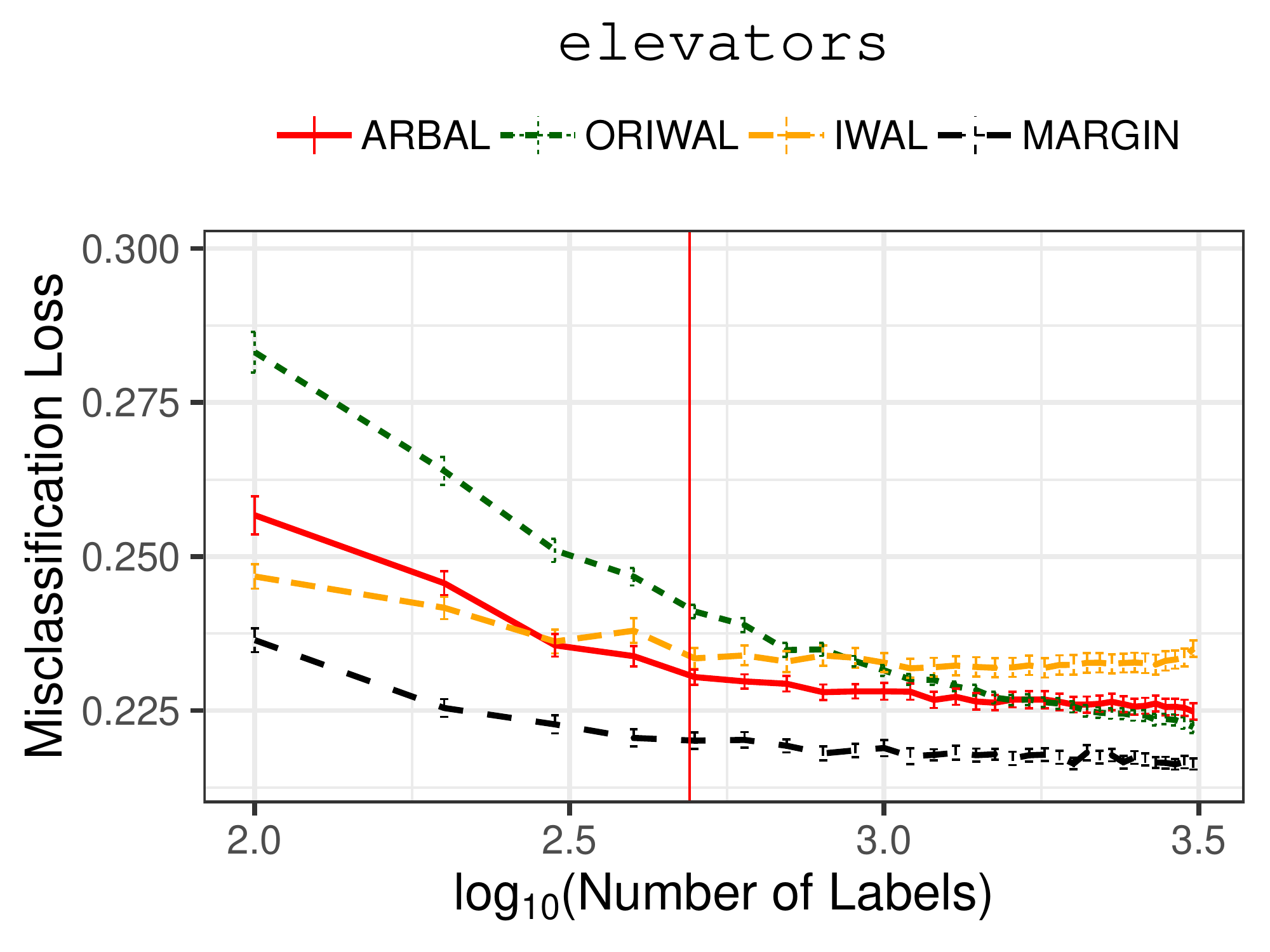}}
\subfigure{\centering\includegraphics[width=0.3\textwidth,,trim= 5 10 10 5,clip=true]{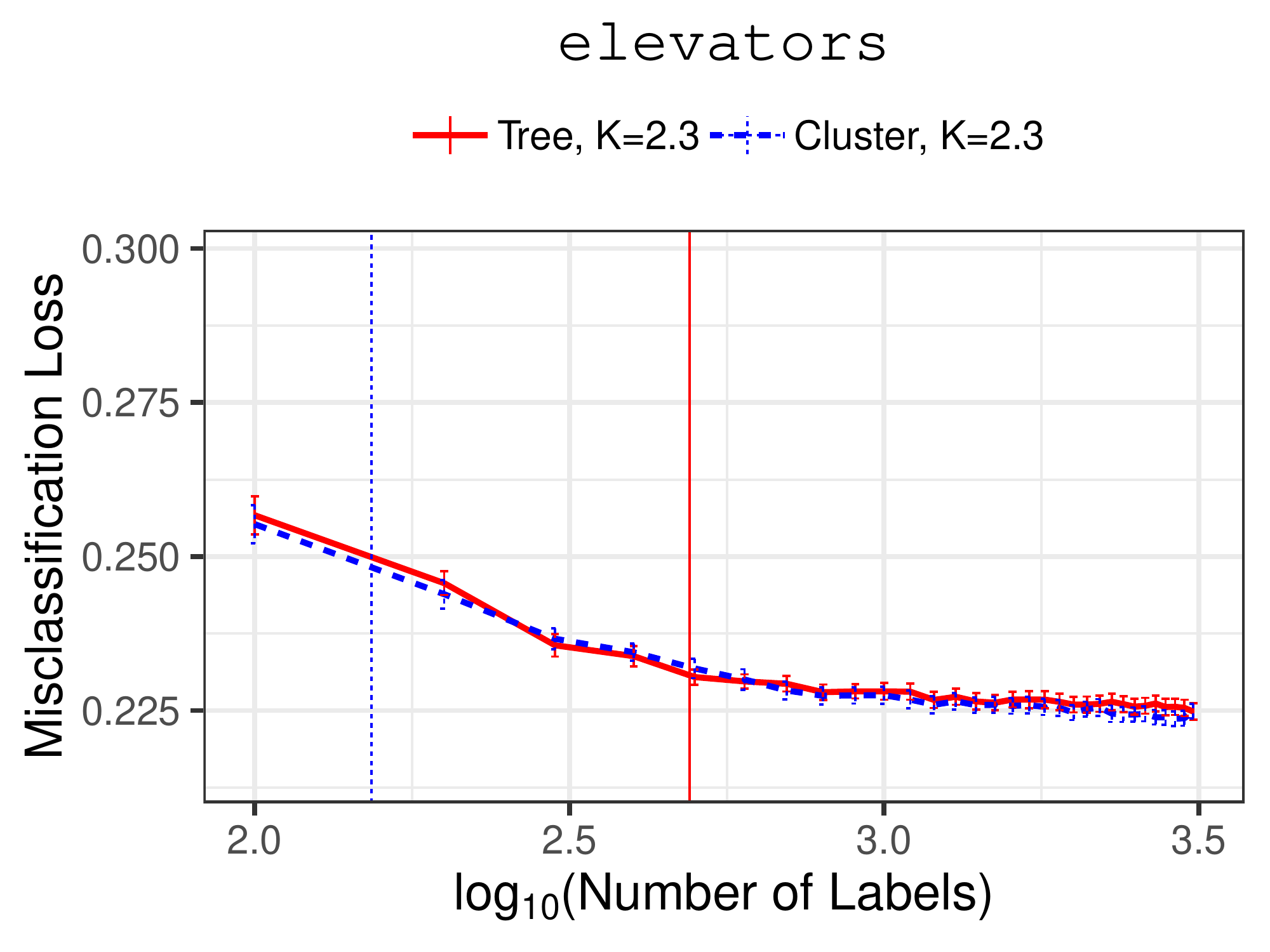}}\\
\end{center}
\caption{Misclassification loss on hold out test data versus number of labels requested ($\log_{10}$ scale). 
Left: \arbal\ with fixed and adaptive threshold $\gamma$.
Middle: \arbal, \riwal, \iwal, and \margin.
Right: \arbal\ with different partitioning methods: binary tree and hierarchical clustering.
For $\kappa=20$ and $\tau=800$,
dataset \texttt{\small{jm1}}, \texttt{\small{phishing}}, \texttt{\small{mnist35}}, \texttt{\small{egg}}, \texttt{\small{elevators}}.
For left and right plots, we give the average number of resulting regions $K$ in the legend. The vertical lines indicate
when \arbal\ transits from the first to the second phase.}
\label{fig:expmis_2}
\vskip -0.2in
\end{figure*}
\begin{figure*}[ht]
\begin{center}
\subfigure{\centering\includegraphics[width=0.3\textwidth,,trim= 5 10 10 5,clip=true]{figures/compare_gamma/gamma_800_20_magic04}}
\subfigure{\centering\includegraphics[width=0.3\textwidth,,trim= 5 10 10 5,clip=true]{figures/compare_baselines/loss_misclass_vs_labels_tau800_k20_magic04}}
\subfigure{\centering\includegraphics[width=0.3\textwidth,,trim= 5 10 10 5,clip=true]{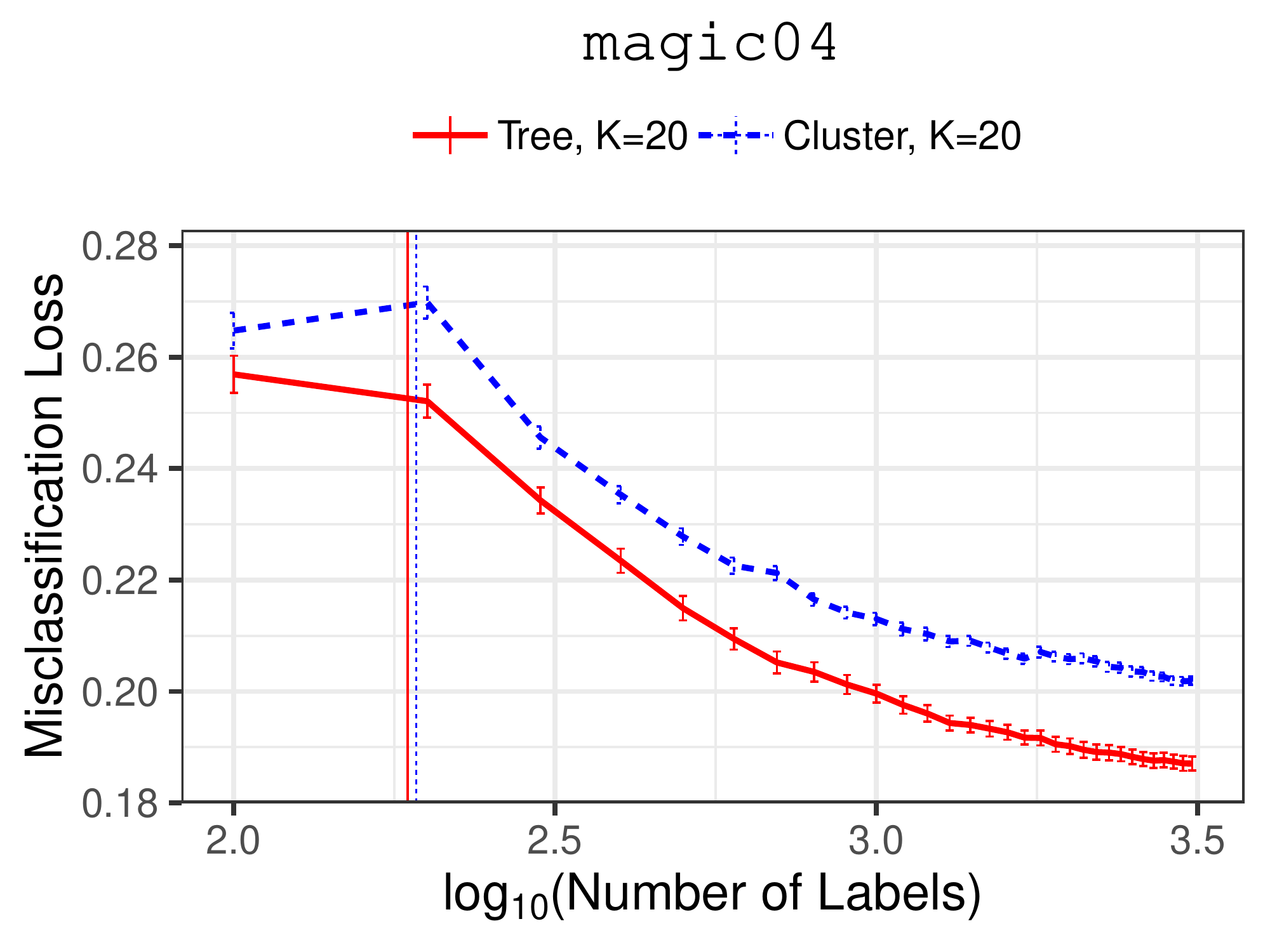}}\\
\subfigure{\centering\includegraphics[width=0.3\textwidth,,trim= 5 10 10 5,clip=true]{figures/compare_gamma/gamma_800_20_house16H}}
\subfigure{\centering\includegraphics[width=0.3\textwidth,,trim= 5 10 10 5,clip=true]{figures/compare_baselines/loss_misclass_vs_labels_tau800_k20_house16H}}
\subfigure{\centering\includegraphics[width=0.3\textwidth,,trim= 5 10 10 5,clip=true]{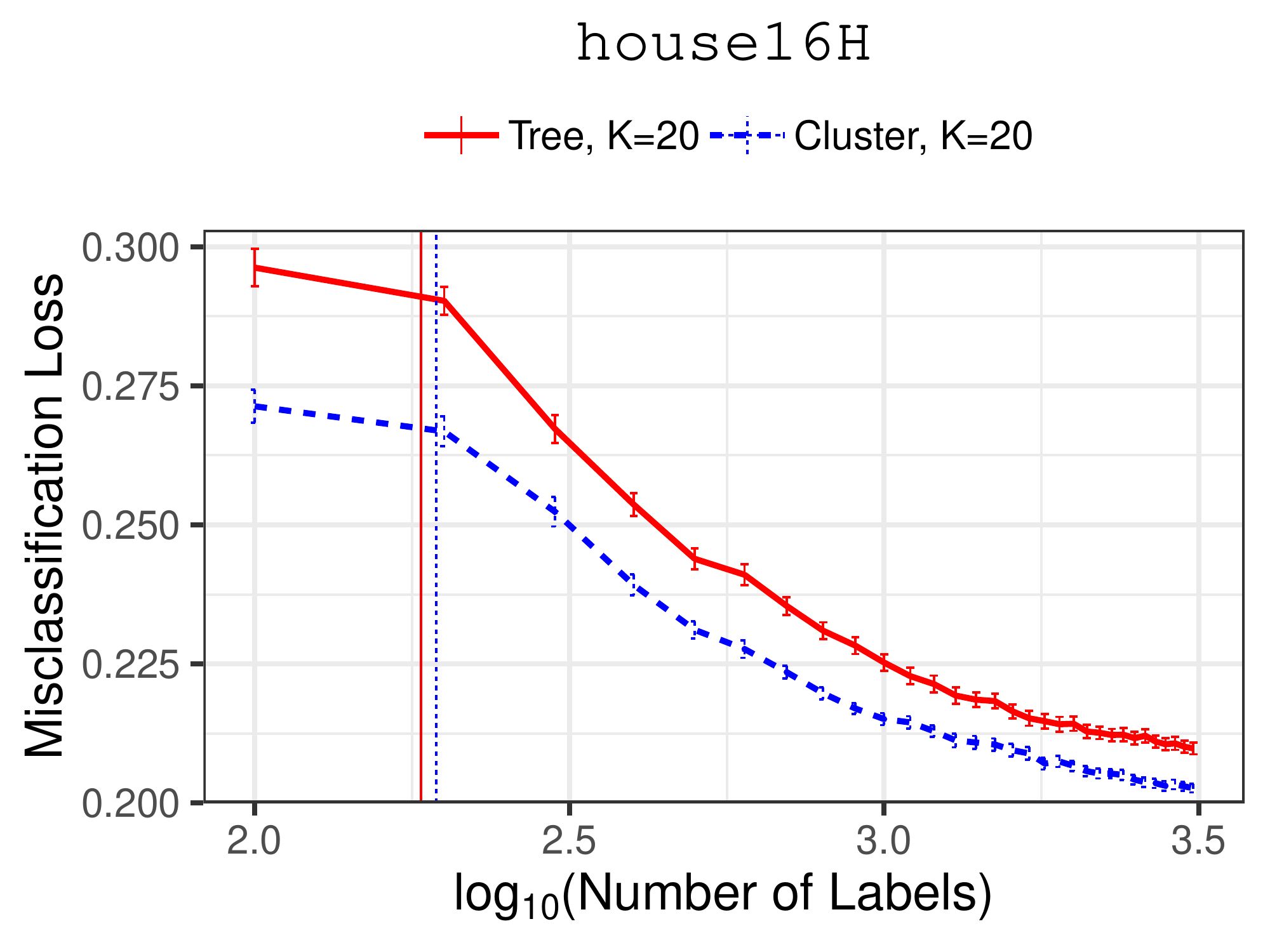}}\\
\subfigure{\centering\includegraphics[width=0.3\textwidth,,trim= 5 10 10 5,clip=true]{figures/compare_gamma/gamma_800_20_nomao}}
\subfigure{\centering\includegraphics[width=0.3\textwidth,,trim= 5 10 10 5,clip=true]{figures/compare_baselines/loss_misclass_vs_labels_tau800_k20_nomao}}
\subfigure{\centering\includegraphics[width=0.3\textwidth,,trim= 5 10 10 5,clip=true]{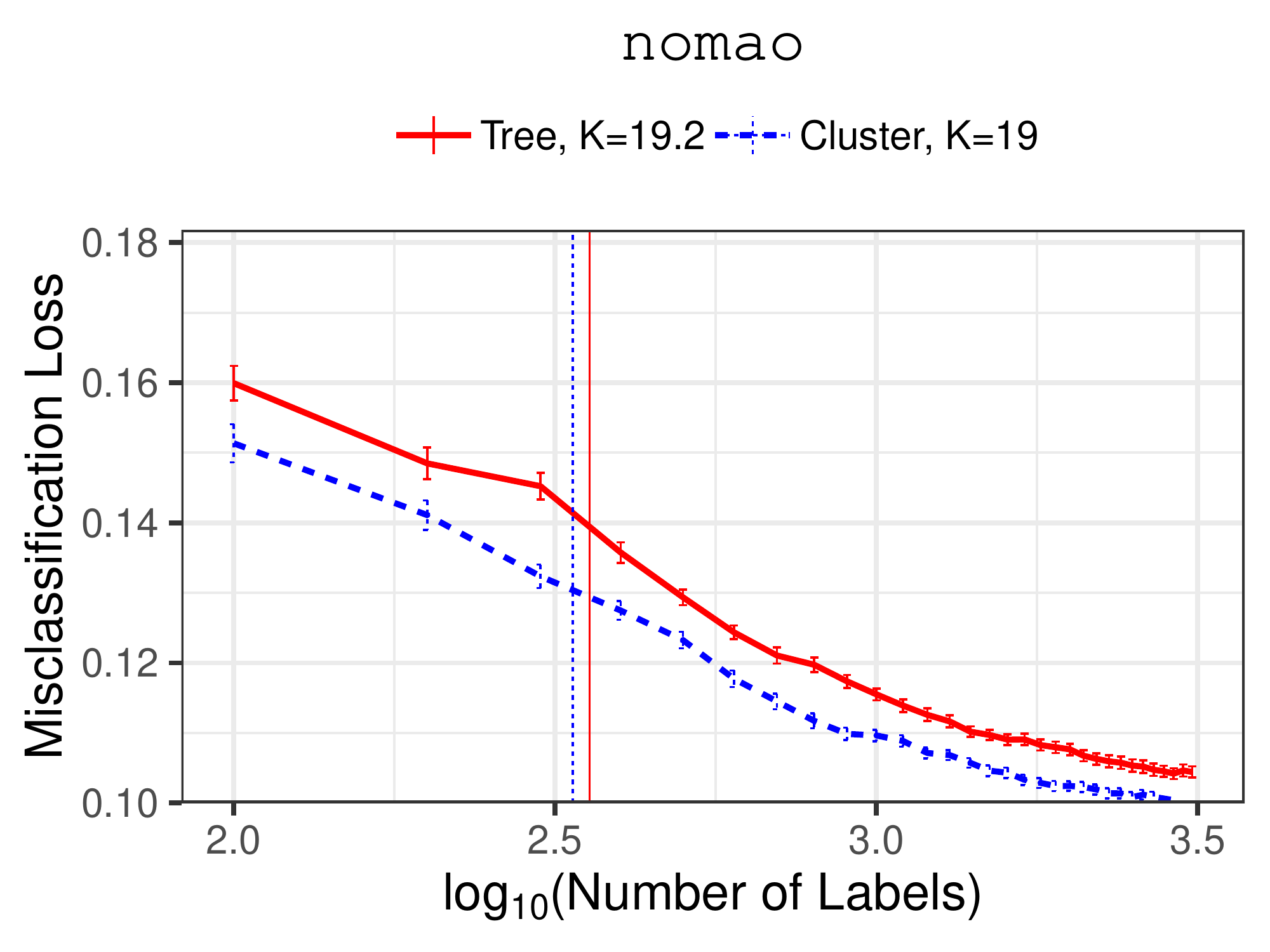}}\\
\subfigure{\centering\includegraphics[width=0.3\textwidth,,trim= 5 10 10 5,clip=true]{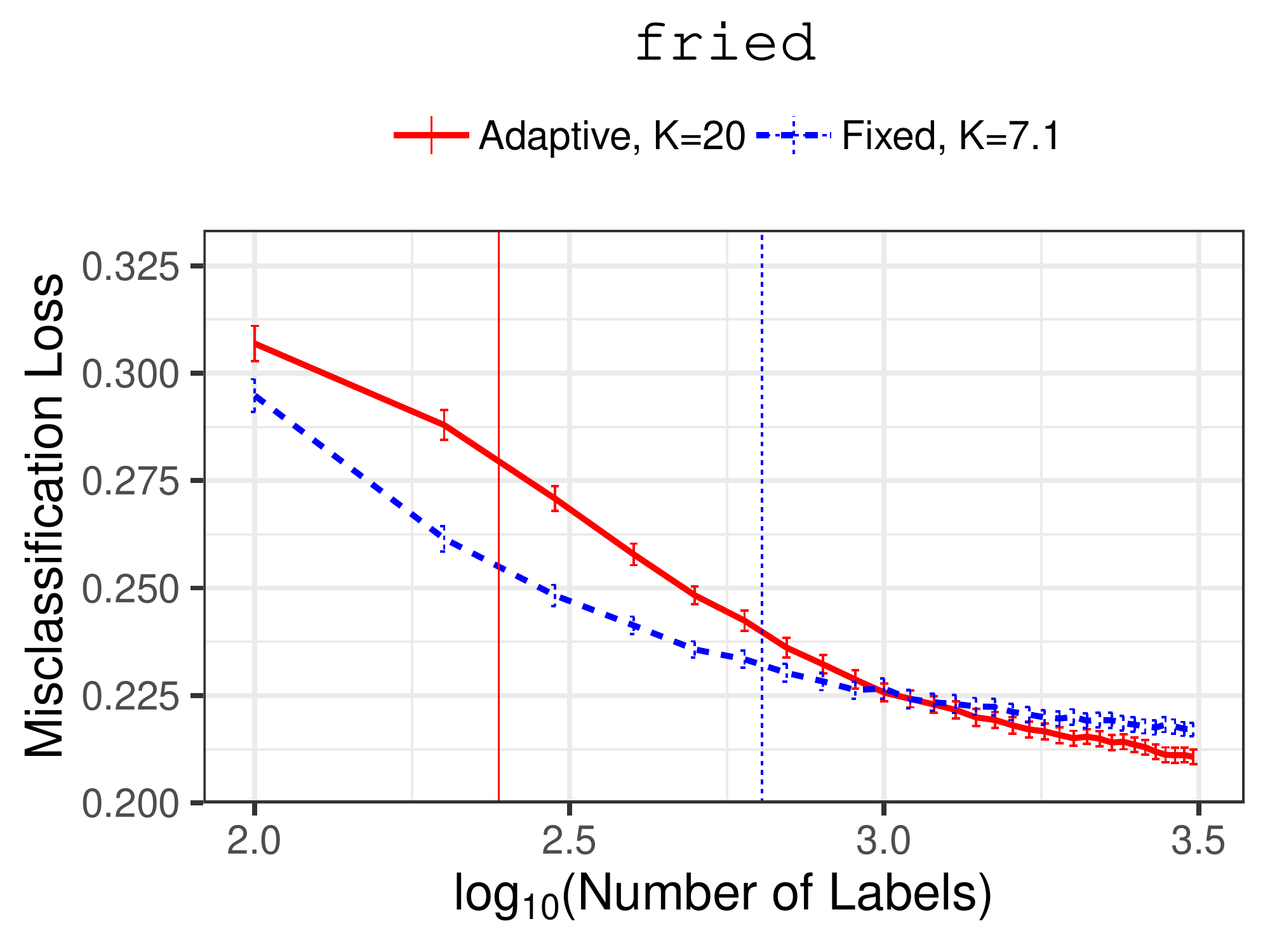}}
\subfigure{\centering\includegraphics[width=0.3\textwidth,,trim= 5 10 10 5,clip=true]{figures/compare_baselines/loss_misclass_vs_labels_tau800_k20_fried}}
\subfigure{\centering\includegraphics[width=0.3\textwidth,,trim= 5 10 10 5,clip=true]{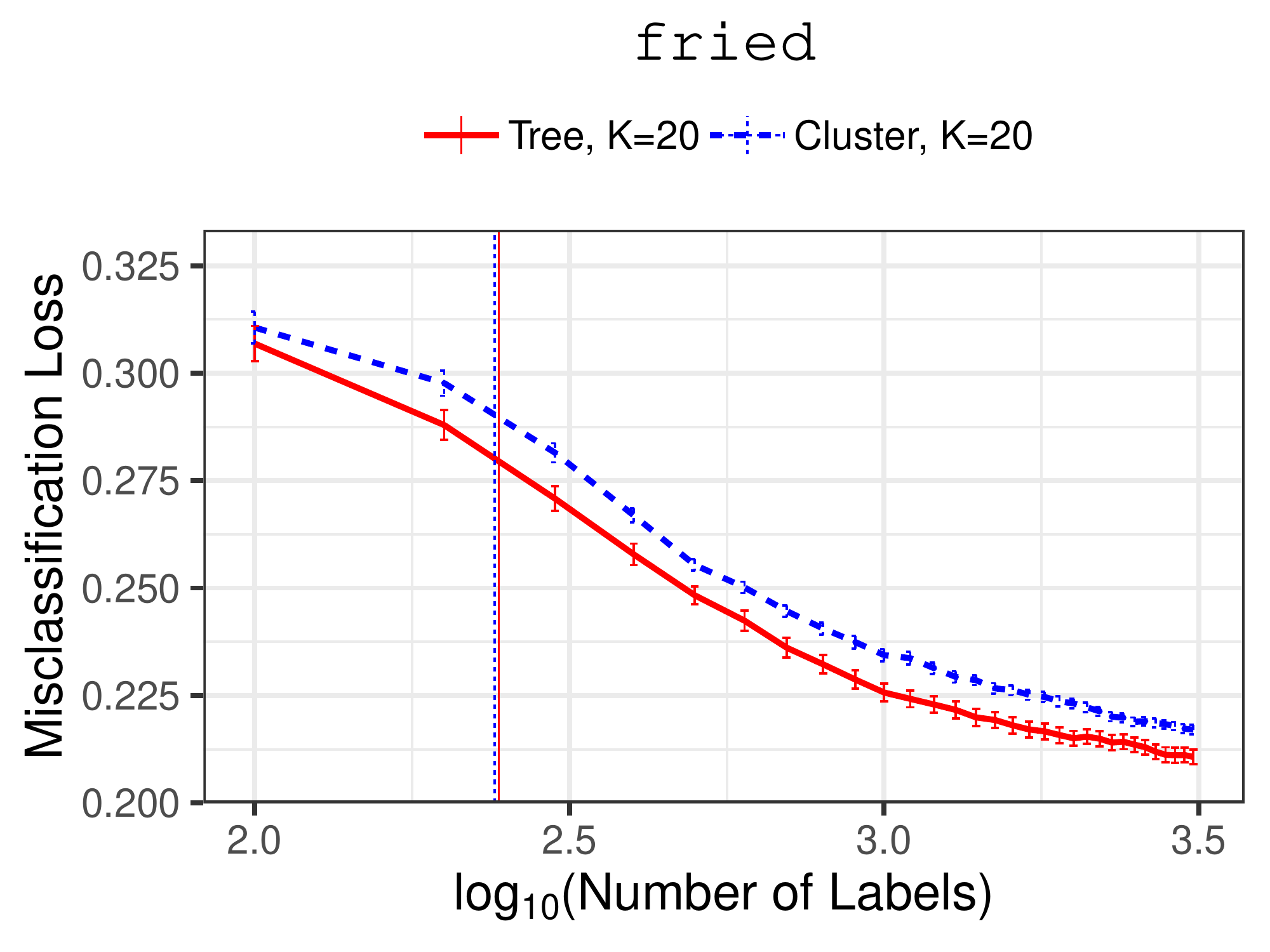}}\\
\subfigure{\centering\includegraphics[width=0.3\textwidth,,trim= 5 10 10 5,clip=true]{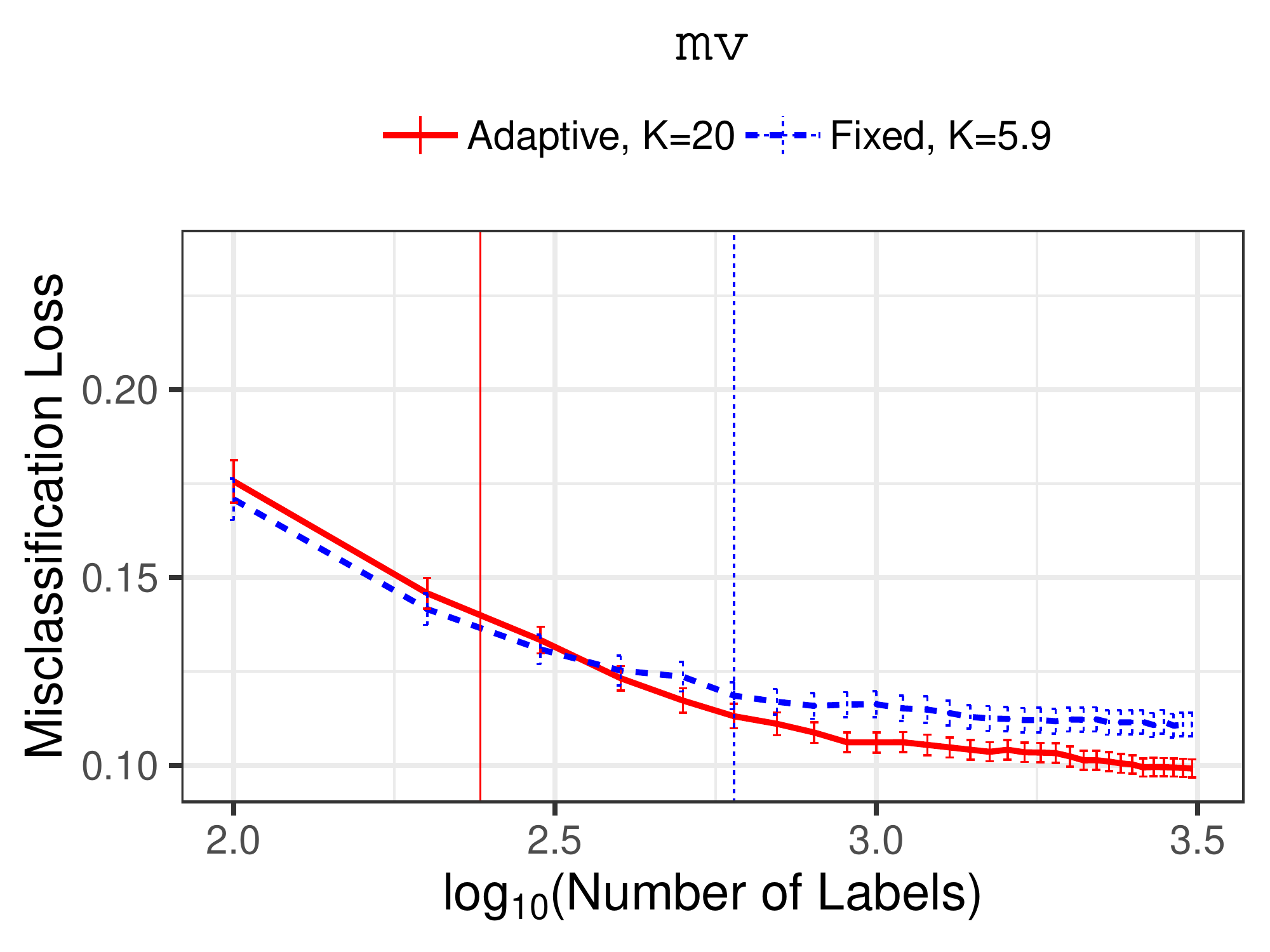}}
\subfigure{\centering\includegraphics[width=0.3\textwidth,,trim= 5 10 10 5,clip=true]{figures/compare_baselines/loss_misclass_vs_labels_tau800_k20_mv}}
\subfigure{\centering\includegraphics[width=0.3\textwidth,,trim= 5 10 10 5,clip=true]{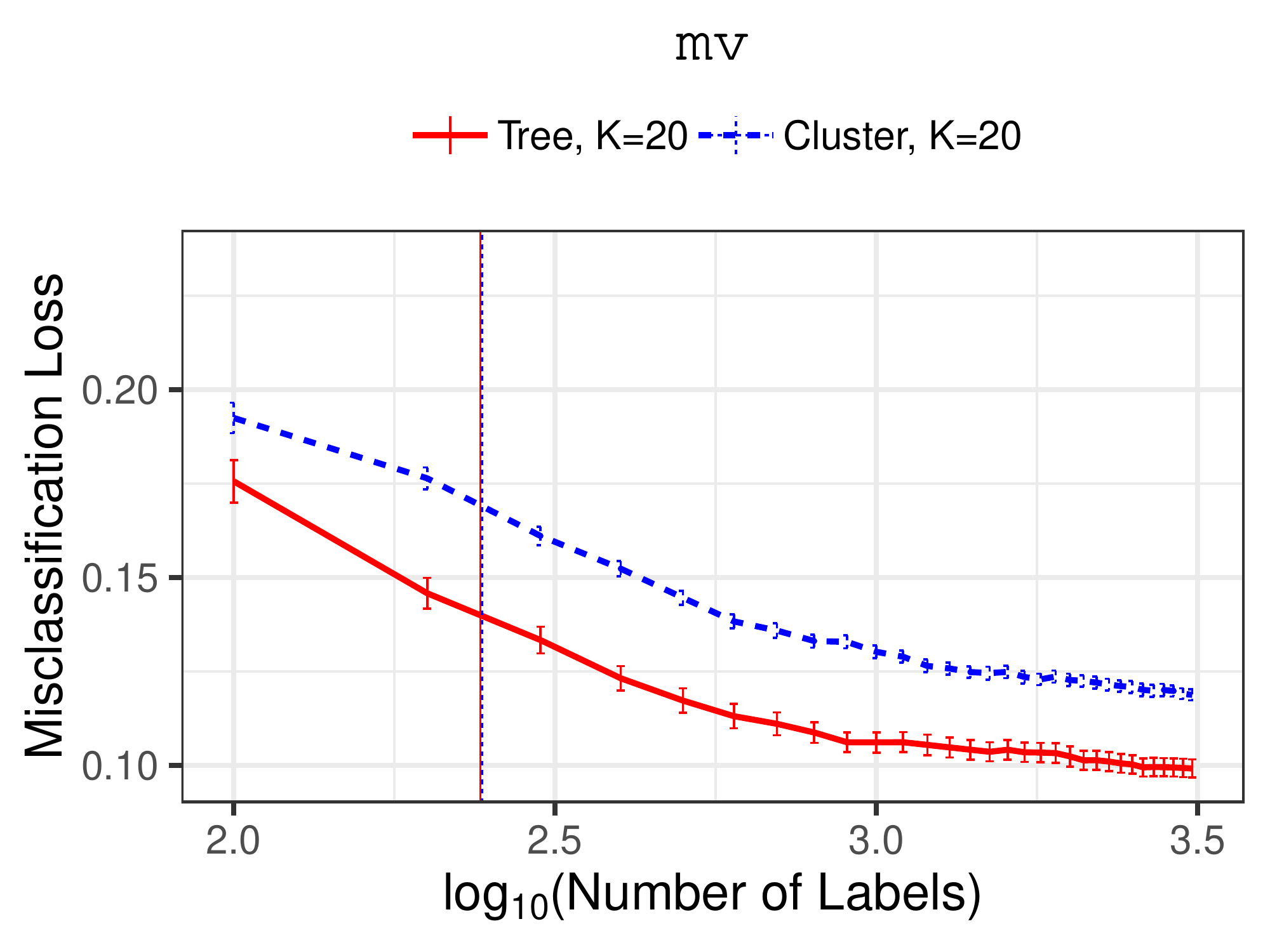}}\\
\end{center}
\caption{Misclassification loss on hold out test data versus number of labels requested ($\log_{10}$ scale). 
Left: \arbal\ with fixed and adaptive threshold $\gamma$.
Middle: \arbal, \riwal, \iwal, and \margin.
Right: \arbal\ with different partitioning methods: binary tree and hierarchical clustering.
For $\kappa=20$ and $\tau=800$,
dataset \texttt{\small{magic04}}, \texttt{\small{house16H}}, \texttt{\small{nomao}}, \texttt{\small{fried}}, \texttt{\small{mv}}.
For left and right plots, we give the average number of resulting regions $K$ in the legend. The vertical lines indicate
when \arbal\ transits from the first to the second phase.}
\label{fig:expmis_3}
\vskip -0.2in
\end{figure*}
\begin{figure*}[ht]
\begin{center}
\subfigure{\centering\includegraphics[width=0.3\textwidth,,trim= 5 10 10 5,clip=true]{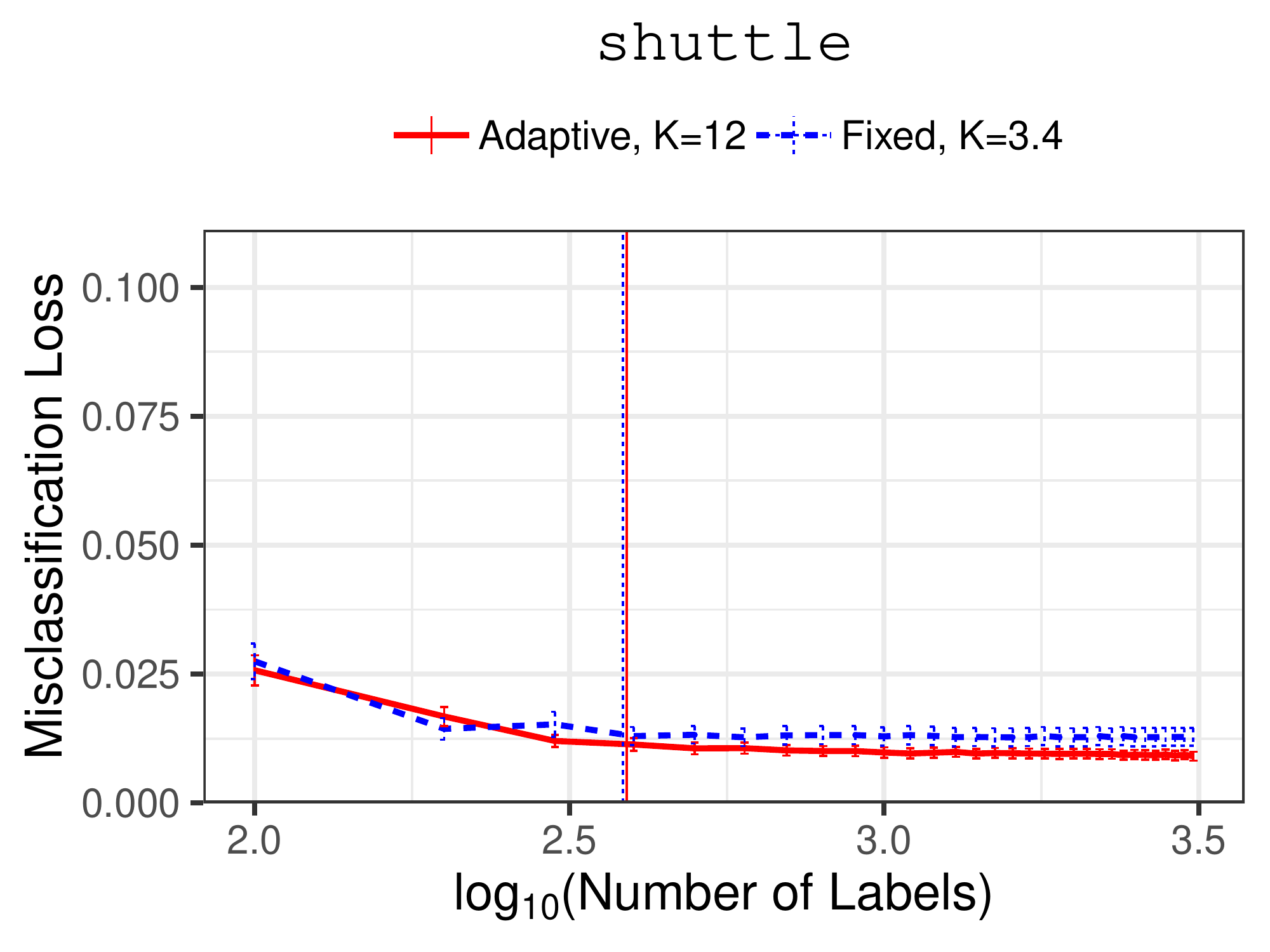}}
\subfigure{\centering\includegraphics[width=0.3\textwidth,,trim= 5 10 10 5,clip=true]{figures/compare_baselines/loss_misclass_vs_labels_tau800_k20_shuttle}}
\subfigure{\centering\includegraphics[width=0.3\textwidth,,trim= 5 10 10 5,clip=true]{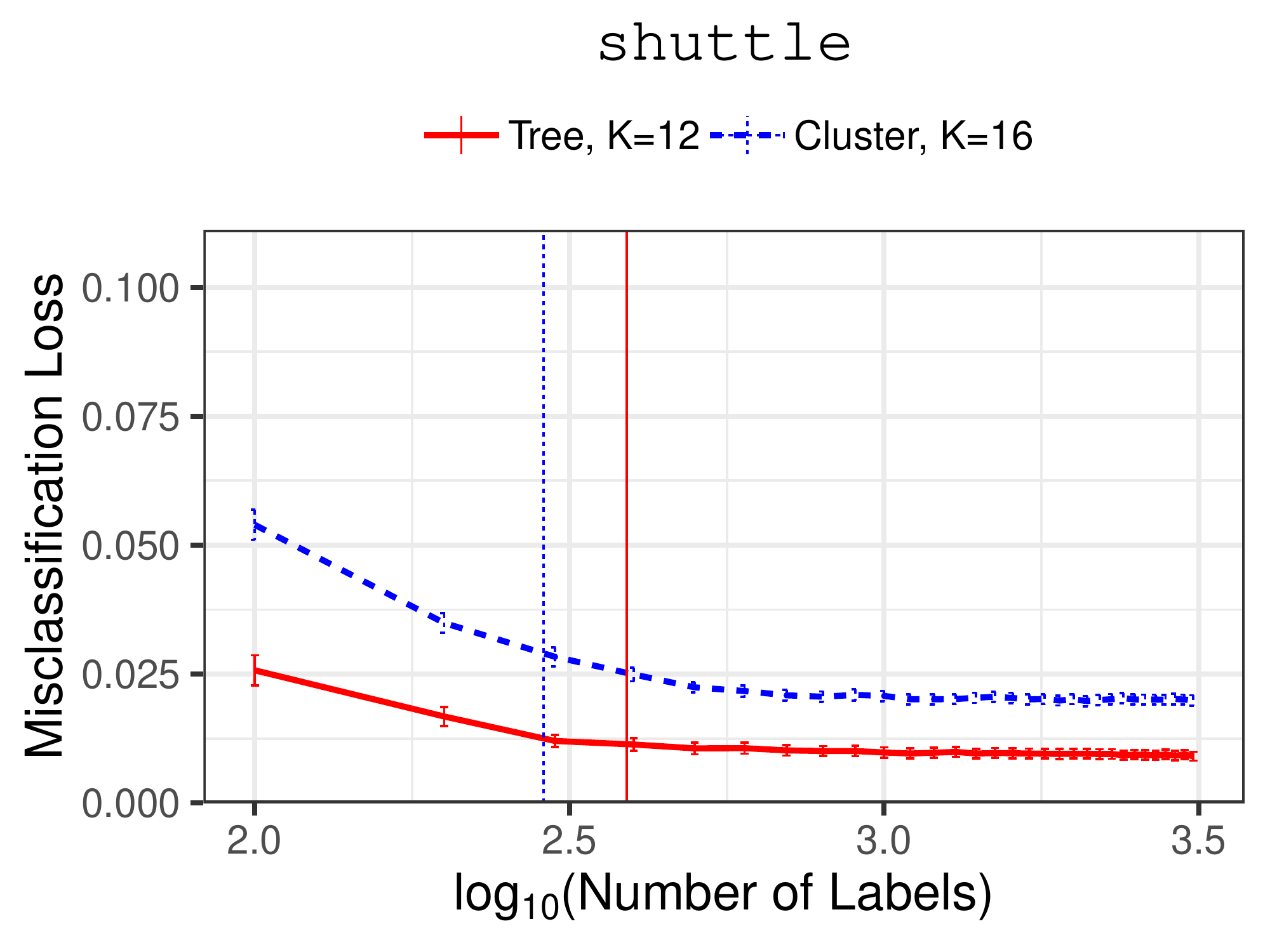}}\\
\subfigure{\centering\includegraphics[width=0.3\textwidth,,trim= 5 10 10 5,clip=true]{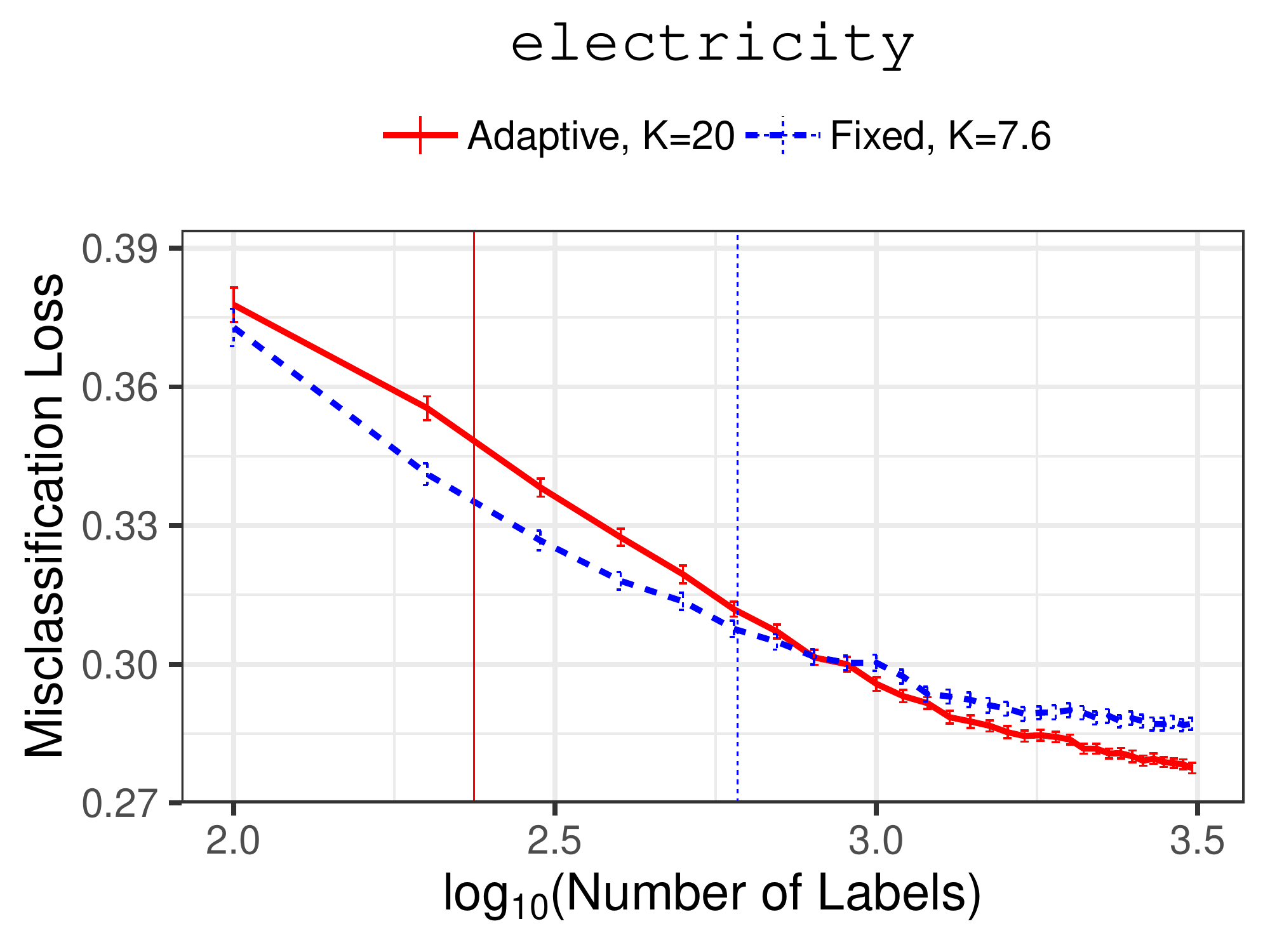}}
\subfigure{\centering\includegraphics[width=0.3\textwidth,,trim= 5 10 10 5,clip=true]{figures/compare_baselines/loss_misclass_vs_labels_tau800_k20_electricity}}
\subfigure{\centering\includegraphics[width=0.3\textwidth,,trim= 5 10 10 5,clip=true]{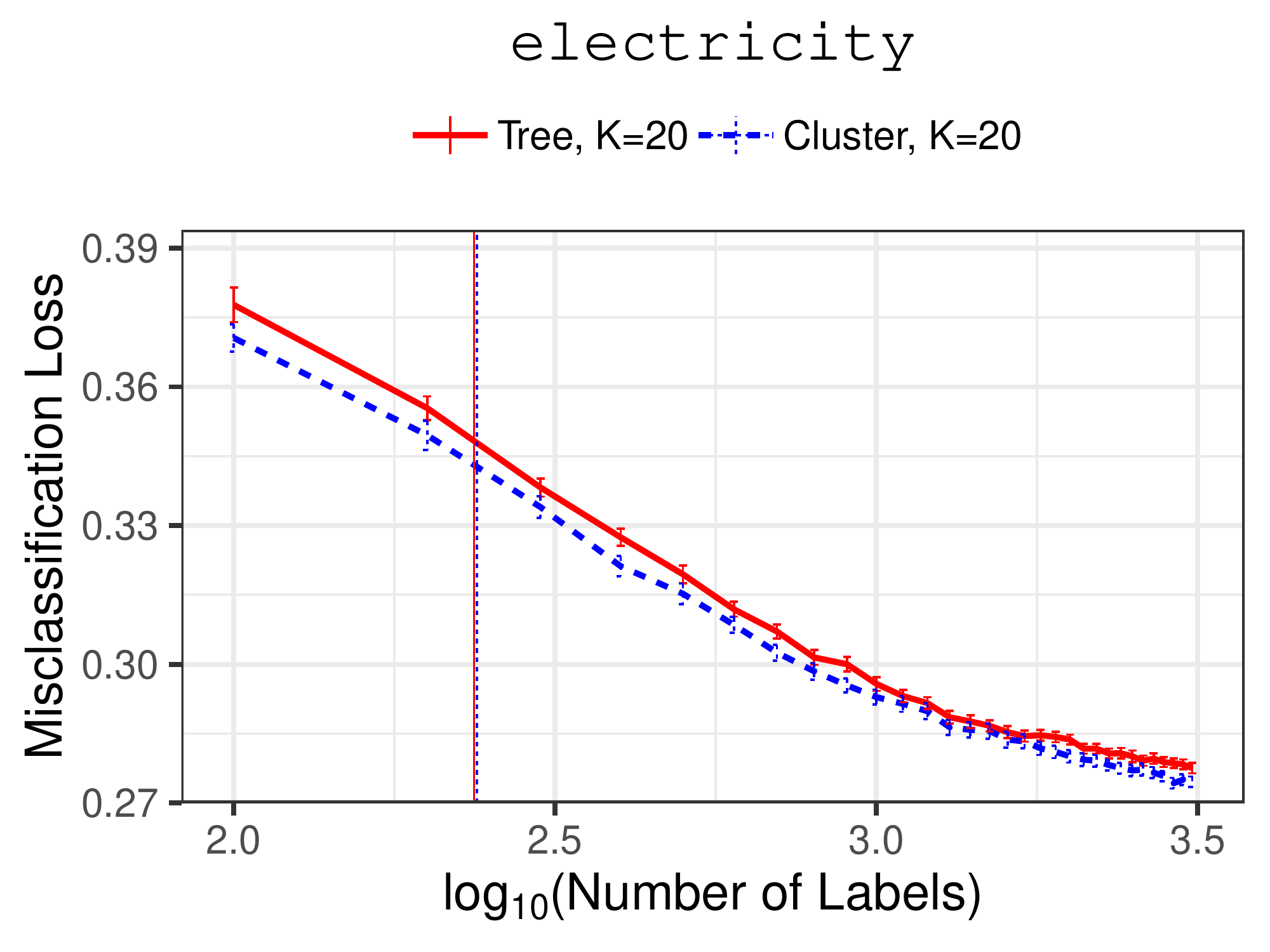}}\\
\subfigure{\centering\includegraphics[width=0.3\textwidth,,trim= 5 10 10 5,clip=true]{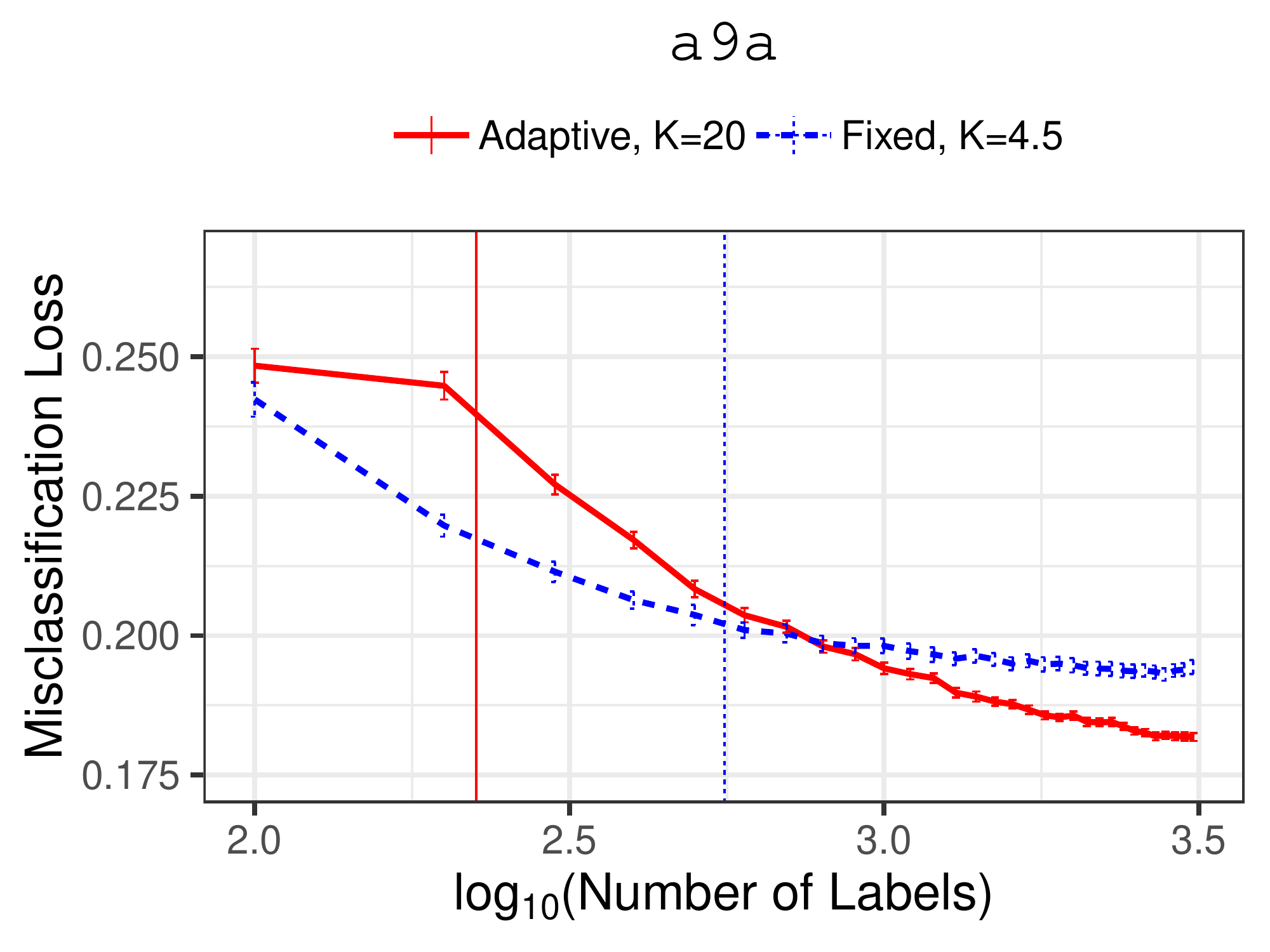}}
\subfigure{\centering\includegraphics[width=0.3\textwidth,,trim= 5 10 10 5,clip=true]{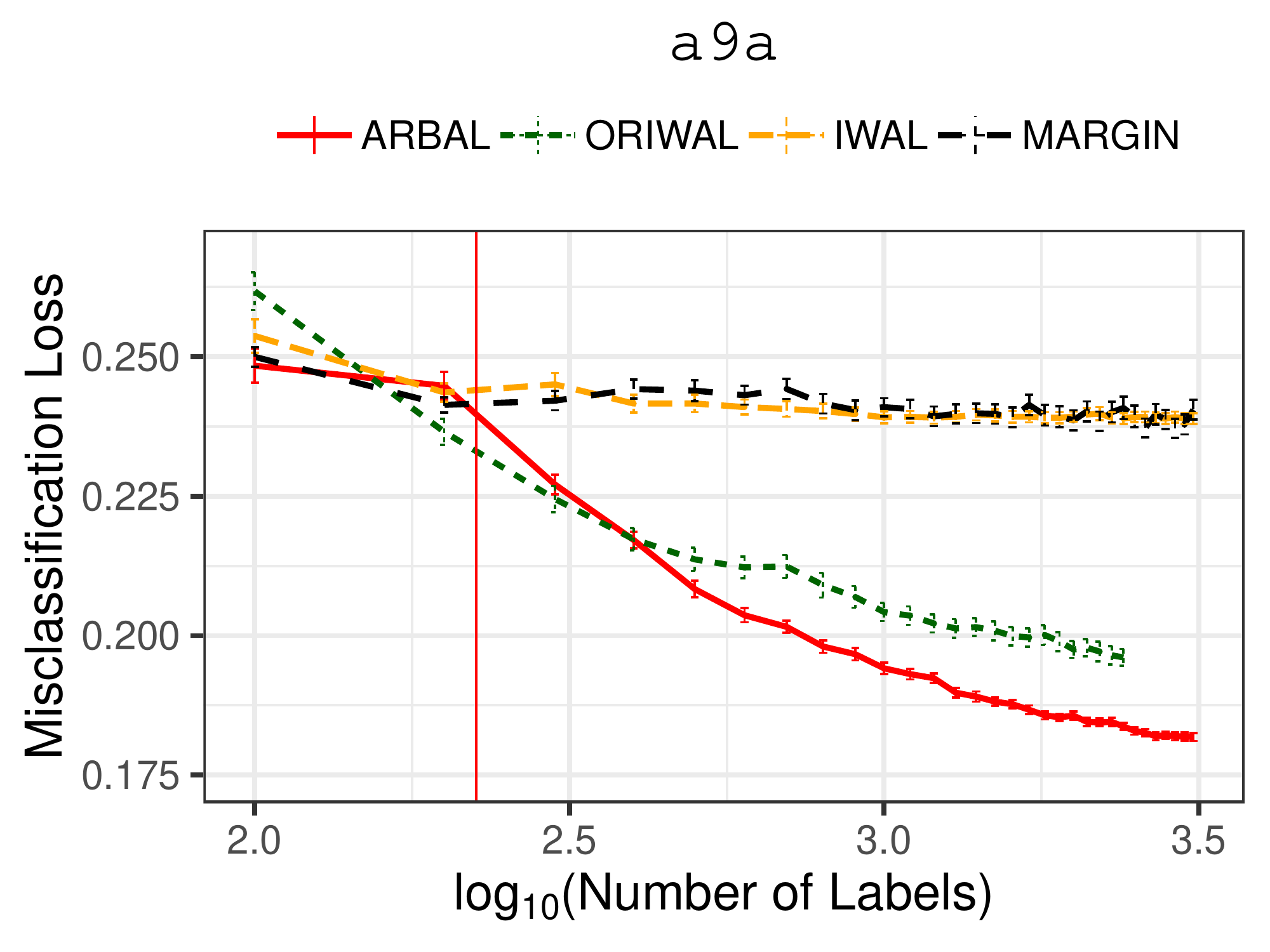}}
\subfigure{\centering\includegraphics[width=0.3\textwidth,,trim= 5 10 10 5,clip=true]{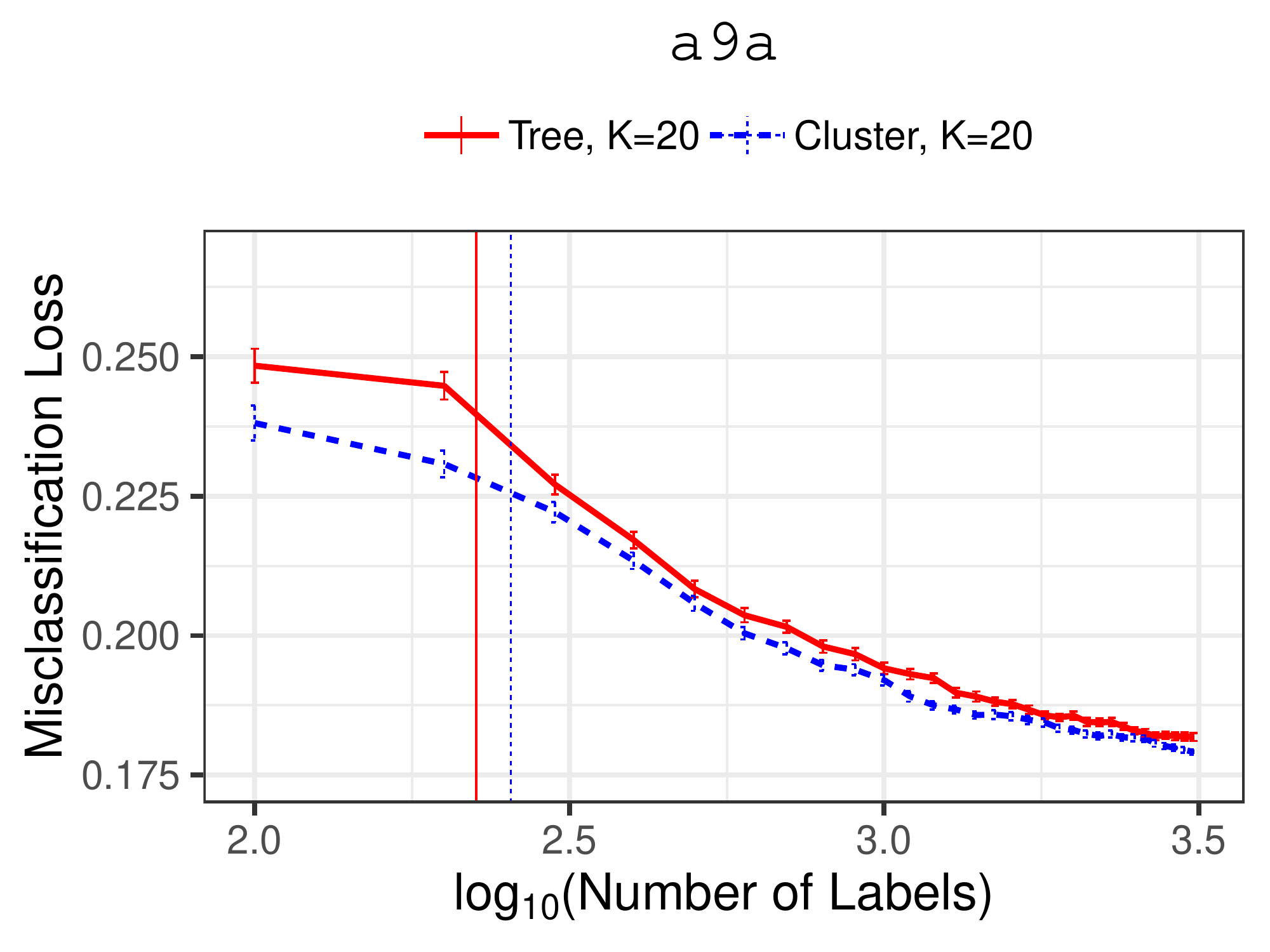}}\\
\subfigure{\centering\includegraphics[width=0.3\textwidth,,trim= 5 10 10 5,clip=true]{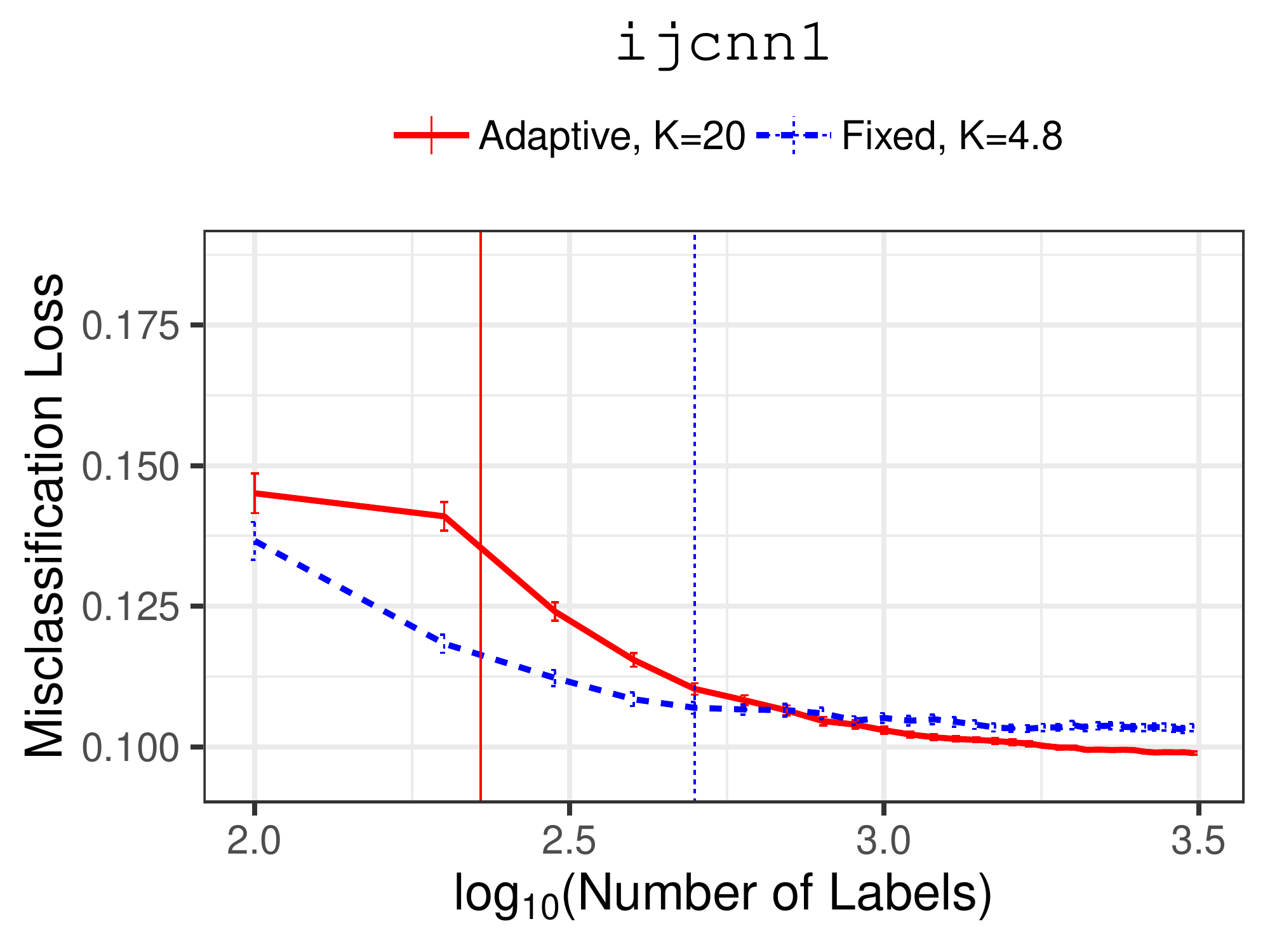}}
\subfigure{\centering\includegraphics[width=0.3\textwidth,,trim= 5 10 10 5,clip=true]{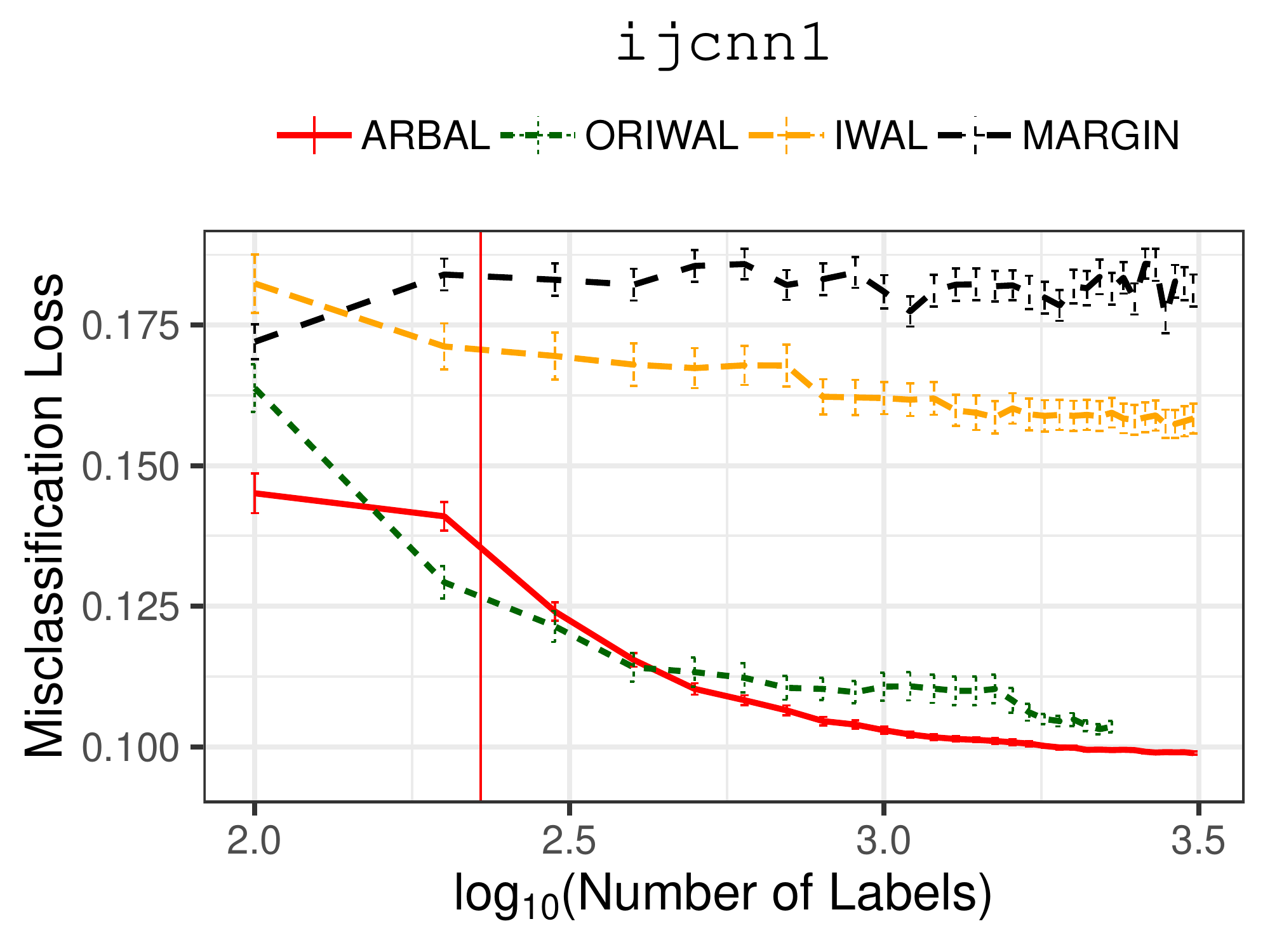}}
\subfigure{\centering\includegraphics[width=0.3\textwidth,,trim= 5 10 10 5,clip=true]{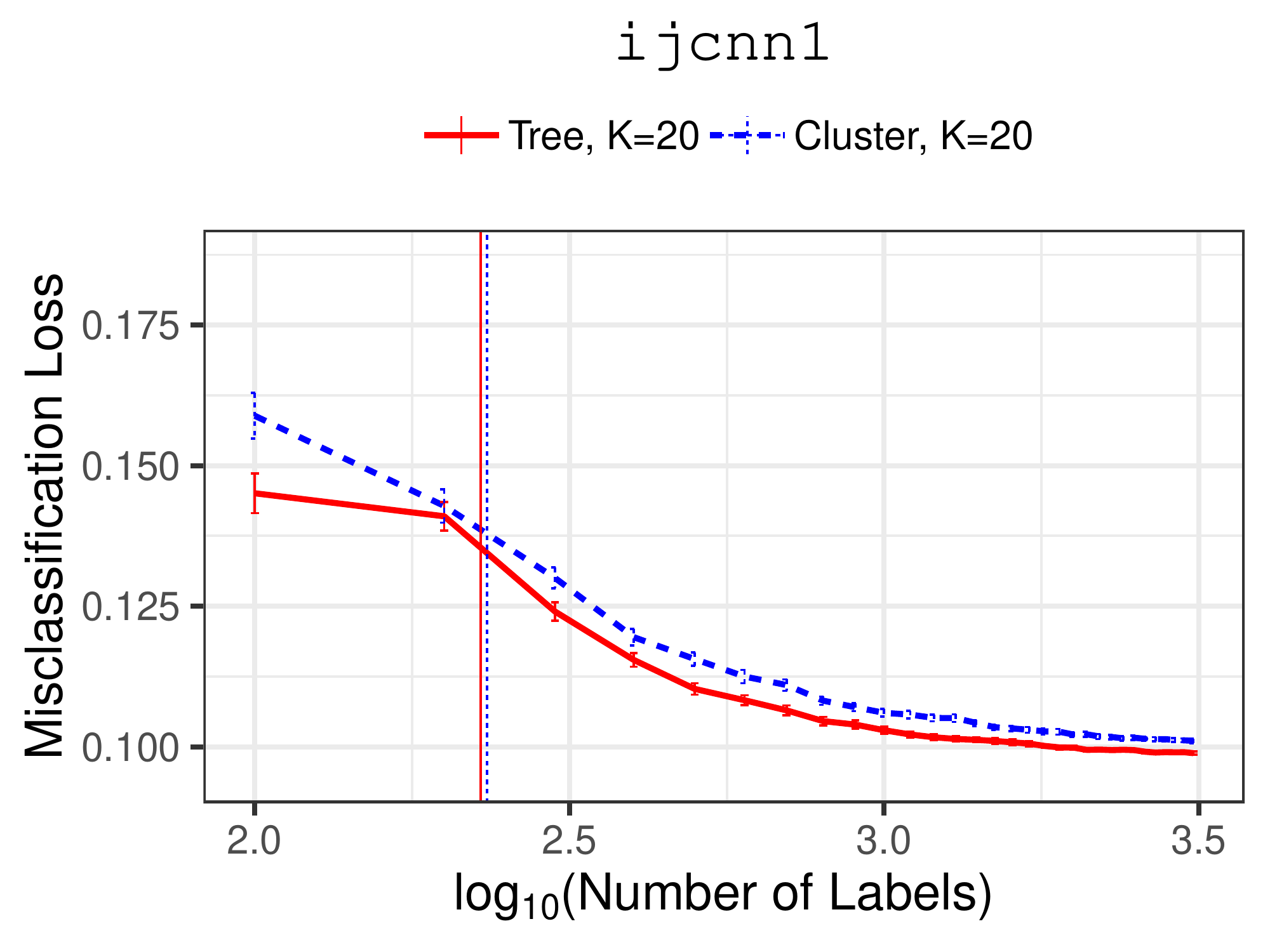}}\\
\subfigure{\centering\includegraphics[width=0.3\textwidth,,trim= 5 10 10 5,clip=true]{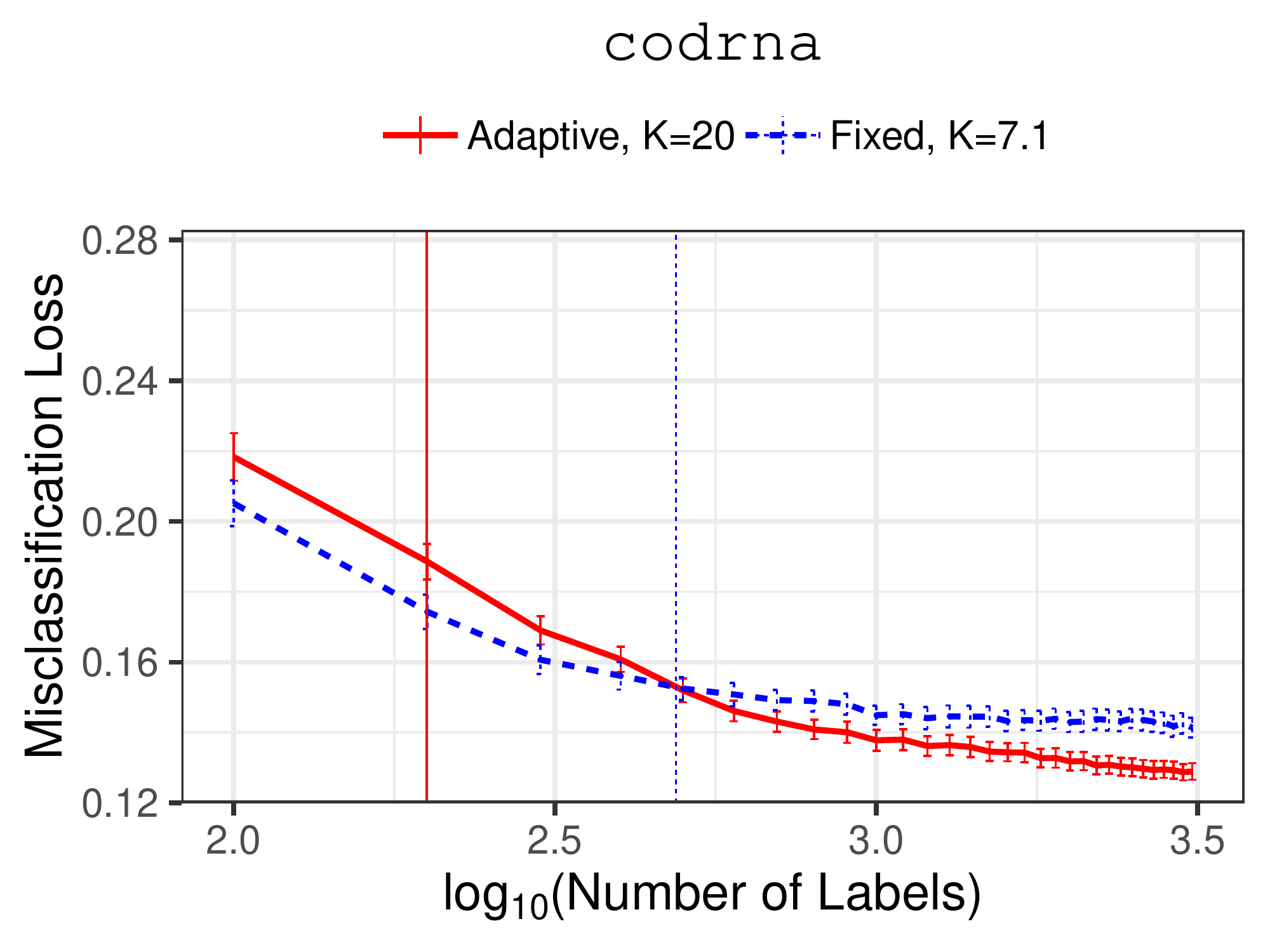}}
\subfigure{\centering\includegraphics[width=0.3\textwidth,,trim= 5 10 10 5,clip=true]{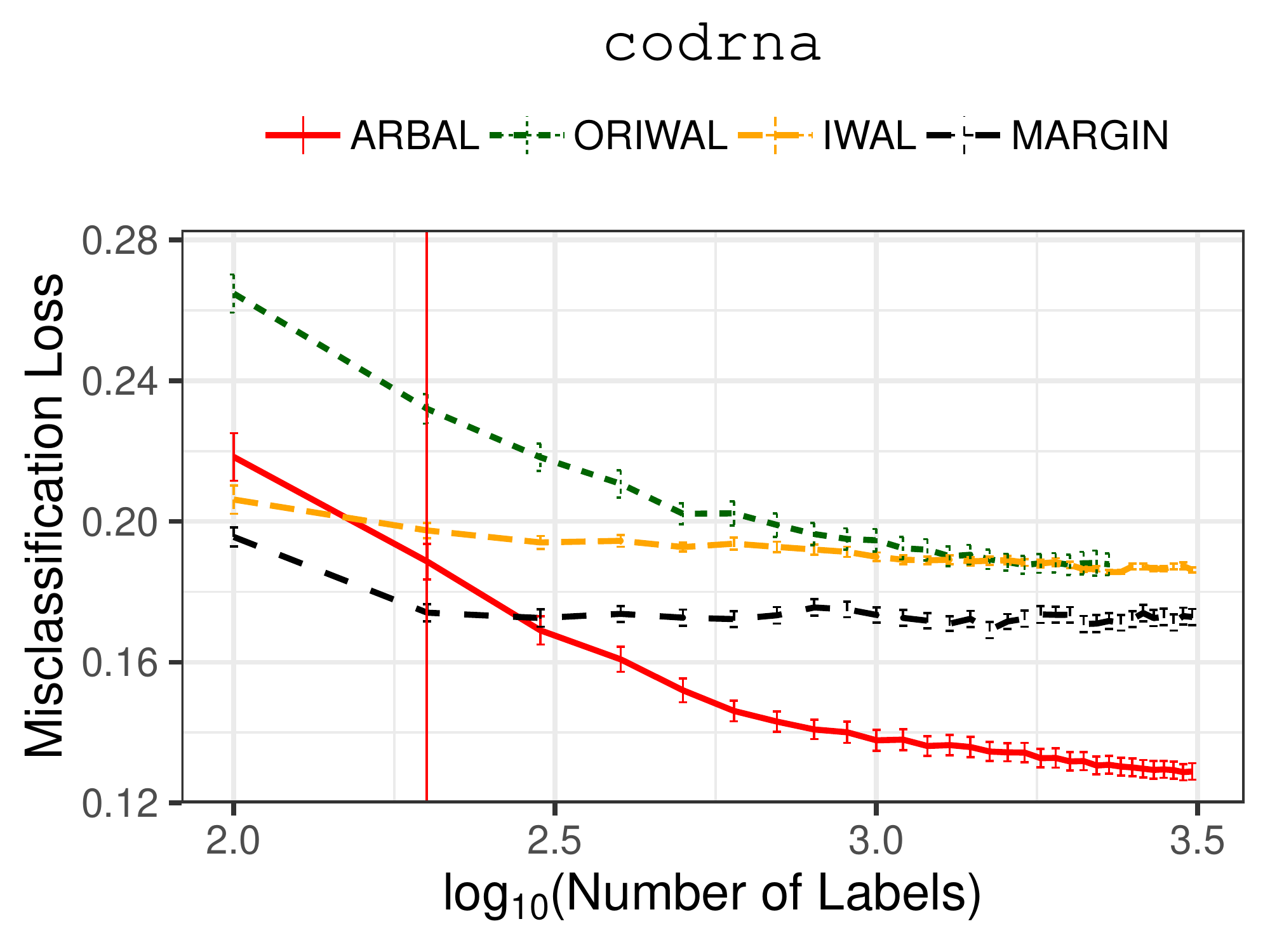}}
\subfigure{\centering\includegraphics[width=0.3\textwidth,,trim= 5 10 10 5,clip=true]{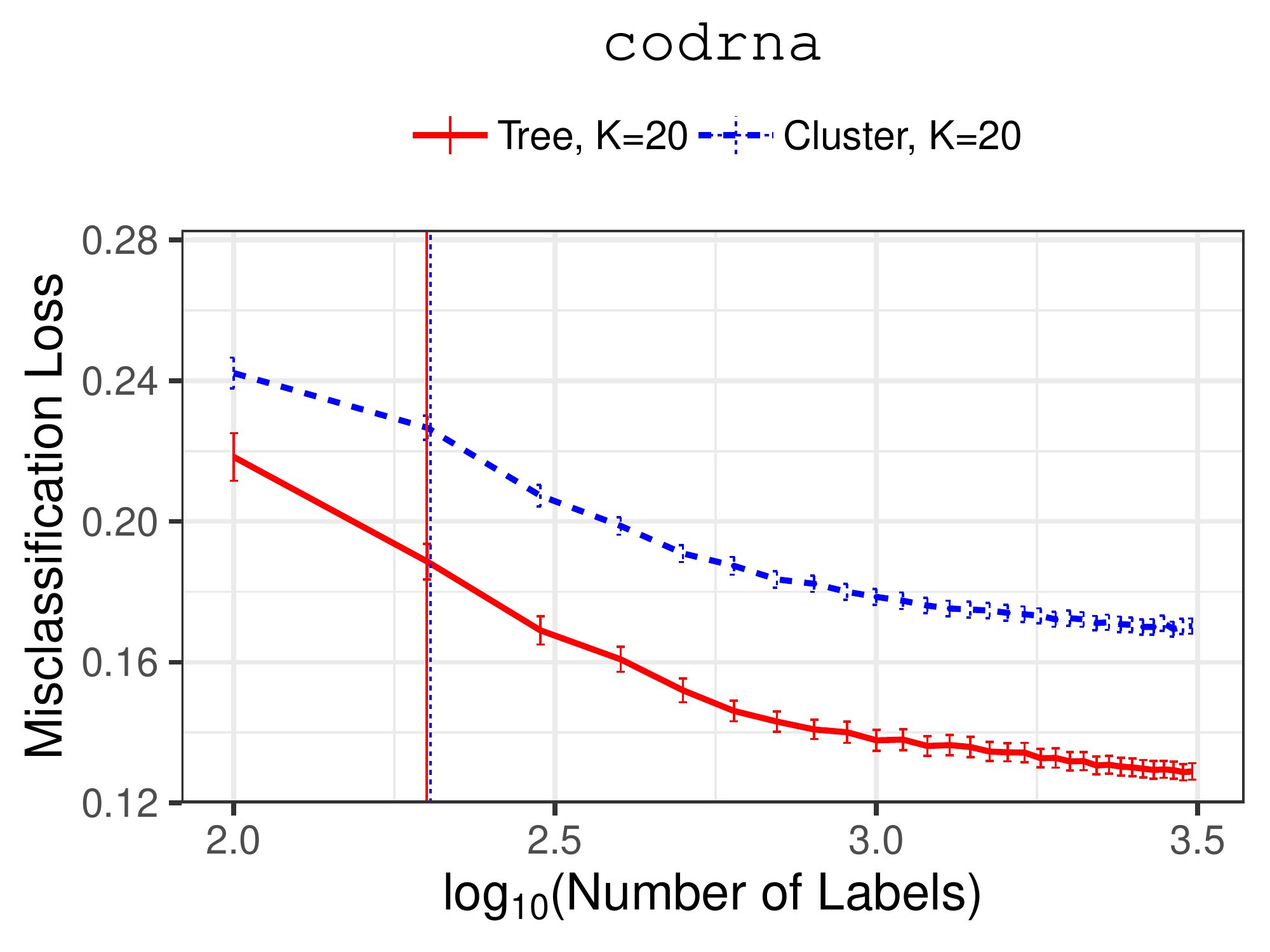}}\\
\end{center}
\caption{Misclassification loss on hold out test data versus number of labels requested ($\log_{10}$ scale). 
Left: \arbal\ with fixed and adaptive threshold $\gamma$.
Middle: \arbal, \riwal, \iwal, and \margin.
Right: \arbal\ with different partitioning methods: binary tree and hierarchical clustering.
For $\kappa=20$ and $\tau=800$,
dataset \texttt{\small{shuttle}}, \texttt{\small{electricity}}, \texttt{\small{a9a}}, \texttt{\small{ijcnn1}}, \texttt{\small{codrna}}.
For left and right plots, we give the average number of resulting regions $K$ in the legend. The vertical lines indicate
when \arbal\ transits from the first to the second phase.}
\label{fig:expmis_4}
\vskip -0.2in
\end{figure*}
\begin{figure*}[ht]
\begin{center}
\subfigure{\centering\includegraphics[width=0.3\textwidth,,trim= 5 10 10 5,clip=true]{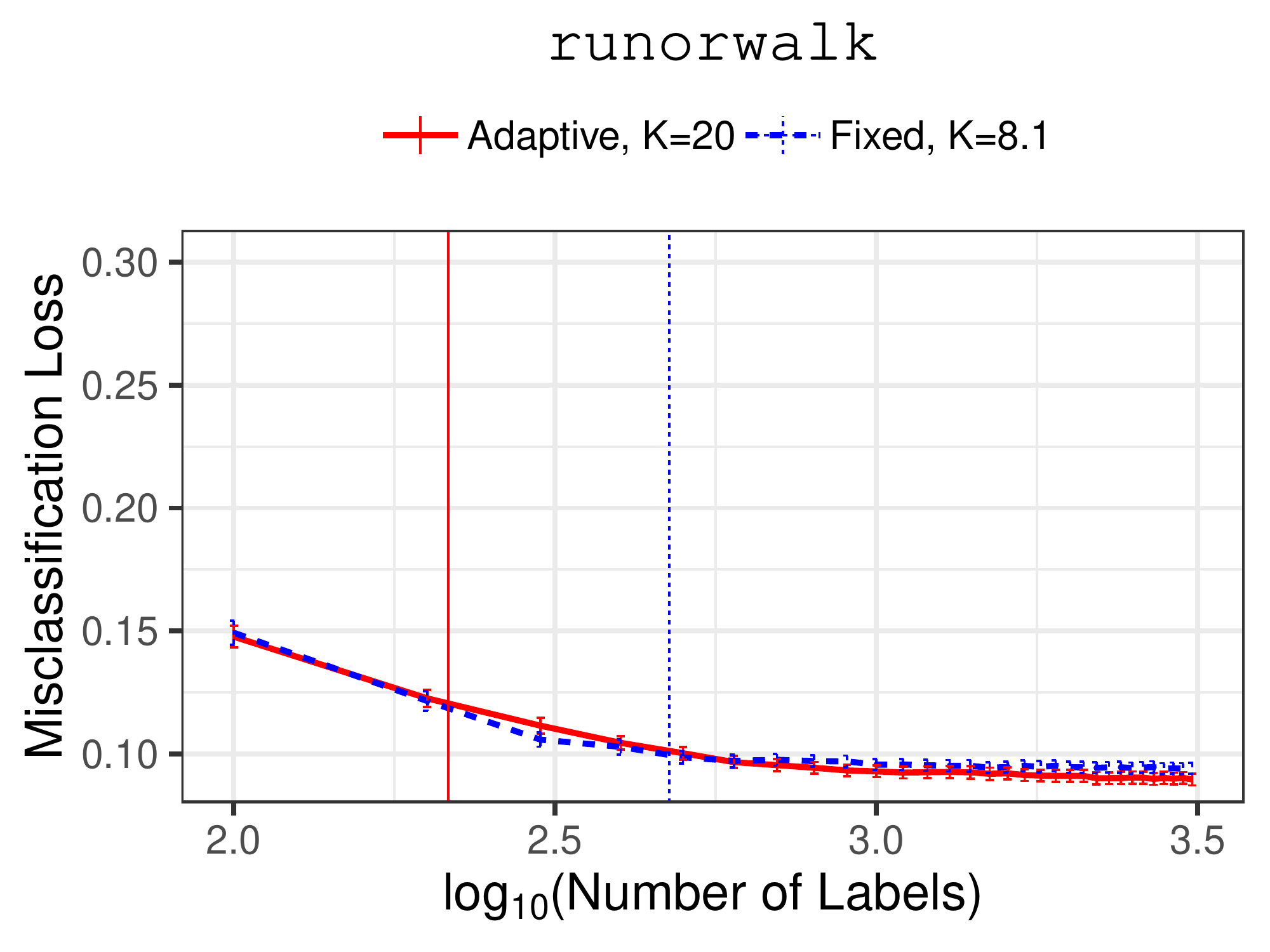}}
\subfigure{\centering\includegraphics[width=0.3\textwidth,,trim= 5 10 10 5,clip=true]{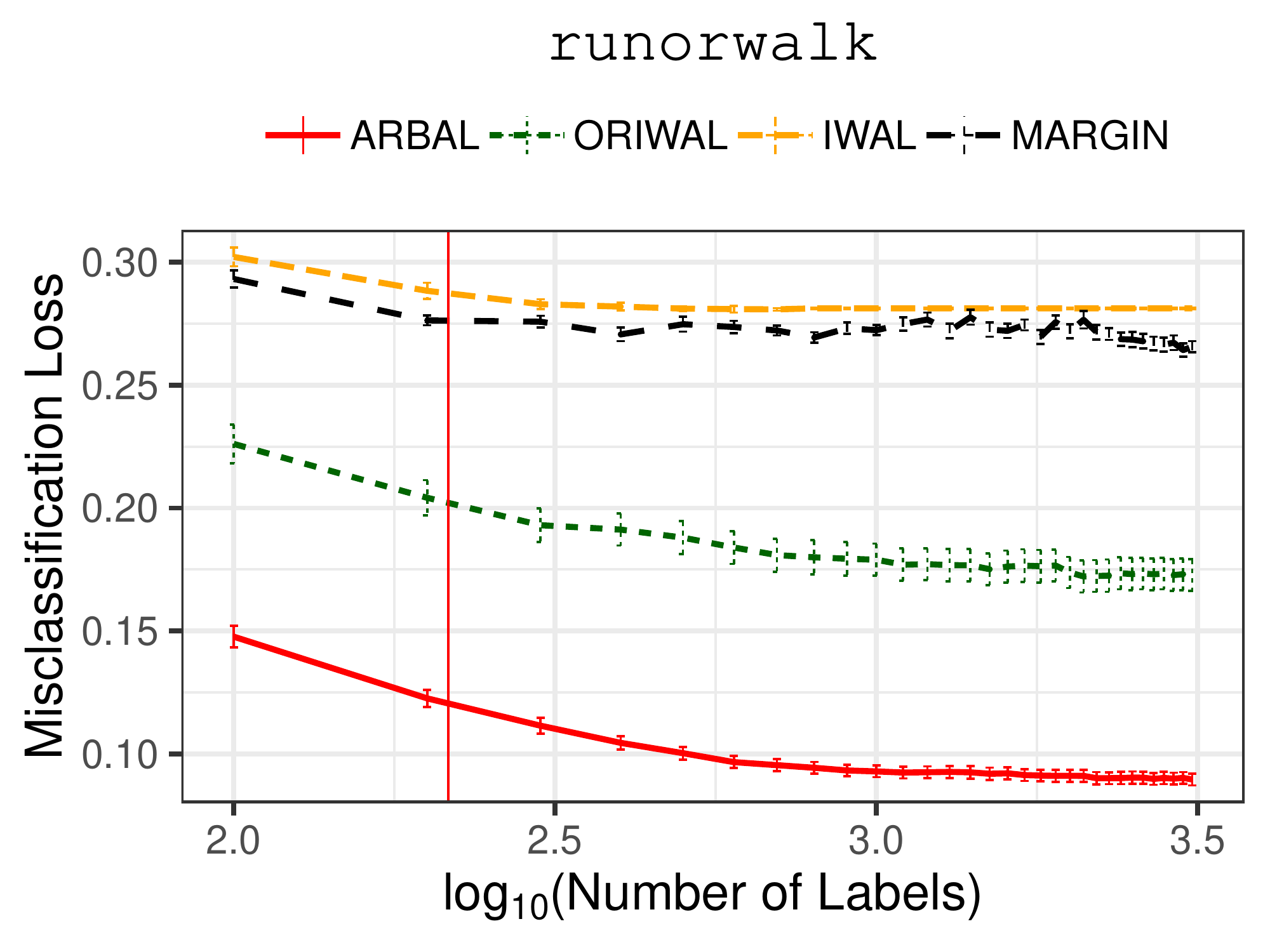}}
\subfigure{\centering\includegraphics[width=0.3\textwidth,,trim= 5 10 10 5,clip=true]{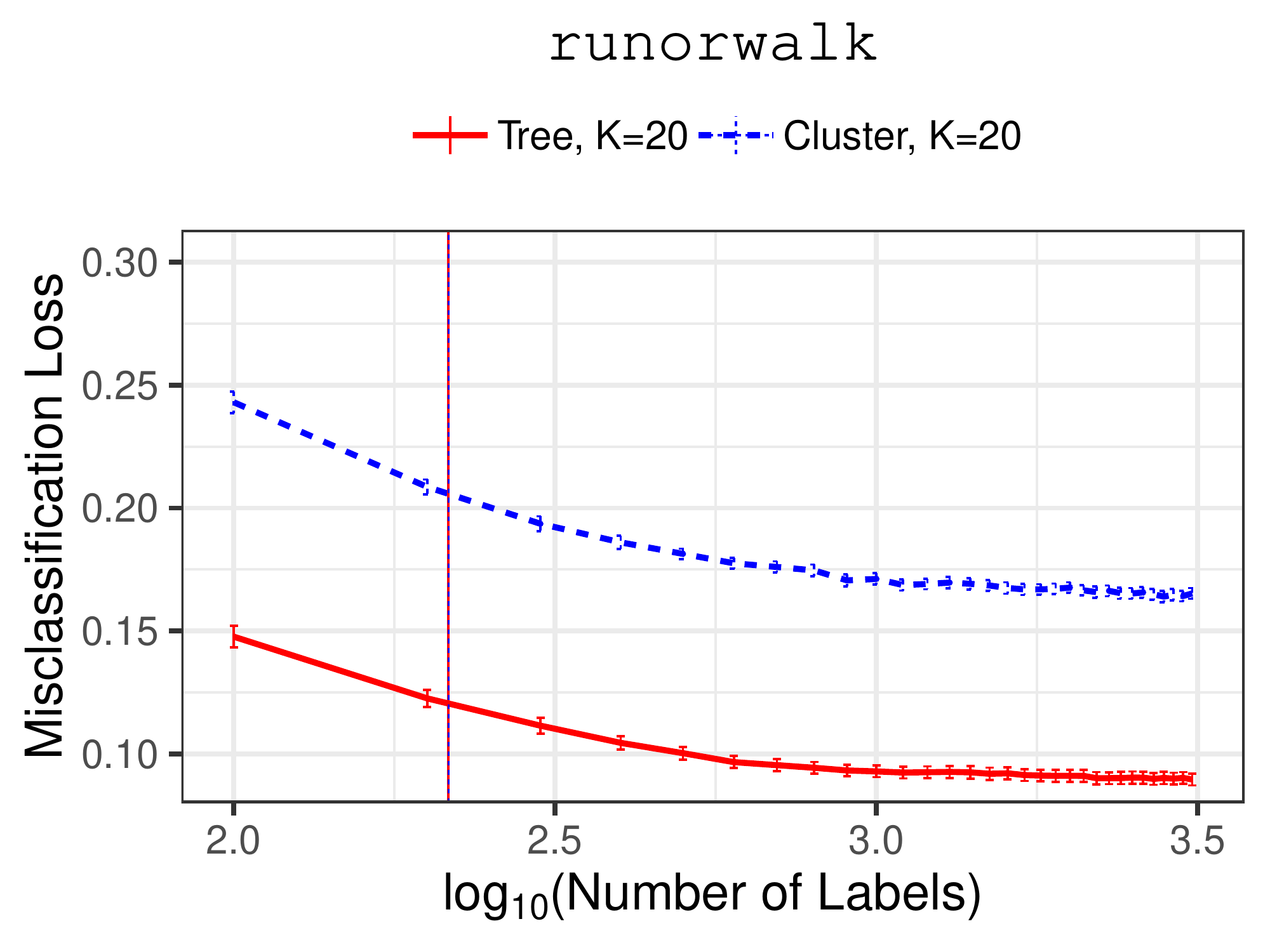}}\\
\subfigure{\centering\includegraphics[width=0.3\textwidth,,trim= 5 10 10 5,clip=true]{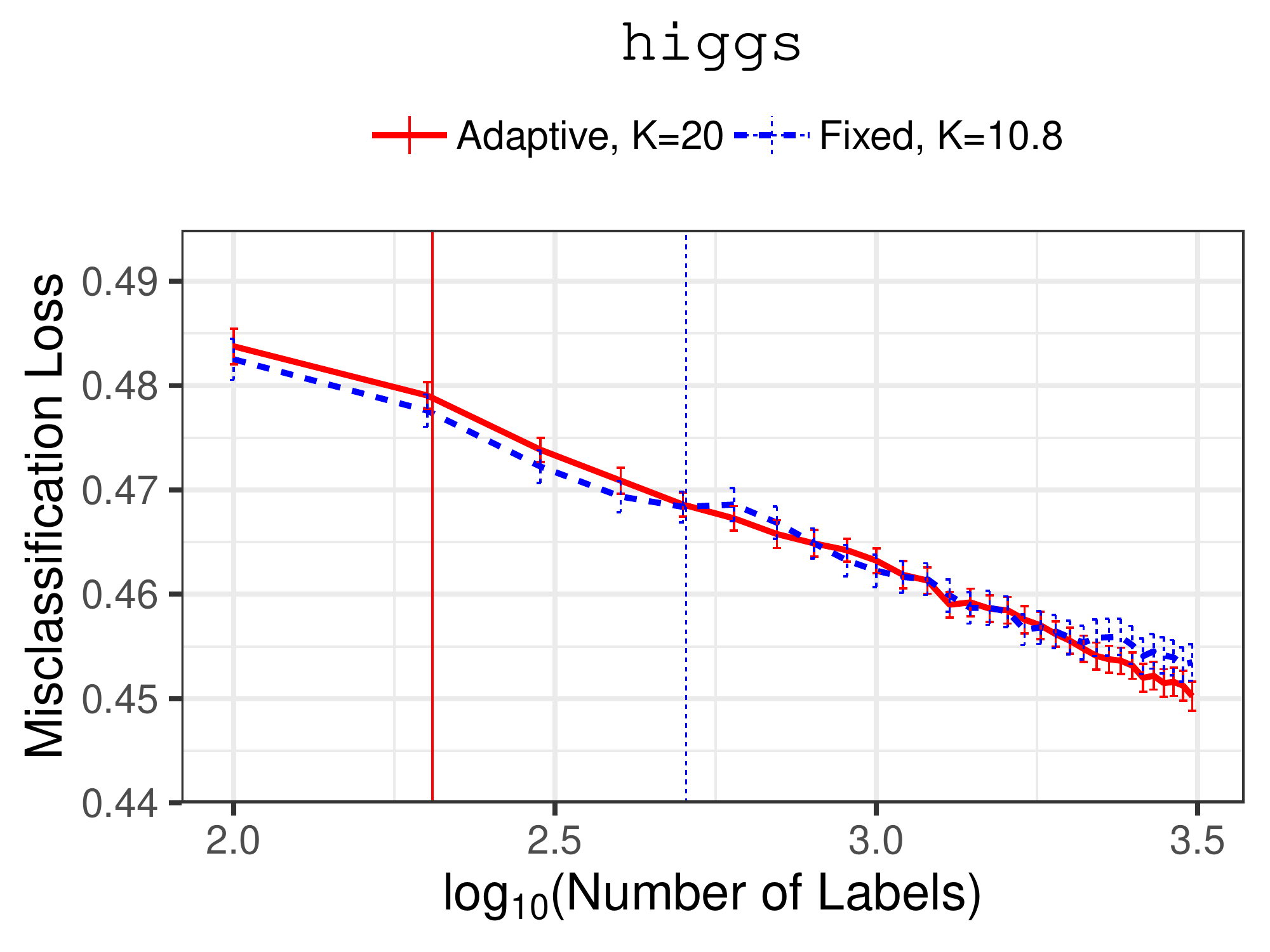}}
\subfigure{\centering\includegraphics[width=0.3\textwidth,,trim= 5 10 10 5,clip=true]{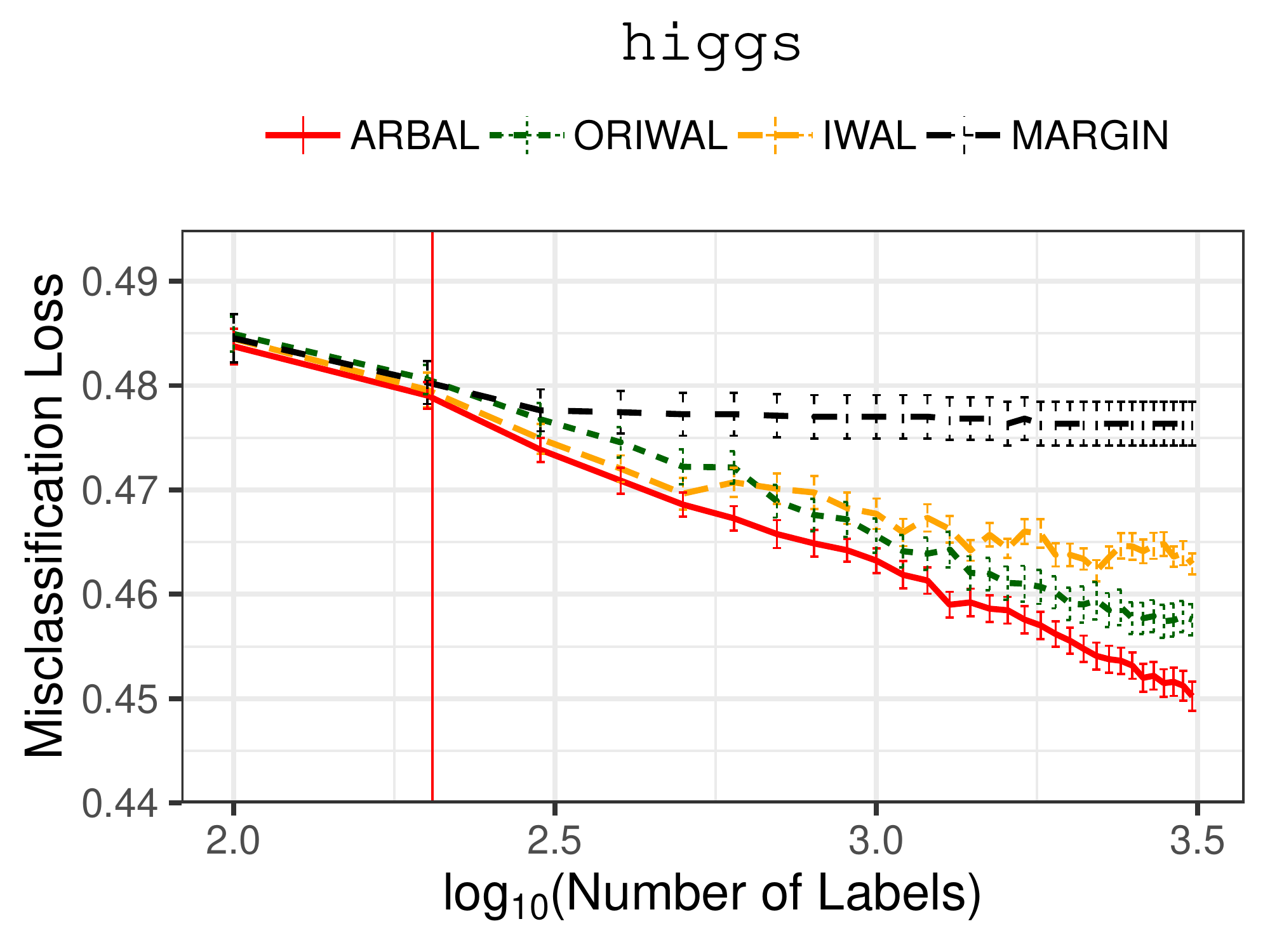}}
\subfigure{\centering\includegraphics[width=0.3\textwidth,,trim= 5 10 10 5,clip=true]{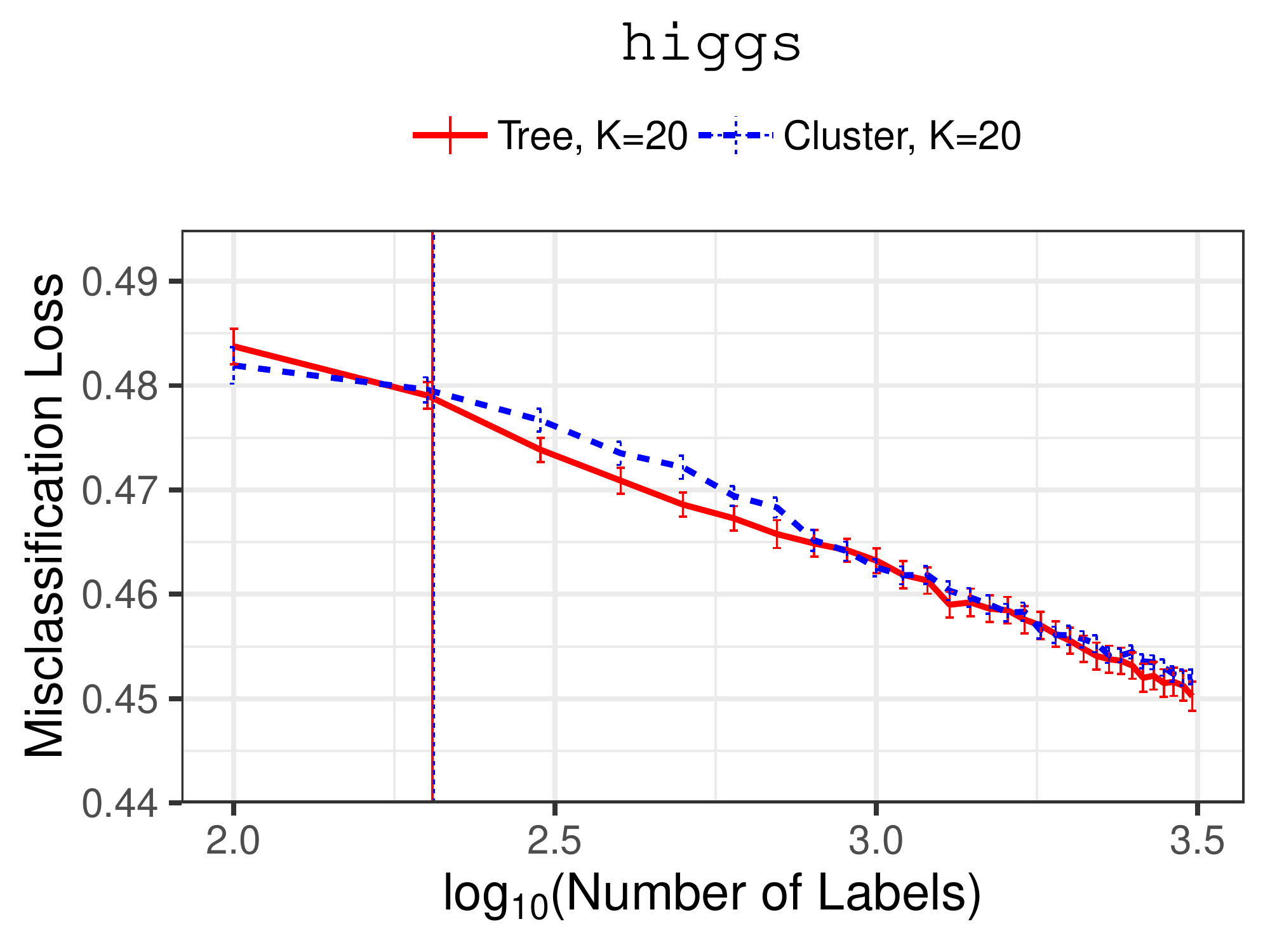}}\\
\subfigure{\centering\includegraphics[width=0.3\textwidth,,trim= 5 10 10 5,clip=true]{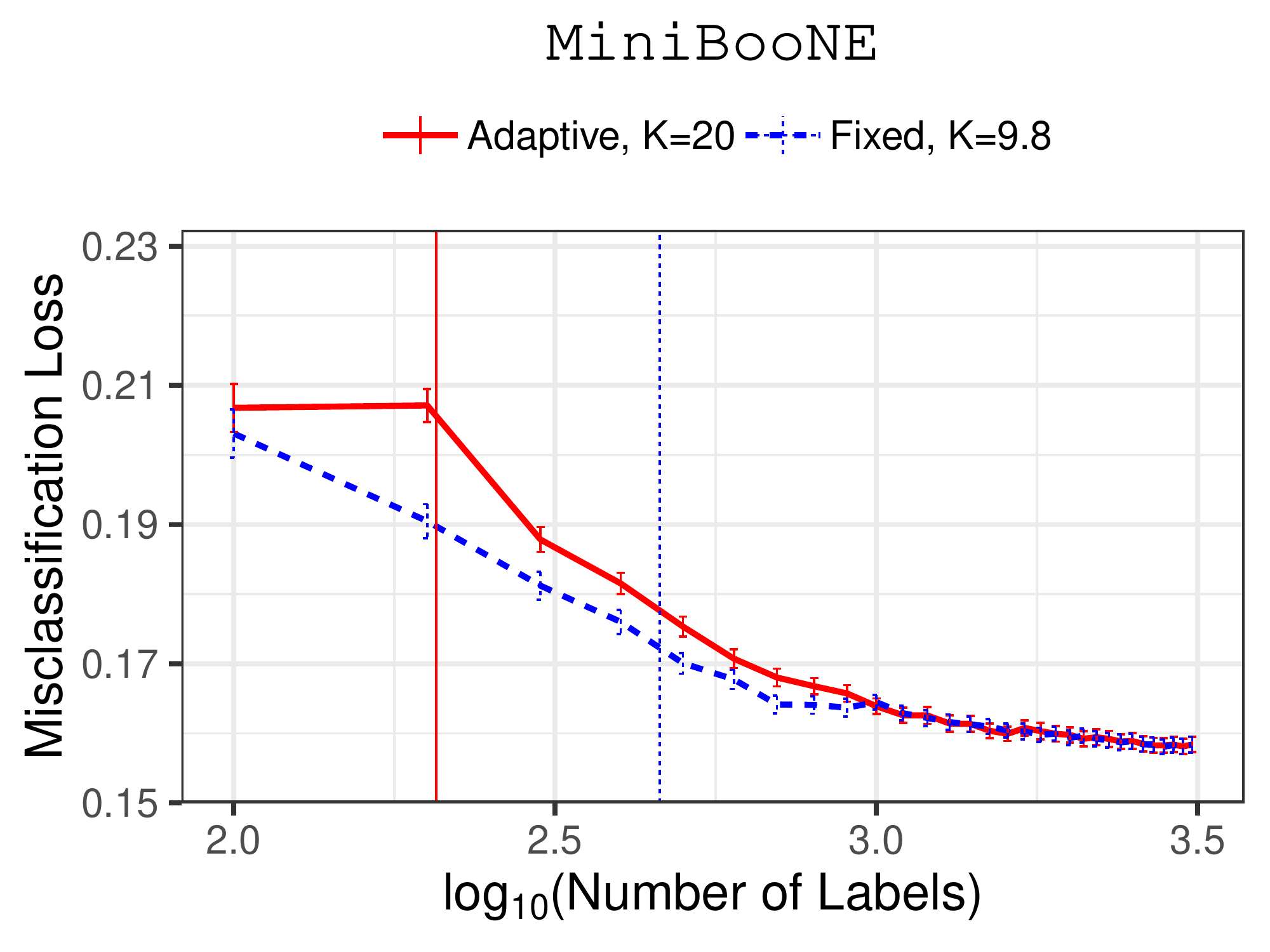}}
\subfigure{\centering\includegraphics[width=0.3\textwidth,,trim= 5 10 10 5,clip=true]{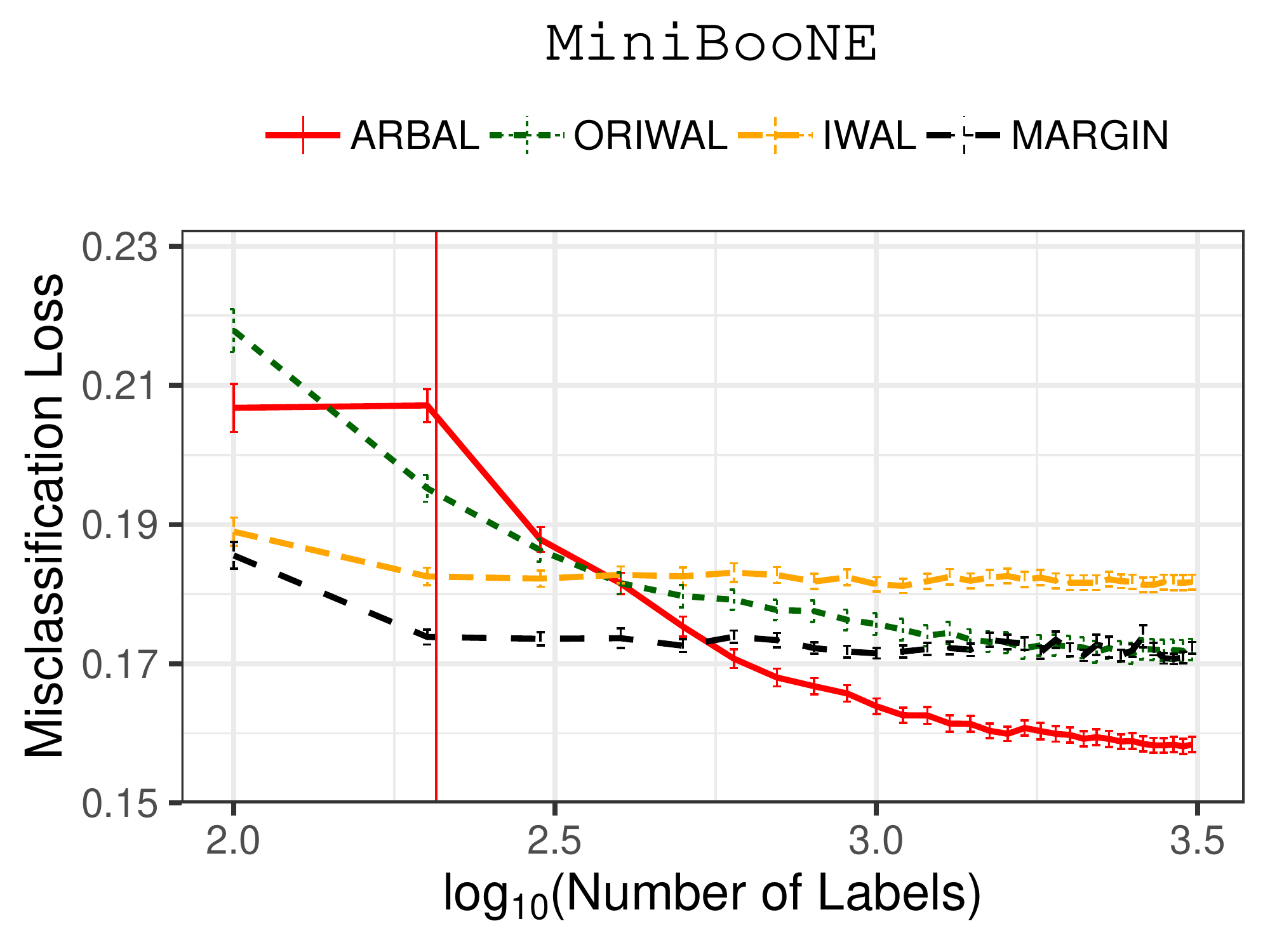}}
\subfigure{\centering\includegraphics[width=0.3\textwidth,,trim= 5 10 10 5,clip=true]{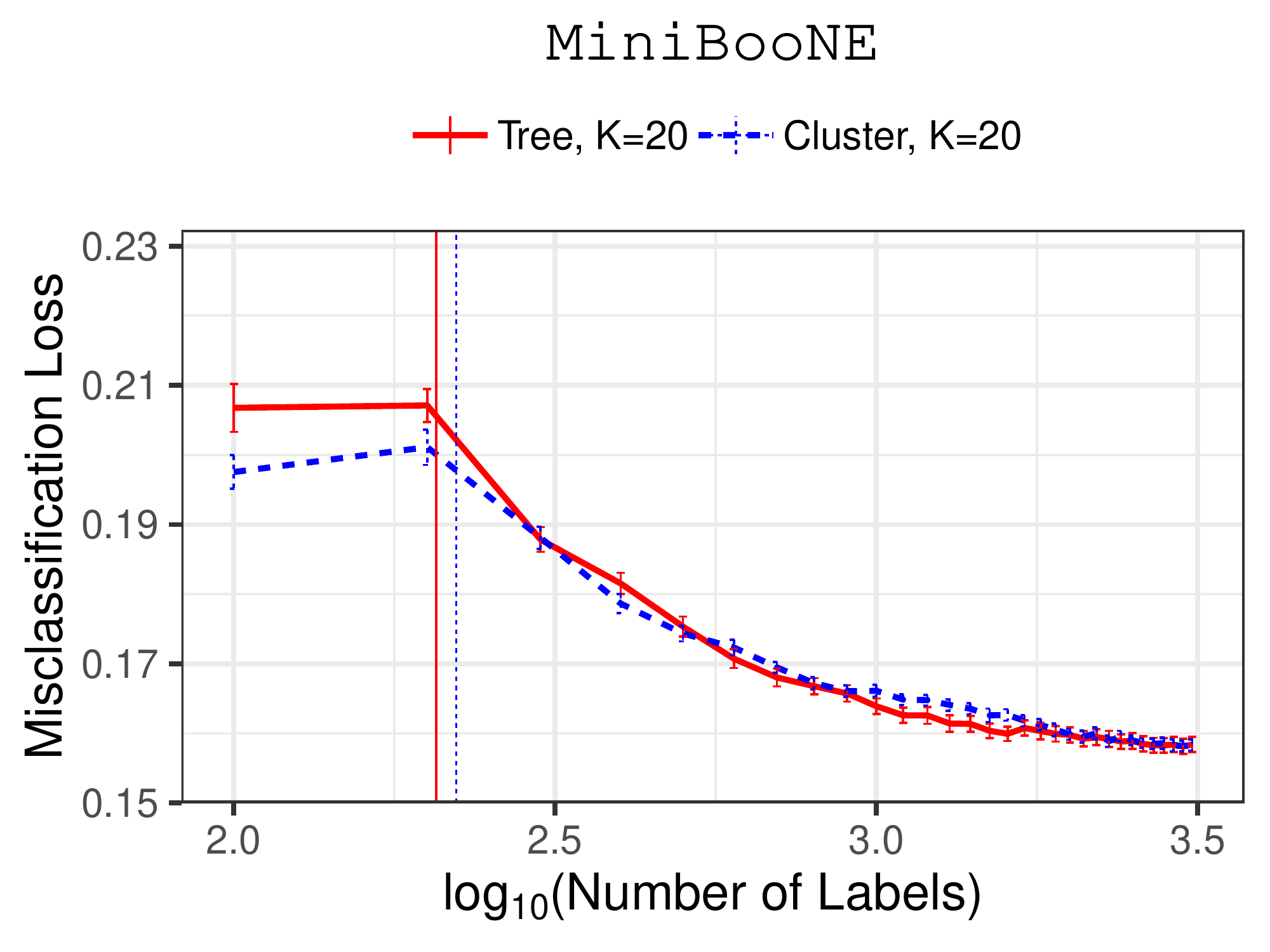}}\\
\subfigure{\centering\includegraphics[width=0.3\textwidth,,trim= 5 10 10 5,clip=true]{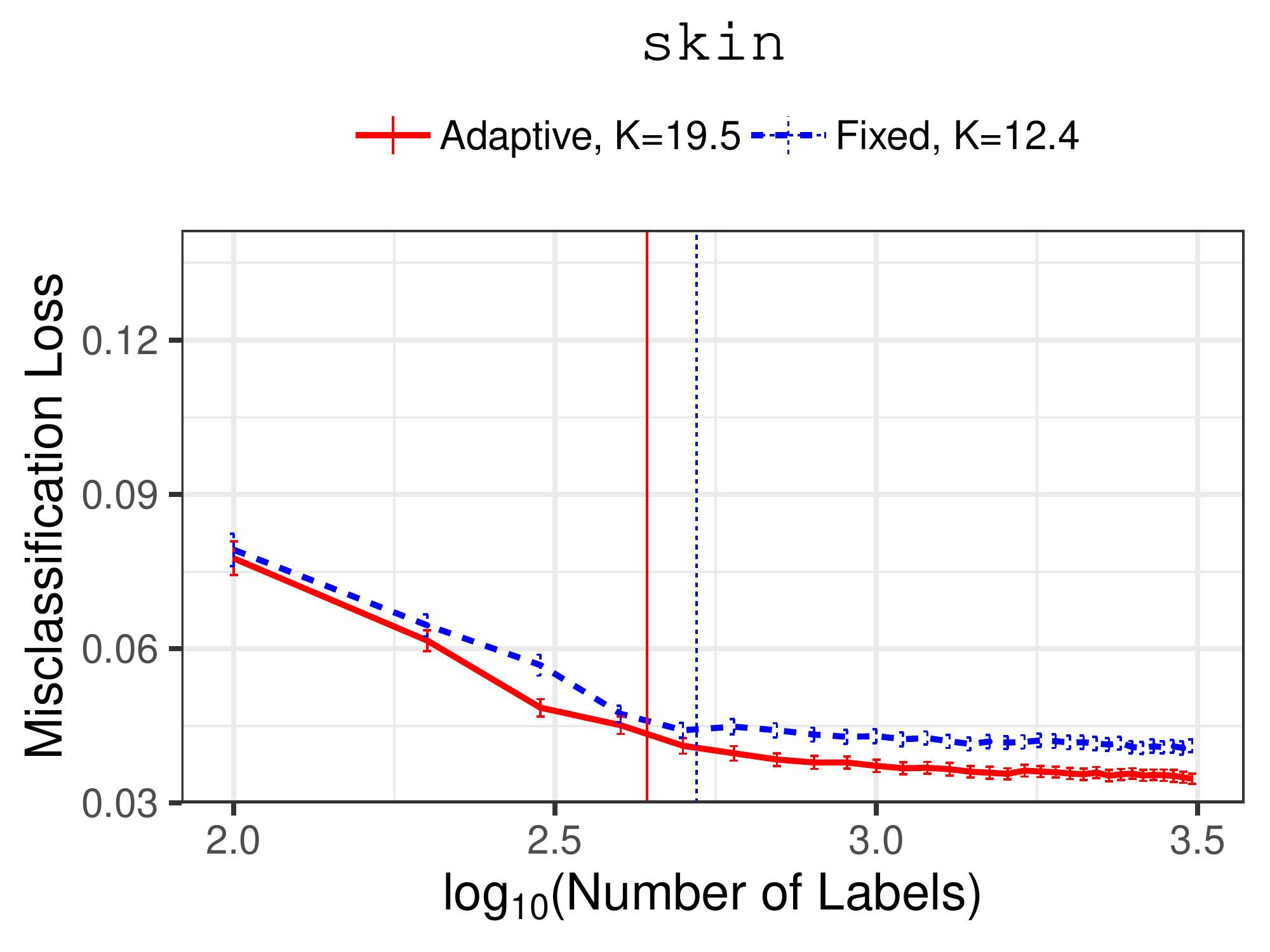}}
\subfigure{\centering\includegraphics[width=0.3\textwidth,,trim= 5 10 10 5,clip=true]{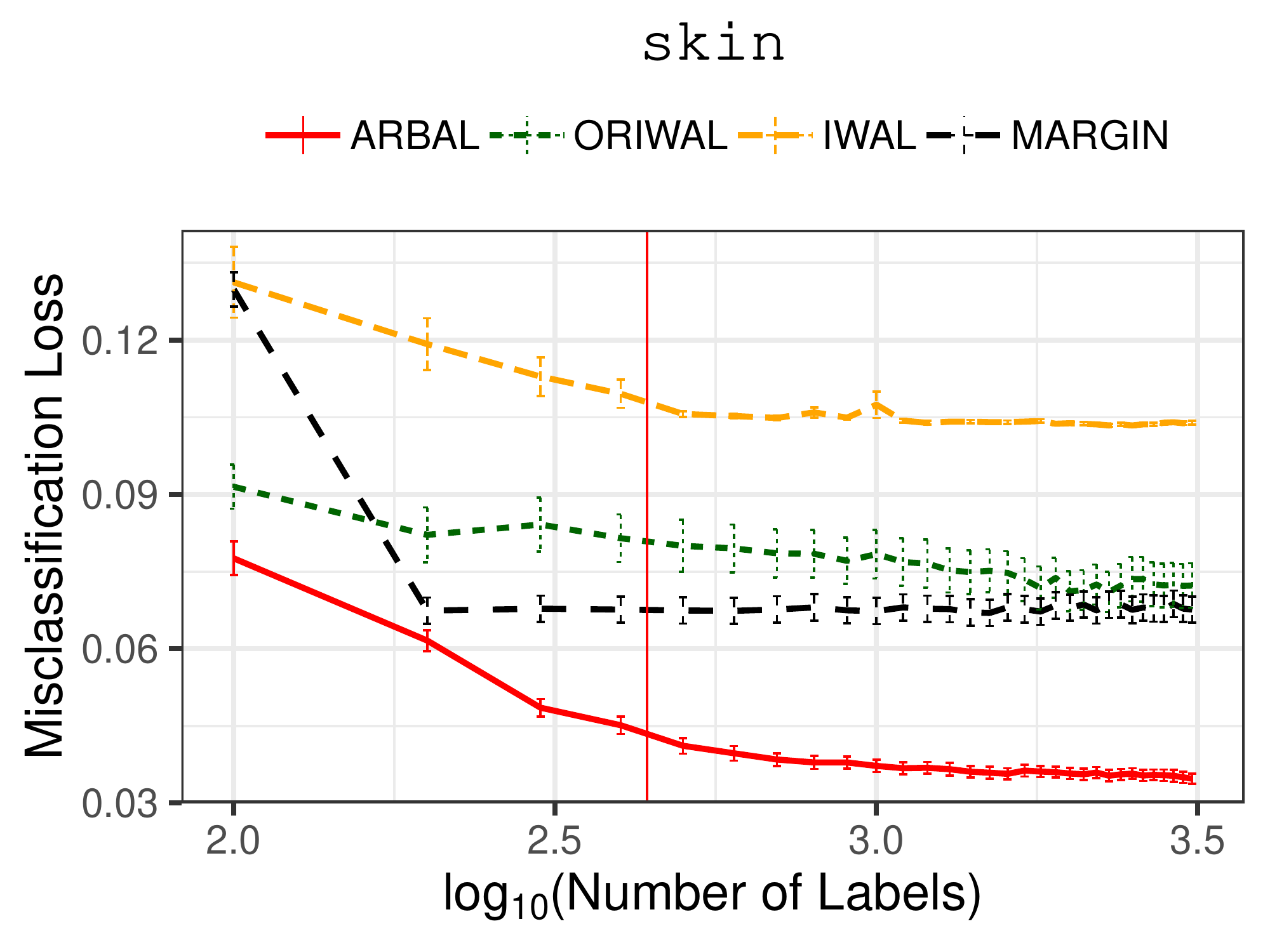}}
\subfigure{\centering\includegraphics[width=0.3\textwidth,,trim= 5 10 10 5,clip=true]{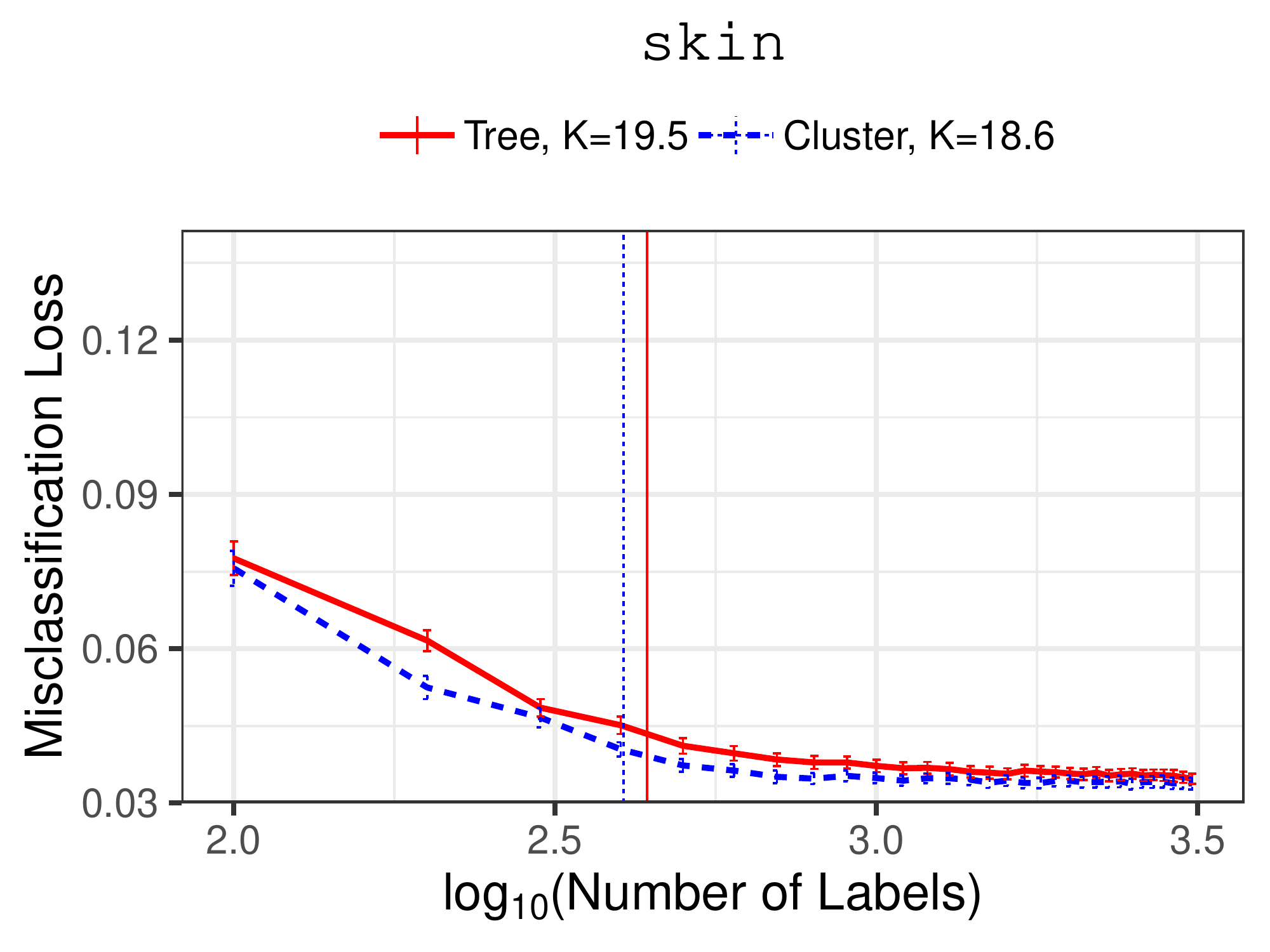}}\\
\subfigure{\centering\includegraphics[width=0.3\textwidth,,trim= 5 10 10 5,clip=true]{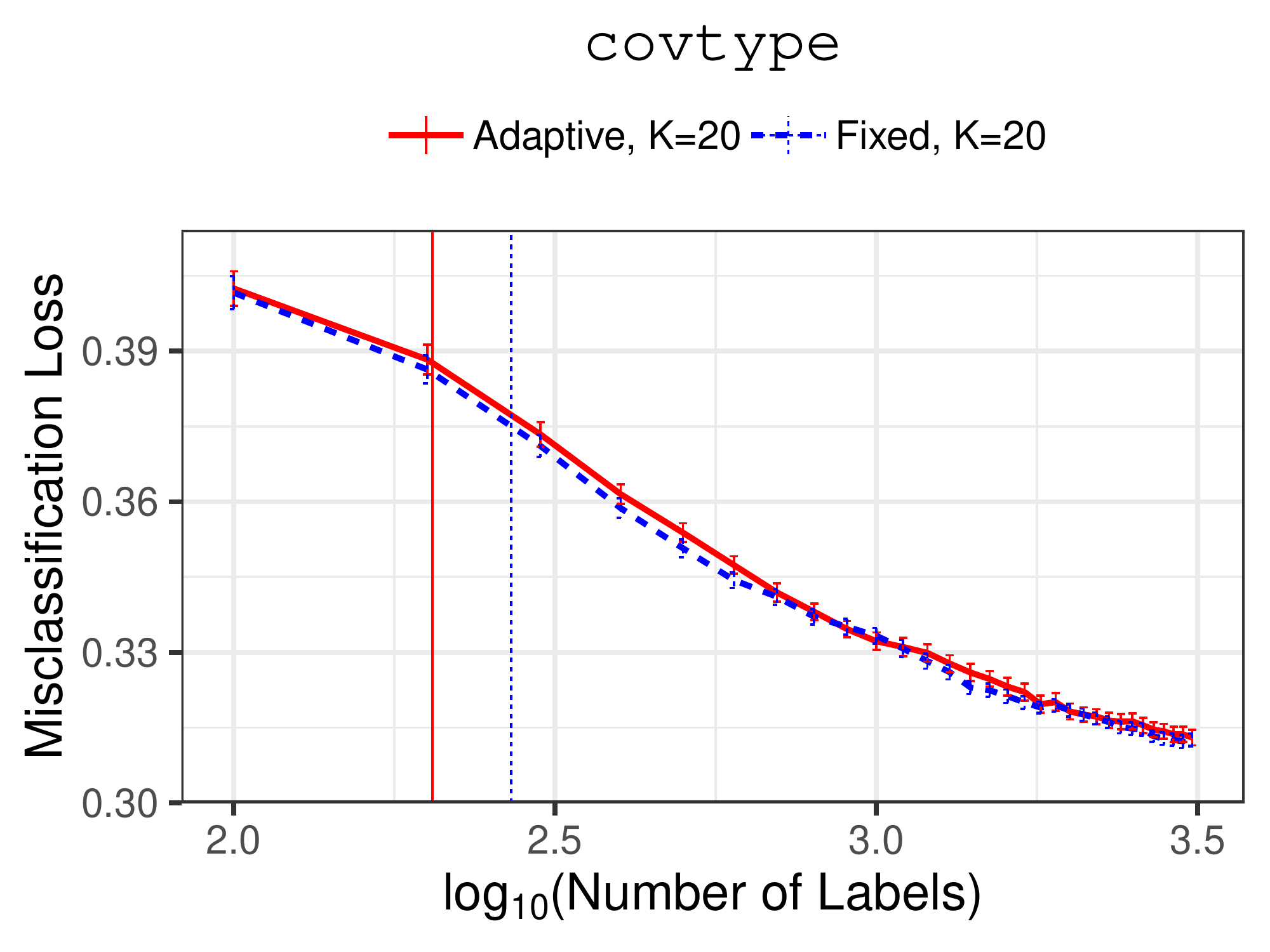}}
\subfigure{\centering\includegraphics[width=0.3\textwidth,,trim= 5 10 10 5,clip=true]{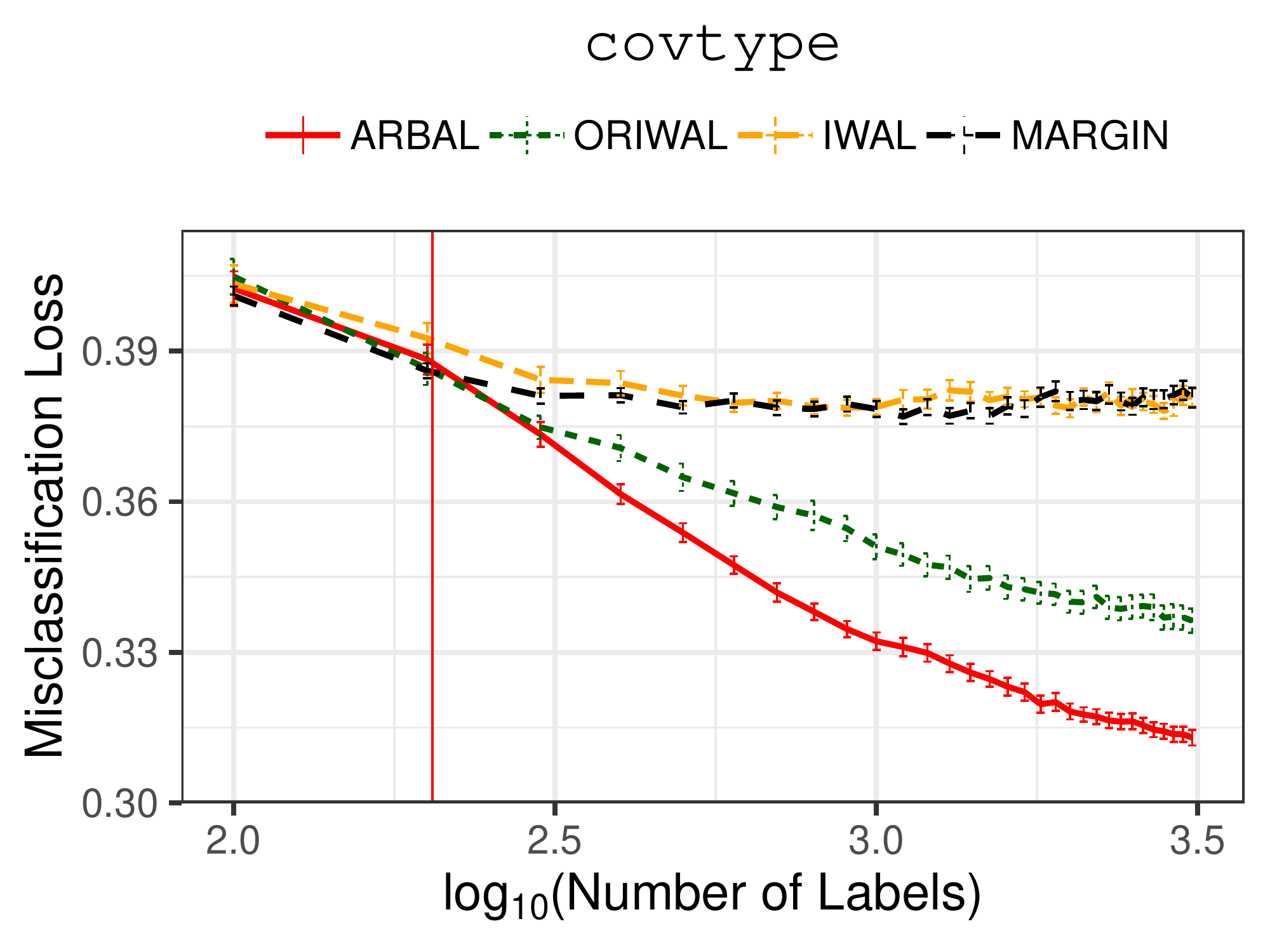}}
\subfigure{\centering\includegraphics[width=0.3\textwidth,,trim= 5 10 10 5,clip=true]{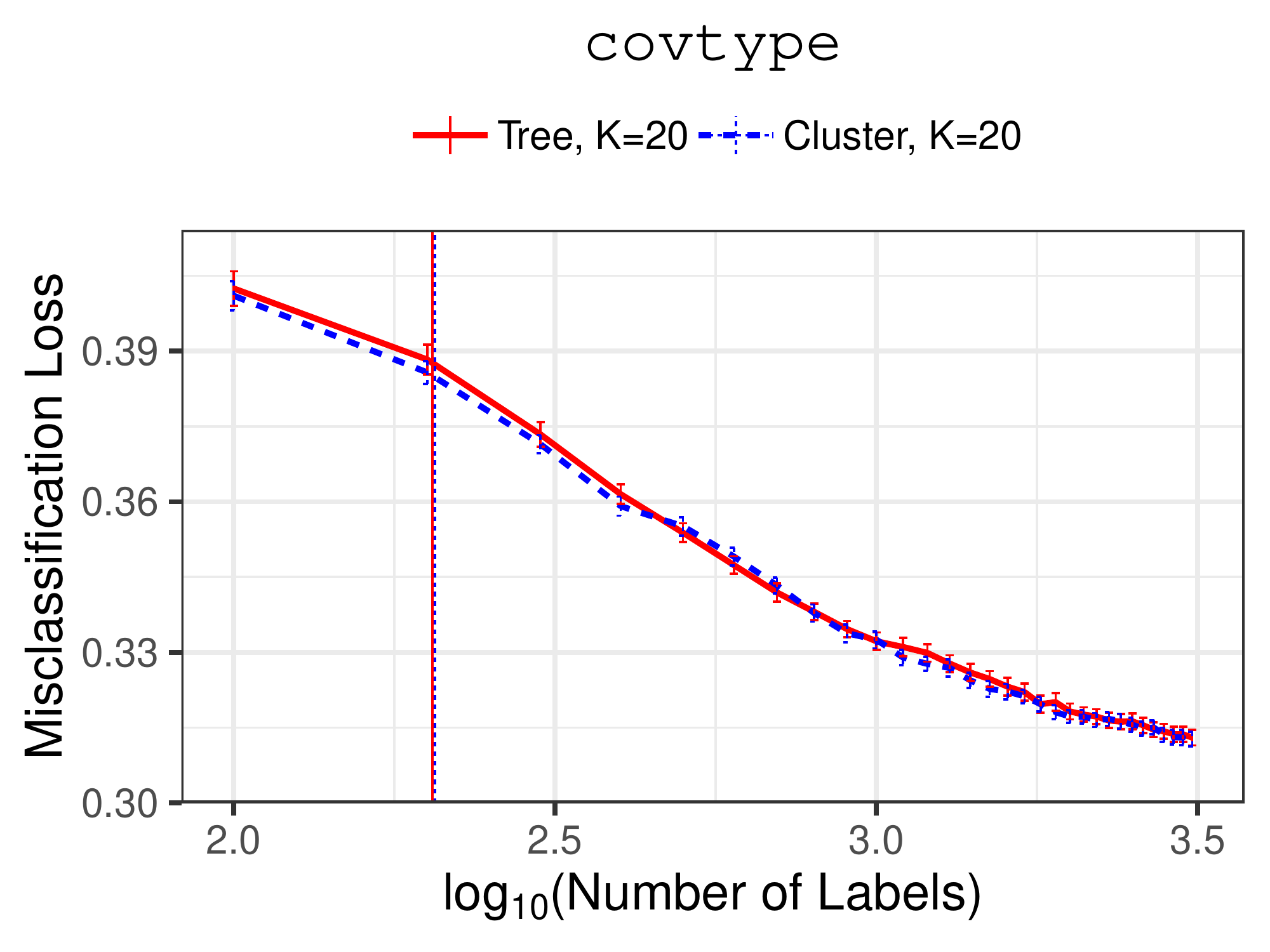}}\\
\end{center}
\caption{Misclassification loss on hold out test data versus number of labels requested ($\log_{10}$ scale). 
Left: \arbal\ with fixed and adaptive threshold $\gamma$.
Middle: \arbal, \riwal, \iwal, and \margin.
Right: \arbal\ with different partitioning methods: binary tree and hierarchical clustering.
For $\kappa=20$ and $\tau=800$,
dataset \texttt{\small{runorwalk}}, \texttt{\small{higgs}}, \texttt{\small{MiniBooNE}}, \texttt{\small{skin}}, \texttt{\small{covtype}}.
For left and right plots, we give the average number of resulting regions $K$ in the legend. The vertical lines indicate
when \arbal\ transits from the first to the second phase.}
\label{fig:expmis_5}
\vskip -0.2in
\end{figure*}
\end{document}